\PassOptionsToPackage{unicode}{hyperref}
\PassOptionsToPackage{hyphens}{url}
\PassOptionsToPackage{dvipsnames,svgnames,x11names}{xcolor}
\documentclass[
  12pt]{article}

\usepackage{amsmath,amssymb}
\usepackage{iftex}
\ifPDFTeX
  \usepackage[T1]{fontenc}
  \usepackage[utf8]{inputenc}
  \usepackage{textcomp} 
\else 
  \usepackage{unicode-math}
  \defaultfontfeatures{Scale=MatchLowercase}
  \defaultfontfeatures[\rmfamily]{Ligatures=TeX,Scale=1}
\fi
\usepackage{lmodern}
\ifPDFTeX\else  
\fi
\IfFileExists{upquote.sty}{\usepackage{upquote}}{}
\IfFileExists{microtype.sty}{
  \usepackage[]{microtype}
  \UseMicrotypeSet[protrusion]{basicmath} 
}{}
\makeatletter
\@ifundefined{KOMAClassName}{
  \IfFileExists{parskip.sty}{%
    \usepackage{parskip}
  }{
    \setlength{\parindent}{0pt}
    \setlength{\parskip}{6pt plus 2pt minus 1pt}}
}{
  \KOMAoptions{parskip=half}}
\makeatother
\usepackage{xcolor}
\makeatletter
\ifx\paragraph\undefined\else
  \let\oldparagraph\paragraph
  \renewcommand{\paragraph}{
    \@ifstar
      \xxxParagraphStar
      \xxxParagraphNoStar
  }
  \newcommand{\xxxParagraphStar}[1]{\oldparagraph*{#1}\mbox{}}
  \newcommand{\xxxParagraphNoStar}[1]{\oldparagraph{#1}\mbox{}}
\fi
\ifx\subparagraph\undefined\else
  \let\oldsubparagraph\subparagraph
  \renewcommand{\subparagraph}{
    \@ifstar
      \xxxSubParagraphStar
      \xxxSubParagraphNoStar
  }
  \newcommand{\xxxSubParagraphStar}[1]{\oldsubparagraph*{#1}\mbox{}}
  \newcommand{\xxxSubParagraphNoStar}[1]{\oldsubparagraph{#1}\mbox{}}
\fi
\makeatother

\usepackage{longtable,booktabs,array}
\usepackage{calc} 
\usepackage{etoolbox}
\makeatletter
\patchcmd\longtable{\par}{\if@noskipsec\mbox{}\fi\par}{}{}
\makeatother
\IfFileExists{footnotehyper.sty}{\usepackage{footnotehyper}}{\usepackage{footnote}}
\makesavenoteenv{longtable}
\usepackage{graphicx}
\makeatletter
\def\maxwidth{\ifdim\Gin@nat@width>\linewidth\linewidth\else\Gin@nat@width\fi}
\def\maxheight{\ifdim\Gin@nat@height>\textheight\textheight\else\Gin@nat@height\fi}
\makeatother
\setkeys{Gin}{width=\maxwidth,height=\maxheight,keepaspectratio}
\makeatletter
\def\fps@figure{htbp}
\makeatother

\makeatletter
\@ifpackageloaded{caption}{}{\usepackage{caption}}
\AtBeginDocument{%
\ifdefined\contentsname
  \renewcommand*\contentsname{Table of contents}
\else
  \newcommand\contentsname{Table of contents}
\fi
\ifdefined\listfigurename
  \renewcommand*\listfigurename{List of Figures}
\else
  \newcommand\listfigurename{List of Figures}
\fi
\ifdefined\listtablename
  \renewcommand*\listtablename{List of Tables}
\else
  \newcommand\listtablename{List of Tables}
\fi
\ifdefined\figurename
  \renewcommand*\figurename{Figure}
\else
  \newcommand\figurename{Figure}
\fi
\ifdefined\tablename
  \renewcommand*\tablename{Table}
\else
  \newcommand\tablename{Table}
\fi
}
\@ifpackageloaded{float}{}{\usepackage{float}}
\floatstyle{ruled}
\@ifundefined{c@chapter}{\newfloat{codelisting}{h}{lop}}{\newfloat{codelisting}{h}{lop}[chapter]}
\floatname{codelisting}{Listing}

\makeatother
\makeatletter
\@ifpackageloaded{caption}{}{\usepackage{caption}}
\makeatother

\ifLuaTeX
  \usepackage{selnolig}  
\fi
\usepackage[]{natbib}
\usepackage{bookmark}

\IfFileExists{xurl.sty}{\usepackage{xurl}}{} 
\hypersetup{
  pdftitle={Manifold-Aligned Generative Transport},
  pdfauthor={},
  pdfkeywords={manifold learning; diffusion; flows; high fidelity; synthetic data generation},
  colorlinks=true,
  linkcolor={black},
  filecolor={black},
  citecolor={black},
  urlcolor={black},
  pdfcreator={LaTeX}}

\usepackage{enumerate}
\usepackage{url}
\usepackage{tabularx}
\usepackage{array}
\usepackage{ragged2e}
\usepackage{caption}
\usepackage{algorithm}
\usepackage{algpseudocode}
\usepackage{tikz} 
\usepackage{booktabs}
\usepackage{makecell}
\usepackage{multirow}
\usepackage{bm}
\IfFileExists{orcidlink.sty}{\usepackage{orcidlink}}{\newcommand{\orcidlink}[1]{}}

\usepackage{float}
\IfFileExists{placeins.sty}{\usepackage{placeins}}{\newcommand{\FloatBarrier}{\clearpage}}

\usepackage{amsthm,booktabs,mathtools}
\newtheorem{assumption}{Assumption}
\theoremstyle{definition} 

\newtheorem{definition}{Definition} 
\newtheorem{theorem}{Theorem}

\newtheorem{proposition}{Proposition} 
\newtheorem{lemma}{Lemma}

\newcommand{\R}{\mathbb{R}}

\newcommand{\M}{\mathcal{M}}

\newcommand{\U}{\bm U}
\newcommand{\Y}{\bm Y}
\newcommand{\Z}{\bm Z}
\newcommand{\y}{\bm y}

\newcommand{\z}{\bm z}

\usepackage{amsmath,amssymb,amsthm}
\makeatletter
\@ifundefined{KL}{}{}
\@ifundefined{Law}{}{}
\makeatother

\providecommand{\E}{\mathbb{E}}
\newcommand{\Unif}{\mathrm{Unif}}

\providecommand{\method}{MAGT} 

\newcommand{\cN}{\mathcal{N}}

\newcommand{\safeincludegraphics}[2][]{%
  \IfFileExists{#2}{\includegraphics[#1]{#2}}{%
    \fbox{\parbox[c]{0.95\linewidth}{\centering\small Missing figure file: \texttt{\detokenize{#2}}}}%
  }%
}

\newcommand{\anon}{1}

\newcommand{\beginsupplement}{%
  \setcounter{theorem}{0}%
  \setcounter{lemma}{0}%
  \setcounter{assumption}{0}%
  \setcounter{definition}{0}%
  \setcounter{proposition}{0}%
  \setcounter{figure}{0}%
  \setcounter{table}{0}%
  \setcounter{equation}{0}%

  \renewcommand{\thetheorem}{S\arabic{theorem}}%
  \renewcommand{\thelemma}{S\arabic{lemma}}%
  \renewcommand{\theassumption}{S\arabic{assumption}}%
  \renewcommand{\thedefinition}{S\arabic{definition}}%
  \renewcommand{\theproposition}{S\arabic{proposition}}%
  \renewcommand{\thefigure}{S\arabic{figure}}%
  \renewcommand{\thetable}{S\arabic{table}}%
  \renewcommand{\theequation}{S\arabic{equation}}%
}

\begin{document}

\def\spacingset#1{\renewcommand{\baselinestretch}%
{#1}\small\normalsize} \spacingset{1}


\if1\anon
{
\hypersetup{pdfauthor={Xinyu Tian and Xiaotong Shen}}
\newcommand{\magtfunding}{This work was supported in part by the National
Science Foundation (NSF) under Grant DMS-2513668, by the National Institutes
of Health (NIH) under Grants R01AG069895, R01AG065636, R01AG074858, and
U01AG073079, and by the Minnesota Supercomputing Institute.
(Corresponding author: Xiaotong Shen.)}
\author{Xinyu Tian\thanks{Xinyu Tian is with the School of Statistics,
University of Minnesota, MN, 55455 USA (email: tianx@umn.edu)},
Xiaotong Shen\textsuperscript{\orcidlink{0000-0003-1300-1451}}\thanks{Xiaotong
Shen is with the School of Statistics, University of Minnesota, MN,
55455 USA (email: xshen@umn.edu)}}
\date{}

  \title{\bf{Manifold-Aligned Generative Transport}\thanks{\magtfunding}}
  \maketitle
} \fi

\if0\anon
{
  \bigskip
  \bigskip
  \bigskip
  \begin{center}
    {\LARGE\bf Manifold-Aligned Generative Transport}
\end{center}
  \medskip
} \fi

\bigskip
\begin{abstract}
{Many high-dimensional datasets concentrate near a low-dimensional
structure embedded in the ambient space.  Generative models for such data
must control off-support mass while remaining computationally practical.
Diffusion models use iterative denoising at inference, whereas standard
normalizing flows require invertible, dimension-preserving maps.  We propose
\method{} (Manifold-Aligned Generative Transport), a direct transport from a
low-dimensional base distribution to the data space.  Its core objective
compares the data and generator-induced scores at a selected Gaussian
smoothing level.  A posterior
identity expresses this score through a latent conditional mean, which is
approximated by self-normalized importance sampling over a finite anchor set.
After training, generation requires one evaluation of the transport, whose
image also carries an intrinsic density with respect to manifold volume.  We
establish a minimax-optimal Wasserstein convergence rate for an explicitly
constructed localized spline-RePU coordinate transport estimator, and treat finite-anchor
approximation separately.
Experiments on synthetic, image,
and tabular benchmarks compare fidelity, support alignment, and sampling cost
with diffusion, flow-matching, and adversarial baselines.}
\end{abstract}

\noindent%
\medskip
\spacingset{1}

\section{Introduction}

{Generative models balance sample quality against the cost and
structure of the sampling mechanism.  Diffusion models produce samples by
iterative denoising, so inference usually requires repeated network
evaluations even with improved solvers or distillation
\citep{dhariwal2021beatgans,karras2022edm,song2023consistency,lu2022dpmsolver}.
Normalizing flows built from compositions of invertible maps permit direct
sampling and change-of-variables likelihoods
\citep{dinh2016realnvp,kingma2018glow,papamakarios2021nf}.  Continuous
normalizing flows instead define transport through an ODE
\citep{huang2020universalflows}.  Flow matching
trains the velocity field without solving that ODE during training, while
probability-flow ODEs, rectified flow, and stochastic interpolants provide
related continuous-time constructions
\citep{lipman2022flow,song2021score,liu2022flow,albergo2023stochastic}.
Sampling from these continuous-time models still requires numerical
integration.  In their standard ambient formulations, both discrete and
continuous normalizing flows are dimension preserving.}

{This dimension-preserving formulation motivates a different
construction in the manifold regime, where observations concentrate near a
lower-dimensional set embedded in a
high-dimensional ambient space, as often occurs for images, biological
measurements, and learned feature representations.  A generative model for
such data must capture both the geometry of the support and the distribution
of mass along it.  Existing geometry-aware diffusion and flow methods
incorporate manifold structure into their dynamics
\citep{debortoli2022riemannian,huang2022riemannian}.  Our aim is instead to
represent the lower-dimensional distribution directly by a map from intrinsic
coordinates, while retaining a smooth ambient-space training criterion and
direct sampling.}

{We propose \method{} (\emph{Manifold-Aligned Generative
Transport}) for this setting.  The model is a deterministic map
$h:\mathbb R^d\to\mathbb R^D$ with $d<D$; for a $d$-dimensional base distribution
$\pi$, the pushforward distribution $h_{\#}\pi$ is supported on
the image of $h$.  Training perturbs both the data and candidate distributions by
Gaussian noise, where they have smooth positive densities and comparable
scores.  A posterior identity expresses the candidate score through a latent
conditional mean.  Finite latent anchors and self-normalized importance
sampling make this conditional mean computable, following the connection
between denoising and score matching
\citep{hyvarinen2005scorematching,vincent2011scoredae}.  The smoothing is used
for training only.  Once $h$ is learned, an unsmoothed sample requires one draw
from the base distribution and one evaluation of $h$.}

{On the theoretical side, we establish a general single-level
transfer theorem that translates the empirical score equation at a fixed
Gaussian smoothing level into a Wasserstein error bound between the
corresponding unsmoothed data and generator distributions.  This result
highlights the roles of smoothing and underlying manifold geometry in
determining generation error.  We also analyze fixed-level score-matching
risk minimization over the localized spline-RePU coordinate class and obtain
a minimax-optimal nonasymptotic rate governed by the intrinsic dimension.
Supplementary results quantify the finite-anchor approximation error in the
fixed-proposal setting.}

{Our contributions are as follows.
\begin{enumerate}
\item \textbf{Methodology.} We introduce a dimension-mismatched transport
trained by a single-level denoising score criterion.  The posterior score
identity and finite latent anchors give a computable objective, while the
learned map provides direct sampling and an intrinsic density on its image;
see Table~\ref{tab:summary}.

\item \textbf{Theory.} For a regular $C^{\eta+1}$ transport class, the localized
spline-RePU coordinate estimator in Definition~\ref{def:repu-construction} attains the $W_2$ rate
$n^{-(\eta+1)/(2\eta+d^*)}$ when $\eta>2$ and $d^*>2$.  This rate is minimax
optimal.

\item \textbf{Computation.} We develop Monte Carlo, quasi--Monte Carlo, and
 MAP-based proposals for the same posterior-mean score estimator.
 Supplementary Appendix~\ref{sec:mcapprox} bounds the resulting finite-anchor
 approximation errors in the posterior mean, score, and loss.

\item \textbf{Empirical evidence.} Experiments on synthetic manifolds and on
image, tabular, and genomic data compare support fidelity, sample quality, and
sampling cost with diffusion, flow-matching, and adversarial baselines; see
Tables~\ref{tab:wd_time} and~\ref{tab:mnist_cifar10}.
\end{enumerate}}

{The remainder of the paper is organized as follows.
Section~\ref{sec:magnet} introduces the \method{} framework, including the
transport-based score identity, the resulting one-shot sampler, and practical
considerations for likelihood evaluation.  Section~\ref{sec:risk} develops
nonasymptotic risk bounds that relate single-level score estimation error to
Wasserstein generation accuracy.  Section~\ref{sec:experiments} presents
empirical results on synthetic manifolds and real image/tabular datasets.
Section~\ref{sec:discussion} concludes with a brief discussion.  The appendix
provides proofs, auxiliary lemmas, practical implementation details for
training \method{}, and additional experimental information; in particular,
Appendix~\ref{sec:mcapprox} presents practical Monte Carlo and quasi--Monte
Carlo schemes for approximating the conditional expectations in the \method{}
score estimator.}
\section{MAGT: Manifold-aligned generative transport}
\label{sec:magnet}

\subsection{Dimension alignment via perturbation}
\label{subsec:dim}

Consider generative modeling where observations $\Y_0 \in \mathbb{R}^D$ concentrate near a low-dimensional
manifold $\mathcal{M} \subset \mathbb{R}^D$ of intrinsic dimension $d \ll D$.
We aim to {learn} a deterministic transport (generator)
\( h : \mathbb{R}^{d} \to \mathbb{R}^{D} \) such that, for a latent variable
\( \U \) drawn from a {base probability law \( \pi \)}
(e.g., standard Gaussian on $\R^d$ or uniform on $[0,1]^d$),
the generated sample \( \Y_0 = h(\U) \) follows the data distribution
\( {P_{\Y_0}} \). {The map $h$ need not be invertible,
and the latent and data dimensions may differ, which is essential when the
target distribution concentrates on or near a lower-dimensional manifold.}
{Throughout, $P_X$ denotes the law of a random variable $X$,
whereas $p_X$ denotes its density when one exists.  We use $\pi$ for the
base law and, by the usual abuse of notation, also for its density.}

{If $P_{\Y_0}$ is supported on a manifold, it can be singular with respect to Lebesgue measure on $\mathbb{R}^D$,
so an ambient density and score for $\Y_0$ may be ill-defined.
\method{} resolves this by working at a fixed smoothing level $t$: we add Gaussian noise in the ambient space so that the corrupted
variable $\Y_t$ has an everywhere-positive density and a well-defined score $\nabla_{\y_t}\log p_{\Y_t}(\y_t)$.
This smoothed score admits a posterior/mixture representation in terms of $h$ and the base distribution $\pi$, which we approximate with a finite
set of latent ``anchors'' and then convert into a one-shot transport map.}

\paragraph*{Ambient Gaussian perturbations.}
We introduce a noise schedule $(\alpha_t,\sigma_t)_{t\in[0,1]}$ and define,
for each $t$, a perturbed observation
\begin{equation}
\label{eq:ambient-forward}
\Y_t = \alpha_t \Y_0 + \sigma_t \Z_t,
\qquad \Z_t \sim \mathcal{N}(0, I_D).
\end{equation}
This construction defines a dimension-preserving Gaussian corruption of $\Y_0$ in the ambient space, thereby linking the law of $\Y_t\in\R^D$ to that of the clean data $\Y_0$, which may be supported on a $d$-dimensional manifold.  All subsequent definitions use the marginal Gaussian corruption in \eqref{eq:ambient-forward}.

\subsection{Score matching and generator}
\label{subsec:magnet-score}

\begingroup

\method{} uses the perturbed data $\Y_t$ to define a score-matching objective
that learns $h$ through the candidate smoothed score
$\nabla_{\y_t}\log p_t^h(\y_t)$ in the ambient space.  The noise level is
controlled by $t$ through the schedule $(\alpha_t,\sigma_t)$.

For a candidate transport $h$, let
$\Y_t^h=\alpha_t h(\U)+\sigma_t\Z_t$.  Its conditional density given
$\U=\bm u$ is
$
p_t^h(\y_t\mid \bm u)=\phi\!\left(\y_t;\alpha_t h(\bm u),\sigma_t^2 I_D\right)$,
where $\phi(\y;\bm m,\bm\Sigma)$ denotes the density of
$\mathcal N(\bm m,\bm\Sigma)$ and $\pi$ is the base density of $\U$.
The candidate marginal density is the continuous mixture
$p_t^h(\y_t)=\int \phi\!\left(\y_t;\alpha_t h(\bm u),\sigma_t^2 I_D\right)\,\pi(\bm u)\,\mathrm{d}\bm u$.
Differentiating $\log p_t^h(\y_t)$ yields the mixture score
\begin{align}
\label{eq:mixture-score}
\nabla_{\y_t}\log p_t^h(\y_t)
&=
\mathbb{E}_h\!\left[
\nabla_{\y_t}\log p_t^h(\y_t\mid \U)\mid \Y_t^h=\y_t
\right]
=
\frac{1}{\sigma_t^2}\Bigl(\alpha_t\,\mathbb E_h[h(\U)\mid \Y_t^h=\y_t] - \y_t\Bigr),
\end{align}
where $\nabla_{\y_t}\log p_t^h(\y_t\mid \bm u)=-(\y_t-\alpha_t h(\bm u))/\sigma_t^2$ and the conditional expectation is under
$p_h(\bm u\mid \y_t)\propto \pi(\bm u)\,p_t^h(\y_t\mid \bm u)$; the
subscript in $\mathbb E_h$ refers to this candidate posterior.
This identity shows that the mixture score depends on the posterior mean 
$\mathbb E_h[h(\U)\mid \Y_t^h=\y_t]$, which encodes the geometry of the
latent space and the generator $h$.

\noindent\textbf{Transport-based score estimator.} In
\eqref{eq:mixture-score}, score estimation at noise level $t$ requires the
candidate posterior mean $\mathbb E_h[h(\U)\mid \Y_t^h=\y_t]$ under
$p_h(\bm u\mid \y_t)\propto\pi(\bm u)\,\phi\bigl(\y_t;\alpha_t h(\bm u),\sigma_t^2 I_D\bigr)$.
We approximate this conditional expectation by self-normalized importance
sampling. Specifically, let $\tilde\pi(\cdot\mid \y_t)$ be a proposal
distribution on the latent space, possibly depending on $\y_t$. We draw
$\U^{(1)},\dots,\U^{(K)}\stackrel{\text{i.i.d.}}{\sim}\tilde\pi(\cdot\mid \y_t)$, and form the unnormalized importance weights
\begin{equation}
\label{eq:is-weights}
\omega_t^{(j)}(\y_t)
\;:=\;
\frac{\pi\bigl(\U^{(j)}\bigr)}{\tilde\pi\bigl(\U^{(j)}\mid \y_t\bigr)}
\;\phi\!\left(\y_t;\alpha_t h\!\left(\U^{(j)}\right),\sigma_t^2 \bm{I}_D\right).
\end{equation}
Then, the candidate posterior mean is approximated by
\begin{equation}
\label{eq:posterior-mean-estimator}
\widetilde m_{t,K}(\y_t)
\;:=\;
\frac{\sum_{j=1}^K \omega_t^{(j)}(\y_t)\,h(\U^{(j)})}
     {\sum_{j=1}^K \omega_t^{(j)}(\y_t)}.
\end{equation}
We then define the corresponding transport-based score estimator
\begin{align}
\label{eq:score-estimator}
\widetilde s_{t,K}(\y_t;h,\pi,\tilde\pi)
\;:=\;
\frac{1}{\sigma_t^2}\Bigl(\alpha_t\,\widetilde m_{t,K}(\y_t)-\y_t\Bigr).
\end{align}
When $\tilde\pi(\cdot\mid \y_t)\equiv \pi(\cdot)$ (i.e., we sample anchors from the generative base), the importance ratio cancels and
$\omega_t^{(j)}(\y_t)\propto \phi\!\left(\y_t;\alpha_t h(\U^{(j)}),\sigma_t^2 \bm{I}_D\right)$, recovering the finite-mixture form.
This estimator is ``transport-based'' because it is an explicit functional of the learned map $h$ and the base distribution, with the conditional
expectation in~\eqref{eq:mixture-score} approximated by a weighted set of anchors $\{\U^{(j)}\}_{j=1}^K$.
\endgroup

\paragraph*{{Base law and computational proposal.}}
{For score-trained \method{} with a specified base law \(\pi\),
the generative model is \(P_h=h_\#\pi\).  The distribution
\(\tilde\pi(\cdot\mid y_t)\) in~\eqref{eq:is-weights} serves as a computational
proposal for approximating the posterior mean; the generative law remains \(P_h\).
The ratio \(\pi/\tilde\pi\) is defined with respect to a common dominating
 measure.  The MC, QMC, and MAP proposal constructions are described in
 Supplementary Appendix~\ref{sec:mcapprox}; the data-adapted implementations
 used in the experiments are described in Section~\ref{sec:experiments} and
 Supplementary Appendix~\ref{sec:supp-real-data}, including
 Algorithm~\ref{alg:magt-map-bank}.}

\paragraph*{Training loss.}
Given the transport-based score estimator~\eqref{eq:score-estimator}, we learn the transport map $h$ by minimizing a \emph{single-level} denoising score-matching objective. {For each training draw $\Y_0^{\,i}\sim P_{\Y_0}$, write $\y_0^{\,i}$ for its observed value,} draw $\z^{\,i}\sim\cN(0,I_D)$, and form the corrupted observation $\y_t^{\,i}=\alpha_t \y_0^{\,i}+\sigma_t \z^{\,i}$. We then define
\begin{equation}
\label{eq:ellK-def}
\ell_K(\y_t,\y_0;h)
:=
\Bigl\|
  \widetilde s_{t,K}(\y_t;h,\pi,\tilde\pi)
  - \nabla_{\y_t} \log p(\y_t\mid \y_0)
\Bigr\|_2^2,
\end{equation}
where $p(\y_t\mid \y_0)$ denotes the Gaussian density
$\cN(\alpha_t\y_0,\sigma_t^2 I_D)$.  {Write
$s_t(\cdot;h):=\nabla\log p_t^h$ for the exact candidate score.  The exact,
corruption-averaged loss used in the theory is
\begin{equation}
\label{eq:ideal-clean-loss}
\ell_t(x;h)
:=\mathbb E_Z\left\|
s_t(\alpha_t x+\sigma_tZ;h)+\frac Z{\sigma_t}
\right\|_2^2,\qquad
R_t(h):={\mathbb E_{P_{Y_0}}\ell_t(Y_0;h)}.
\end{equation}
For independent clean observations $X_1,\ldots,X_n$, write
\begin{equation}\label{eq:ideal-empirical-loss}
 L_{n,t}(h):=\frac1n\sum_{i=1}^n\ell_t(X_i;h).
\end{equation}
}
Given data $\{(\y_t^{\,i},\y_0^{\,i})\}_{i=1}^n$, we solve the empirical risk minimization problem
\begin{equation}
\label{risk-minimization}
L_{n,K}(h)
:=
\frac{1}{n}\sum_{i=1}^n \ell_K\!\bigl(\y_t^{\,i},\y_0^{\,i};h\bigr),
\qquad
\hat h_{\lambda}\in\arg\min_{h\in\mathcal{H}} L_{n,K}(h),
\end{equation}
over a prescribed hypothesis class $\mathcal{H}$.  The multiscale localized
RePU-coordinate class analyzed below is specified in
Section~\ref{subsec:neural-realization}.
When $\mathcal{H}$ is instantiated by a neural network family $\{h_{\theta}:\theta\in\Theta\}$, we equivalently optimize over parameters $\theta$ and obtain
$\hat\theta_{\lambda}\in\arg\min_{\theta\in\Theta} L_{n,K}(h_{\theta})$, with the learned transport defined as $\hat h_{\lambda}:=h_{\hat\theta_{\lambda}}$.
{For brevity, we suppress the dependence on $\hat\theta_{\lambda}$ and write
$\hat h_{\lambda}$ for the learned function.  Here
$\lambda := (t,d,K)$ collects the tuning parameters: $t$ controls the smoothing
level, $d$ controls the dimension of the latent representation, and $K$ controls
the accuracy of the Monte Carlo approximation in $\ell_K$.  Choosing $d$ too
small can omit intrinsic variation, whereas an unnecessarily large $d$ adds
redundant coordinates and increases estimation cost.  We select
$\hat\lambda=(\hat t,\hat d,\hat K)$ by minimizing a generative criterion (e.g.,
an estimated Wasserstein distance) on an independent validation set and report
the resulting generator $\hat h_{\hat\lambda}$.  Supplementary
Algorithm~\ref{algorithm_1} in Appendix~\ref{sec:implementation} displays the
search over $t$ for fixed $(d,K)$; the same
held-out criterion is used for the outer comparison over candidate values of
$d$ and $K$.}

{The well-specified theory below takes $d=d^*$, where $d^*$ denotes
the true intrinsic dimension.  If $d<d^*$, a
rank-$d$ image cannot represent a generic $d^*$-dimensional intrinsic density;
if $d>d^*$, the full-rank $d$-dimensional area formula and theorem class no
longer apply without an effective-rank or stratified analysis.  Thus $d$ is a
model dimension selected on held-out data in applications.}

{To make the role of this objective explicit, under the
well-specified model $P_{Y_0}=P_{h^*}$ let
\(p_t^*:=p_t^{h^*}\) be the smoothed data density and let \(p_t^h\) be the
smoothed density induced by a candidate \(h\). The Gaussian conditional-score
target in~\eqref{eq:ellK-def} is observed directly from \((\Y_0,\Y_t)\), and
\[
\E\!\left[\nabla_{\Y_t}\log p(\Y_t\mid\Y_0)\mid\Y_t\right]
=\nabla\log p_t^*(\Y_t).
\]
The exact risk in~\eqref{eq:ideal-clean-loss} satisfies the orthogonal decomposition
\[
R_t(h)-R_t(h^*)
=\E_{\Y\sim p_t^*}\!\left\|\nabla\log p_t^h(\Y)-\nabla\log p_t^*(\Y)\right\|_2^2.
\]
Thus the fixed-level denoising loss is an exact Fisher discrepancy between the data and generator laws after smoothing. The generator-induced representation
\(
s_t(y;h)=\{\alpha_t\E_h[h(\U)\mid \Y_t^h=y]-y\}/\sigma_t^2
\)
makes this discrepancy a functional of \(h\), which remains the statistical target. Section~\ref{sec:risk} relates the excess risk to unsmoothed \(W_2\) error, while Supplementary~Appendix~\ref{sec:mcapprox} controls the finite-anchor approximation in \(\ell_K\).}

\noindent{\textbf{Sample generation.}}
{Given the selected score-trained generator
\(\widehat h_{\widehat\lambda}\), draw \(u\sim\pi\) and output
\(\widehat h_{\widehat\lambda}(u)\).  Generation therefore requires one
evaluation of the learned map.}

\paragraph*{Intrinsic density and likelihood evaluation.}
Assume $h:\mathbb{R}^{d}\to\mathbb{R}^D$ is $C^1$ and has rank $d$ almost everywhere, and let $\M:=h(\mathcal{U})$ denote its image over a latent domain $\mathcal{U}\subset\mathbb{R}^{d}$ that contains the support of $\pi$.
Assume further that $\pi$ admits a density on $\mathcal{U}$ with respect to Lebesgue measure.
By the classical area formula~{\citep{federer1969gmt}}, the pushforward measure $h_{\#}\pi$ is absolutely continuous with respect to the $d$-dimensional Hausdorff measure $\mathcal H^{d}$ restricted to~$\M$, and for $\mathcal H^{d}$-a.e.\ $\y_0\in\M$,
\begin{equation}
\label{eq:manifold-likelihood}
 p_\M(\y_0)\;=\;\sum_{\bm u\in h^{-1}(\{\y_0\})} \frac{\pi(\bm u)}{J_h(\bm u)}\,,
 \qquad \text{where}\quad J_h(\bm u)
 \;:=\; \sqrt{\det\!\bigl(\,Dh(\bm u)^\top Dh(\bm u)\,\bigr)}.
\end{equation}
Here $p_\M$ is the intrinsic density of $\Y_0=h(\U)$ with respect to $\mathcal H^d|_{\M}$, $Dh(\bm u)\in\mathbb{R}^{D\times d}$ is the Jacobian matrix, and $J_h(\bm u)$ is its $d$-dimensional Jacobian.

If $h$ is injective on $\mathcal U$ (e.g., a $C^1$ embedding), then $h^{-1}(\{\y_0\})$ is a singleton for $\mathcal H^d$-a.e.\ $\y_0\in\M$, and \eqref{eq:manifold-likelihood} reduces to the familiar chart formula
\begin{equation}
\label{eq:manifold-ll}
\log p_\M\bigl(h(\bm u)\bigr)
\;=\;
\log \pi(\bm u)\;-\;\tfrac{1}{2}\log\det\!\bigl(Dh(\bm u)^\top Dh(\bm u)\bigr).
\end{equation}
More generally, if $h$ has bounded multiplicity, the sum in \eqref{eq:manifold-likelihood} contains finitely many terms; evaluating $p_\M(\y_0)$ requires identifying and summing all contributing preimages.
Equation~\eqref{eq:manifold-ll} gives the branchwise chart contribution associated with a specific preimage $u$.
For generated samples $\y_0=h(\bm u)$, this branchwise log-density is directly computable from $(\bm u,h(\bm u))$ and coincides with $\log p_\M(\y_0)$ whenever the fiber is a singleton (in particular, under injectivity).
If $h$ is not injective and additional preimages exist, \eqref{eq:manifold-ll} should be interpreted as a local contribution unless the remaining preimages are recovered and included in \eqref{eq:manifold-likelihood}.
For an observed point $\y_0\in\M$, evaluating $p_\M(\y_0)$ requires identifying one or more latent preimages solving $h(\bm u)=\y_0$ (or approximately minimizing $\|h(\bm u)-\y_0\|_2$), using a separate encoder or numerical optimization.  Sampling follows the forward construction $\bm u\sim\pi$, $\y_0=h(\bm u)$.

\subsection{Comparisons with diffusion and flow models}
\label{sec:comparison}

\begingroup

\method{} is computationally closest to one-shot latent-to-data generators
and is fitted by a distinct criterion.  It minimizes a denoising loss at
one Gaussian smoothing level, with the candidate score determined by the
transport and approximated through latent anchors.  Discrete normalizing flows
use an invertible, dimension-preserving map and train by change-of-
variables likelihood.  Continuous normalizing flows and most flow-matching
models define an ODE transport and require numerical integration at sampling
time.  Adversarial generators such as WGAN-GP sample in one pass and fit the
generator through a critic-based min--max objective.  Table~\ref{tab:summary}
summarizes the main distinctions.

Wasserstein GANs such as WGAN-GP estimate a Wasserstein integral probability metric through a
learned Lipschitz critic
\citep{arjovsky2017wasserstein,gulrajani2017improved}.  Inferential WGAN jointly
learns an encoder and generator in a Wasserstein adversarial formulation
\citep{chen2022inferential}, while the Wasserstein auto-encoder minimizes a
reconstruction objective with an aggregated-posterior penalty
\citep{tolstikhin2018wasserstein}.  Score-trained \method{} optimizes the
transport $h$ through a generator-induced score, without a critic or encoder.
Its population excess risk is exactly the Fisher divergence between the
smoothed data and candidate laws, which is linked to unsmoothed $W_2$ error in
Section~\ref{sec:risk}.  Finite-anchor posterior integration is the
corresponding additional training cost and is reported explicitly in
Section~\ref{sec:experiments}.

\begin{table}[H]
\centering

\caption{Method-level comparison.  NFE denotes the number of function
evaluations used to generate a sample.}
\label{tab:summary}
\footnotesize
\setlength{\tabcolsep}{2.0pt}
\begin{tabularx}{\linewidth}{@{}>{\raggedright\arraybackslash}p{0.15\linewidth}
>{\raggedright\arraybackslash}X >{\raggedright\arraybackslash}X
>{\raggedright\arraybackslash}X@{}}
\toprule
& \textbf{\method{}} & \textbf{Diffusion/flow matching}
& \textbf{Discrete normalizing flows} \\
\midrule
Training & Fixed-level denoising loss & Multi-level score or path loss
& Change-of-variables likelihood \\
Sampling & One evaluation of $h$ & Numerical SDE/ODE solve; multiple NFE
& One inverse-map evaluation \\
Support & Image of a dimension-mismatched map & Typically an ambient
approximation & Full-dimensional ambient density \\
Map constraint & Need not be invertible & Score or velocity field
& Invertible; dimensions match \\
Likelihood & Intrinsic on the image; smoothed mixture at $t>0$
& Model-dependent likelihood estimates & Exact ambient likelihood \\
\bottomrule
\end{tabularx}
\end{table}

\paragraph*{Sampling and support.}
After training, \method{} draws $U\sim\pi$ and returns $h(U)$ in one forward
evaluation.  Its generated distribution is supported on the image of $h$, and the
latent and ambient dimensions may differ.  Diffusion, flow-matching, and
continuous-flow samplers generally evaluate a learned field along a numerical
path, whereas discrete flows retain one-pass sampling by imposing
invertibility.  During training, each \method{} score evaluation includes the finite-anchor posterior
calculation in~\eqref{eq:score-estimator}.

\paragraph*{Likelihoods and evaluation.}
The two likelihood objects available to \method{} were defined above: the
area formula gives an intrinsic density on the learned image, and the anchor
mixture approximates the ambient density after smoothing at the selected noise
level.  The former is defined with respect to manifold volume, while the latter
is an ambient density indexed by $t$.  Section~\ref{sec:experiments} evaluates
image generation with FID and Inception Score under a common protocol.

\endgroup

\begingroup

\section{Theory}\label{sec:risk}
\newcommand{\thetah}{\theta_{\rm h}}
\newcommand{\thetao}{\theta_{\rm o}}

We study the law $P_h=h_\#\pi$ produced by the transport in \method{}.
We first give an architecture-independent score-ERM bound, then apply it
to a localized spline-RePU coordinate estimator to obtain a minimax-optimal
rate. Supplementary Appendices~\ref{sec:vs-main-proof} and~\ref{sec:mlsn-proof}
prove the two theorems in this order; the construction and technical
estimates follow in Appendices~\ref{sec:multiscale-tools}--\ref{sec:pilot-verification}.

\subsection{Model and population score risk}

Let $d=d^*$ and let $\mathcal U$ be a fixed compact, connected, smooth
$d^*$-dimensional manifold without boundary.
Throughout the asymptotics, $d^*$, $D$, the source atlas, and all regularity
constants are fixed.
For $r>0$, write $C^r(\mathcal U)$ for the order-$r$ H\"older class on
$\mathcal U$. Function-space conventions, geometric notation, and uniform
regularity bounds are given at the beginning of
Supplementary Appendix~\ref{sec:proofs}.

\begin{samepage}
\begin{assumption}[Regular source law]\label{G1}
For some $\eta>2$, the base density satisfies
$\pi\in C^\eta(\mathcal U)$ and
$0<\underline\pi\le\pi\le\overline\pi<\infty$ with respect to the fixed
Riemannian volume on $\mathcal U$.
\end{assumption}
\end{samepage}

\begin{samepage}
\begin{assumption}[Regular transport class]\label{G2}
Let $\beta=(\eta+2)/2$.  The truth class $\mathcal H_*$ is the full nonempty
class of $C^{\eta+1}(\mathcal U;\mathbb R^D)$ embeddings with uniformly bounded
$C^{\eta+1}$ norm whose $C^{\beta+1}$ norm, Jacobian singular values, and
image reach satisfy fixed uniform bounds with a common positive margin.
It has a nonempty $C^{\eta+1}$ interior.
It contains a strictly interior reference $h_0$ whose image is smooth and whose
induced density with respect to image volume is positive and $C^{\eta+1}$.
\end{assumption}
\end{samepage}

The fixed regularity bounds and the common interior margin are given in Supplementary
equations~\eqref{eq:uniform-regularity-bounds}--\eqref{eq:uniform-interiority-bounds}.
The Jacobian bounds prevent local collapse
or excessive expansion, while positive reach provides a common normal tube and
rules out excessive curvature or nearly touching image sheets.  The
$C^{\eta+1}$ smoothness controls the resolution bias, and the fixed interior margin
keeps projections and local perturbations of the truth admissible.

{For a truth $h^*\in\mathcal H_*$ and the VP schedule
$(\alpha_t^2,\sigma_t^2)=(1-t,t)$, write}
$\tau(t)=t/(1-t)$ and let $p_\tau^h$ be the density of
$h(U)+\sqrt\tau Z$, where $U\sim\pi$ and $Z\sim\mathcal N(0,I_D)$; set
$p_\tau^*:=p_\tau^{h^*}$.  With $t=\tau/(1+\tau)$, write
$R_\tau(h):=(1-t)R_t(h)$ for the same ideal risk in additive-noise units.

\subsection{General score-ERM transfer}

Fix a smoothing level $\tau=\tau(t)$.  Write $\theta=(\thetah,\thetao)$
for the parameters of a candidate transport $h_\theta$.
The hidden component $\thetah$ comprises the network weights and biases
that determine the feature map; the output component $\thetao\in\mathbb R^p$
gives the coefficients of these features in the law-coordinate representation.
For example, in the simple feature expansion $a^\top\rho(Wu+b)$ with
componentwise activation $\rho$, $\thetah=(W,b)$ determines the features
and $\thetao=a$ gives their linear coefficients.  In the spline-RePU
construction below, the output coefficients combine the learned features
to describe density perturbations and normal displacements.
Both components are fitted jointly.
Let $\Theta$ be a candidate parameter set with compact convex output
slices, and let
\[
 \widehat\theta=(\widehat\theta_{\rm h},\widehat\theta_{\rm o})
 \in\arg\min_{\theta\in\Theta}L_{n,t}(h_\theta),
 \qquad \widehat h:=h_{\widehat\theta},
\]
be a measurable joint minimizer of the exact empirical loss
in~\eqref{eq:ideal-empirical-loss}.  The population comparator
$\theta^\dagger=(\widehat\theta_{\rm h},\theta_{\rm o}^\dagger)$ minimizes
$R_\tau(h_\theta)$ over the output slice at the fitted hidden configuration,
with an interior output component.

Let $\bm M$ be a positive-definite comparison matrix and write
$\|a\|_{\bm M}^2=a^\top\bm M a$.  For deterministic constants
$b_\tau$ and $C_{\bm M}$, assume almost surely that
\begin{equation}\label{eq:general-chart-w2}
\begin{aligned}
 W_2(P_{h^*},P_{h_{\theta^\dagger}})&\le b_\tau,\\
 W_2(P_{\widehat h},P_{h_{\theta^\dagger}})
 &\le C_{\bm M}
 \|\widehat\theta_{\rm o}-\theta_{\rm o}^\dagger\|_{\bm M}.
\end{aligned}
\end{equation}
Let $g_{n,\tau}$ be the centered empirical gradient of $(1-t)L_{n,t}$ at
$\theta^\dagger$, and let $\widetilde{\bm H}_{n,\tau}$ be its Hessian averaged
along the segment from $\theta^\dagger$ to $\widehat\theta$.
Both derivatives are taken only in the output coordinates.
The comparator depends on the fitted hidden configuration, and the
comparison matrices may depend measurably on the fitting data, including
the empirical path Hessian. Explicit definitions and a uniform
conditional version are given in Supplementary Appendix~\ref{sec:vs-main-proof}.

\begin{samepage}
\begin{theorem}[Joint score-ERM to Wasserstein transfer]
\label{thm:local-erm-transfer}
Under the preceding setup and~\eqref{eq:general-chart-w2}, suppose
$W_2(P_{h^*},P_h)\le B$ over the candidate class, and let
$\bm A_\tau$ be a measurable positive-semidefinite matrix.  If,
on an event $\mathcal E_n$ of probability at least $1-\varepsilon_n$, the
population and empirical output minimizers are interior score roots,
$\widetilde{\bm H}_{n,\tau}$ is invertible, and
\begin{equation}\label{eq:general-path-stability}
 \|\widetilde{\bm H}_{n,\tau}^{-1}z\|_{\bm M}^2
 \le z^\top\bm A_\tau z,
 \qquad z\in\mathbb R^p,
\end{equation}
then the joint ERM estimator $\widehat h$ satisfies
\begin{equation}\label{eq:general-local-erm}
 \mathbb EW_2(P_{h^*},P_{\widehat h})
 \le b_\tau+C_{\bm M}
 \left[\mathbb E\left\{
 \mathbf1_{\mathcal E_n}g_{n,\tau}^\top\bm A_\tau
 g_{n,\tau}\right\}\right]^{1/2}
 +B\varepsilon_n.
\end{equation}
\end{theorem}
\end{samepage}

The first term $b_\tau$ bounds the population approximation error between
the true law and the law induced by the population comparator at smoothing
level $\tau$.  The second term controls statistical estimation error:
$\bm A_\tau$ accounts for the amplification of empirical score fluctuations
by the inverse Hessian, and $C_{\bm M}$ translates coefficient error into
Wasserstein distance.  The final term $B\varepsilon_n$ bounds the expected
error on the exceptional event $\mathcal E_n^c$, where the score-root or
stability conditions may fail.
For the spline-RePU estimator below, Supplementary
Appendix~\ref{sec:mlsn-proof} bounds these three terms uniformly over the
hidden configurations and applies this theorem to the joint ERM.
The statistical term is controlled by integrated second-moment and
weak-density estimates.

\subsection{Localized spline-RePU coordinate estimator}
\label{subsec:neural-realization}

A degree-$q$ rectified power unit (RePU) uses the activation
$\rho_q(x)=(x_+)^q$; ReLU corresponds to $q=1$.
We take an integer $q>\eta+4$ and put $M=q+1$.  For each
law chart $m$, choose order-$M$ density and normal spline multiwavelet banks
$\{\psi_{jk}^{\rm den},\psi_{jk}^{\rm nor}\}$
of size $Q_J\asymp2^{Jd^*}$ \citep{dahmen1999wavelets,jouini2007wavelet}.
Their chartwise spline pieces admit exact fixed-depth RePU realizations
\citep{li2020powernet}; Supplementary Appendix~\ref{sec:multiscale-tools} gives the
construction and dual filters.  Density prototypes have physical mean zero
by a fixed coarse correction; the boxes below apply to their independent
nonpivot coefficients.

\begin{definition}[Localized spline-RePU coordinate estimator]\label{def:repu-construction}
\leavevmode
\begin{enumerate}
\setlength{\itemsep}{3pt}
\setlength{\parsep}{0pt}
\item \textbf{Geometric reconstruction.}
Let $\bar P_m=\bar f_m\bar\nu_m$ be a reference law, where $\bar\nu_m$
is the volume measure on its support.  For an admissible mean-zero density
perturbation $r$ and normal displacement $v$, write
$\mathfrak h_m(r,v):\mathcal U\to\mathbb R^D$ for the embedding fixed by the
reconstruction in Supplementary Proposition~\ref{prop:canonical-law-sieve}.
It satisfies
\[
 \bigl(\mathfrak h_m(r,v)\bigr)_\#\pi
 =(I+v)_\#\{(\bar f_m+r)\bar\nu_m\}.
\]
\item \textbf{Basis and parameters.}
Let $\psi_{jk}^{\rm den}$ and $\psi_{jk}^{\rm nor}$ be fixed mean-zero
density splines and normal spline fields, with scale $j$ and within-scale
index $k$.  Write $\theta=(\thetah,\thetao)$: $\thetah\in\Theta_{\rm h}$
contains all weights and biases of a fixed-depth, fixed-width tangent
RePU network $F_{\thetah}$, while $\thetao$ contains the output coefficients
$\theta_{jk}^{\rm den},\theta_{jk}^{\rm nor}$.
The set $\Theta_{\rm h}$ is a fixed compact box containing $0$.
\item \textbf{Feature transformation.}
For fixed, sufficiently small $c_{\rm flow}>0$, let $U_{\thetah}^{\rm den}$
and $U_{\thetah}^{\rm nor}$ be the mass-preserving density and normal-field
pullback operators induced by the time-one flow of
$c_{\rm flow}\tau_JF_{\thetah}$ (Supplementary
equation~\eqref{eq:mlsn-geometric-pullbacks}).  Their action on the fixed
splines gives the transformed features in
\[
 h_{m,\theta}:=\mathfrak h_m\left(
 \sum_{j,k}\theta_{jk}^{\rm den}U_{\thetah}^{\rm den}\psi_{jk}^{\rm den},
 \sum_{j,k}\theta_{jk}^{\rm nor}U_{\thetah}^{\rm nor}\psi_{jk}^{\rm nor}\right),
 \qquad \tau_J=c_t2^{-2J}.
\]
\item \begin{minipage}[t]{\linewidth}\textbf{Candidate class.}
The class $\mathcal H_J^{\rm NN}$ consists of these maps with
$|\theta_{jk}^{\rm den}|\le A2^{-j(\beta+d^*/2)}$ and
$|\theta_{jk}^{\rm nor}|\le A2^{-j(\beta+1+d^*/2)}$, with fixed coarse
constraints enforcing positivity and the normal-tube condition.
\end{minipage}
\end{enumerate}
\end{definition}

For the candidate transports in $\mathcal H_J^{\rm NN}$, regularity bounds
are derived from the fixed reference atlas and coefficient boxes in
Supplementary Proposition~\ref{prop:canonical-law-sieve}.  These bounds may be
looser than the truth-class bounds but remain independent of $n$ and $J$.
The candidate smoothness bound supplies the uniform derivatives needed for the
local score-information analysis.

\paragraph*{Design rationale.}
The geometry-compatible class uses high-order RePU units for score-Hessian
smoothness and spline multiresolution for approximation and scale-wise
complexity.  The
$O(\tau_J)$ scaling, coefficient constraints, and pilot localization jointly
preserve the Jacobian and reach bounds.  The proof profiles over the
output-coordinate components of $\theta$ while allowing its hidden components
to be optimized jointly.  This avoids requiring identifiable raw weights and
yields the sharp bias--variance balance.

Here $m\in\{1,\ldots,M_0\}$ indexes one of the finitely many reference law
charts constructed in Supplementary Proposition~\ref{prop:canonical-law-sieve}.
Let $\widehat m$ and $\mathcal S_n$ be the chart selector and safety event
obtained from the independent pilot in Supplementary
Lemma~\ref{lem:vs-observable-pilot}.  Set
$\Theta_{n,J}^{\rm NN}:=\Theta_{n,J,\widehat m}^{\rm NN}$ using
Supplementary equation~\eqref{eq:mlsn-joint-domain}; this domain includes both
the coefficient boxes and the fixed coarse constraints in
Definition~\ref{def:repu-construction}.
For $\tau_J=c_t2^{-2J}=\tau(t_J)$, if $\mathcal S_n$ holds and
$\Theta_{n,J}^{\rm NN}\ne\varnothing$, define a measurable joint ERM
on the final subsample by
\[
 \widehat\theta_J
 \in\arg\min_{\theta\in\Theta_{n,J}^{\rm NN}}
 L_{n,t_J}(h_{\widehat m,\theta}),
 \qquad \widehat h_J^{\rm NN}:=h_{\widehat m,\widehat\theta_J}.
\]
Otherwise return a fixed reference transport $h_{\rm ref}$.
Here $n$ is the final-subsample size.  The precise selection and fallback
rules are given in Supplementary Appendix~\ref{sec:repu-class}.

\begin{theorem}[MAGT minimax rate]
\label{thm:acc}
Suppose Assumptions~\ref{G1}--\ref{G2} hold, use the estimator in
Definition~\ref{def:repu-construction}, and let $d^*>2$.  There is $c_*>0$
such that, for every fixed $c_t\in(0,c_*]$, there is a constant
$C=C(c_t)>0$ for which, with
\[
 J_n=\left\lfloor\frac{\log_2 n}{2\eta+d^*}\right\rfloor,
 \qquad \tau(t_{J_n})=c_t2^{-2J_n},
\]
for all sufficiently large $n$,
\[
 \sup_{h^*\in\mathcal H_*}\mathbb E W_2(P_{h^*},P_{\widehat h_{J_n}^{\rm NN}})
 \le C n^{-(\eta+1)/(2\eta+d^*)}.
\]
\end{theorem}

Applying Theorem~\ref{thm:local-erm-transfer} in Supplementary
Appendix~\ref{sec:mlsn-proof} gives, under the resolution conditions of
Lemmas~\ref{lem:vs-observable-pilot} and~\ref{lem:pm-newton-margins},
\[
 \mathbb E W_2(P_{h^*},P_{\widehat h_J^{\rm NN}})
 \lesssim 2^{-J(\eta+1)}+n^{-1/2}2^{J(d^*-2)/2}+n^{-2}.
\]
The leading terms are approximation bias and statistical estimation error;
the last term covers exceptional events. Uniform observation-space and
fixed-truth dual-process bounds give the stochastic order without a
logarithmic loss. Balancing the leading terms gives the choice in
Theorem~\ref{thm:acc}.

\paragraph*{Remark on choosing the smoothing level.}
Smaller $t$ retains finer-scale distinctions in the score criterion,
but can concentrate the latent posterior weights and make the finite-anchor
score estimate less stable.  Larger $t$ can stabilize posterior averaging, at
the cost of attenuating these distinctions.  We therefore select $t$ jointly
with the latent dimension $d$ and anchor budget $K$ using held-out performance.
For the exact-score estimator, Theorem~\ref{thm:acc} uses
$\tau(t_J)=c_t2^{-2J}$; fixed-proposal finite-anchor approximation is analyzed
separately in Supplementary Appendix~\ref{sec:mcapprox}.

\paragraph*{Minimax comparison.}
Theorem~\ref{thm:acc} attains the minimax exponent established by
\citet{nilesweed2022minimax} for Wasserstein estimation on a fixed support.
Supplementary Proposition~\ref{prop:vs-minimax-transfer} transfers the
corresponding lower bound to an interior subclass of regular transports,
establishing minimax optimality over this subclass.
As summarized in Table~\ref{tab:theory-generative-summary}, MAGT achieves an
intrinsic-dimension rate in the stronger $W_2$ metric without logarithmic
loss.  Manifold-adaptive diffusion attains the same smoothness--dimension
form of the exponent for $W_1$, up to logarithmic factors.
In contrast, the displayed full-dimensional KDE flow-matching bound involves
the ambient dimension $D$, highlighting the role of manifold structure in
reducing dimensional dependence.  These distinctions are specific to the
assumptions and estimator classes listed in the table.

\begin{table}[H]
\centering
\caption{Representative statistical guarantees.  Rates are stated in the
$d^*>2$ regime used above, and logarithmic factors are displayed.}
\label{tab:theory-generative-summary}
\renewcommand{\baselinestretch}{1}
\footnotesize
\setlength{\tabcolsep}{3.5pt}
\renewcommand{\arraystretch}{1.2}
\begin{tabularx}{\linewidth}{@{}>{\raggedright\arraybackslash}p{0.24\linewidth}
>{\raggedright\arraybackslash}p{0.26\linewidth}
>{\raggedright\arraybackslash}p{0.25\linewidth}
>{\raggedright\arraybackslash}X@{}}
\toprule
\textbf{Method} & \textbf{Rate and metric} & \textbf{Distribution class} & \textbf{Estimator}\\
\midrule
Manifold-adaptive diffusion\newline \citep{tang2024adaptivity} &
$W_1:$\newline $\widetilde O\{n^{-(\alpha+1)/(2\alpha+d^*)}\}$ &
$\alpha$-smooth density on a boundaryless manifold & Reverse diffusion\\
\addlinespace[6pt]
KDE flow matching\newline \citep{kunkeltrabs2025fm_kde} &
$W_1:$\newline $(n/\log^2 n)^{-(\alpha+1)/(2\alpha+D)}$ &
Compactly supported Besov density; $0<\alpha\le1$ & KDE velocity field\\
\addlinespace[6pt]
MAGT\newline Theorem~\ref{thm:acc} &
$W_2:$\newline $n^{-(\eta+1)/(2\eta+d^*)}$ &
Regular $C^{\eta+1}$ transport; positive intrinsic density &
Localized spline-RePU coordinate estimator\\
\bottomrule
\end{tabularx}
\end{table}

\paragraph*{Proof strategy.}
Condition on the independent pilot and compare the joint ERM with the
population output minimizer at its fitted hidden configuration. Interior
score roots and output curvature give the path identity in
Theorem~\ref{thm:local-erm-transfer}. A common energy and weak-density
argument bounds the population bias and the empirical--population
coordinate second moment, uniformly over hidden configurations.
The statistical bound includes a same-order approximation remainder
at the chosen resolution. Bounded support controls the exceptional
events. Substitution into Theorem~\ref{thm:local-erm-transfer} and
balancing the resolution complete the proof in Supplementary
Appendix~\ref{sec:mlsn-proof}.

\endgroup

\section{\texorpdfstring{{Experiments}}{Experiments}}
\label{sec:experiments}

{We evaluate three aspects of generation: distributional fidelity,
sampling cost, and concentration near the data support.  The synthetic suite
contains manifolds with known geometry, making it possible to report both
Wasserstein distance and off-manifold mass.  The real-data suite covers images,
tabular measurements, and genomic sequences, for which we use FID, Inception
Score, or Wasserstein distance as appropriate.  Comparisons include diffusion
samplers, flow matching, and WGAN-GP, spanning iterative score or velocity
models and a direct adversarial generator.  Sampling times exclude training and
metric evaluation and are measured for a common generated batch size within
each benchmark.}

\subsection{Synthetic benchmarks: low-dimensional manifolds}
\label{sec:experiments-synth}

We first evaluate on controlled synthetic distributions where both the ground-truth distribution and (for manifold datasets) the underlying manifold are known.
Table~\ref{tab:wd_time} reports six representative benchmarks: four non-Gaussian distributions in $\R^2$ (\textit{rings2d}, \textit{spiral2d}, \textit{moons2d}, \textit{checker2d}) and two thin manifolds embedded in $\R^3$ (\textit{helix3d}, \textit{torus3d}).
{For the parameterized manifold datasets, we sample latent
parameters, map them through a nonlinear embedding
into the ambient space, and add small i.i.d.\ Gaussian jitter;} full simulation
details are provided in Supplementary Appendix~\ref{sec:supp-toy-details}.

\noindent \textbf{Data splits and evaluation protocol.}
Across all synthetic benchmarks, the training set contains $n=10{,}000$ observations.
We use an independent validation set of size $5{,}000$ for hyperparameter selection and a held-out test set for computing $W_2$ and the off-manifold rate.
Each method generates $10{,}000$ samples for evaluation.

\noindent \textbf{Experimental configuration and baselines.}
{For each dataset, we use the latent dimension $d$ shown in
Table~\ref{tab:wd_time} and a standard Gaussian base.  The \method{}
transport is a five-hidden-layer ReLU MLP of width 512, trained with the
\method{}--MC score estimator and proposal $\tilde\pi=\pi$.  We select
$t\in\{0.1,0.2,\ldots,0.9\}$ by the held-out generative metric in
Supplementary Algorithm~\ref{algorithm_1} in
Appendix~\ref{sec:implementation} and use
$K=1024$ unless stated otherwise.  Figure~\ref{fig:knt} varies $K$ and $t$.

The diffusion score network uses the same MLP backbone with a 32-dimensional
time embedding.  DDIM~\citep{song2020ddim} uses 1000 steps,
$\eta_{\rm DDIM}=1.0$, and
$t\in[0.05,0.90]$.  DPM-Solver++~\citep{lu2022dpmsolver} uses either 20
first-order evaluations or 40 second-order midpoint evaluations.  Flow
matching uses the same backbone and time embedding; {midpoint integration with
step size $0.05$ uses 20 steps (40 velocity-field evaluations) over $t\in[0,1]$.}

For WGAN-GP~\citep{gulrajani2017improved}, the generator and critic are
five-hidden-layer ReLU MLPs of width 512.  Training uses five critic updates per
generator update, gradient-penalty coefficient 10, and Adam with learning rate
$10^{-4}$ and $(\beta_1,\beta_2)=(0.5,0.999)$.  M-DDIM reuses the fitted
transport and anchor-based score for 205 DDIM refinement steps from $0.90$ to
$0.05$ ($\eta_{\rm DDIM}=1.0$), with one cached anchor bank used throughout.

Table~\ref{tab:wd_time} reports lower $W_2$ for \method{} than for the diffusion
samplers on all six benchmarks and lower $W_2$ than flow matching on five; the
two methods are nearly tied on \textit{checker2d}.  On \textit{helix3d} and
\textit{torus3d}, the \method{} off-manifold rates are $0.094$ and $0.034$,
compared with $0.167$ and $0.556$ for flow matching and
$0.333$--$0.386$ and $0.634$--$0.670$ for the diffusion samplers.  \method{}
uses one transport evaluation and takes about $10^{-3}$ seconds for $10{,}000$
samples in this experiment.}

\begin{table}[t]
\caption{Empirical $W_2$, off-manifold rate (fraction of samples with distance $>0.1$), and wall-clock sampling time (seconds) to generate $10{,}000$ samples. Parentheses report standard deviations across runs. Boldface indicates the best-performing method for each metric.}
\label{tab:wd_time}
\centering
\scriptsize
\resizebox{\textwidth}{!}{%
\begin{tabular}{l*{6}{cc}}
\toprule
& \multicolumn{2}{c}{rings2d ($d=1$)}
& \multicolumn{2}{c}{spiral2d ($d=1$)}
& \multicolumn{2}{c}{moons2d ($d=2$)}
& \multicolumn{2}{c}{checker2d ($d=2$)}
& \multicolumn{2}{c}{helix3d ($d=1$)}
& \multicolumn{2}{c}{torus3d ($d=2$)} \\
\cmidrule(lr){2-3}\cmidrule(lr){4-5}\cmidrule(lr){6-7}\cmidrule(lr){8-9}\cmidrule(lr){10-11}\cmidrule(lr){12-13}
& $W_2$ & Out rate
& $W_2$ & Out rate
& $W_2$ & Out rate
& $W_2$ & Out rate
& $W_2$ & Out rate
& $W_2$ & Out rate \\
\midrule

\makecell{\method{}\\\scriptsize{NFE= 1}}
& \makecell{\textbf{0.0592}\\\scriptsize(0.0070)} & \makecell{\textbf{0.1406}\\\scriptsize(0.0255)}
& \makecell{\textbf{0.0767}\\\scriptsize(0.0185)} & \makecell{0.4150\\\scriptsize(0.1501)}
& \makecell{\textbf{0.0331}\\\scriptsize(0.0036)} & \makecell{\textbf{0.0360}\\\scriptsize(0.0103)}
& \makecell{0.0608\\\scriptsize(0.0040)} & \makecell{\textbf{0.0064}\\\scriptsize(0.0029)}
& \makecell{\textbf{0.0425}\\\scriptsize(0.0022)} & \makecell{\textbf{0.0938}\\\scriptsize(0.0501)}
& \makecell{\textbf{0.0717}\\\scriptsize(0.0049)} & \makecell{\textbf{0.0341}\\\scriptsize(0.0182)} \\
\textit{Time}
& \multicolumn{2}{c}{\makecell{\textbf{0.001}}}
& \multicolumn{2}{c}{\makecell{\textbf{0.001}}}
& \multicolumn{2}{c}{\makecell{\textbf{0.001}}}
& \multicolumn{2}{c}{\makecell{\textbf{0.001}}}
& \multicolumn{2}{c}{\makecell{\textbf{0.001}}}
& \multicolumn{2}{c}{\makecell{\textbf{0.001}}} \\
\cmidrule(lr){1-13}

\makecell{M-DDIM\\\scriptsize{NFE= 205}}
& \makecell{0.0613\\\scriptsize(0.0100)} & \makecell{0.2014\\\scriptsize(0.0211)}
& \makecell{0.0836\\\scriptsize(0.0054)} & \makecell{\textbf{0.3740}\\\scriptsize(0.0927)}
& \makecell{0.0733\\\scriptsize(0.0219)} & \makecell{0.2199\\\scriptsize(0.0709)}
& \makecell{0.0640\\\scriptsize(0.0063)} & \makecell{0.0088\\\scriptsize(0.0040)}
& \makecell{0.0536\\\scriptsize(0.0086)} & \makecell{0.2812\\\scriptsize(0.0740)}
& \makecell{0.0922\\\scriptsize(0.0123)} & \makecell{0.1208\\\scriptsize(0.0223)} \\
\textit{Time}
& \multicolumn{2}{c}{0.615}
& \multicolumn{2}{c}{0.613}
& \multicolumn{2}{c}{0.668}
& \multicolumn{2}{c}{0.615}
& \multicolumn{2}{c}{0.664}
& \multicolumn{2}{c}{0.616} \\
\cmidrule(lr){1-13}

\makecell{DPM++ (1s)\\\scriptsize{NFE= 20}}
& \makecell{0.1052\\\scriptsize(0.0228)} & \makecell{0.6905\\\scriptsize(0.0283)}
& \makecell{0.1992\\\scriptsize(0.0300)} & \makecell{0.7584\\\scriptsize(0.0261)}
& \makecell{0.1199\\\scriptsize(0.0083)} & \makecell{0.1672\\\scriptsize(0.0293)}
& \makecell{0.0944\\\scriptsize(0.0213)} & \makecell{0.0228\\\scriptsize(0.0205)}
& \makecell{0.1115\\\scriptsize(0.0162)} & \makecell{0.3338\\\scriptsize(0.0230)}
& \makecell{0.1650\\\scriptsize(0.0223)} & \makecell{0.6453\\\scriptsize(0.0656)} \\
\textit{Time}
& \multicolumn{2}{c}{0.153}
& \multicolumn{2}{c}{0.132}
& \multicolumn{2}{c}{0.164}
& \multicolumn{2}{c}{0.141}
& \multicolumn{2}{c}{0.183}
& \multicolumn{2}{c}{0.137} \\
\cmidrule(lr){1-13}

\makecell{DPM++ (2m)\\\scriptsize{NFE= 40}}
& \makecell{0.1057\\\scriptsize(0.0239)} & \makecell{0.6922\\\scriptsize(0.0188)}
& \makecell{0.1976\\\scriptsize(0.0290)} & \makecell{0.7644\\\scriptsize(0.0259)}
& \makecell{0.1222\\\scriptsize(0.0115)} & \makecell{0.1614\\\scriptsize(0.0200)}
& \makecell{0.0914\\\scriptsize(0.0216)} & \makecell{0.0257\\\scriptsize(0.0195)}
& \makecell{0.1056\\\scriptsize(0.0128)} & \makecell{0.3356\\\scriptsize(0.0275)}
& \makecell{0.1640\\\scriptsize(0.0221)} & \makecell{0.6344\\\scriptsize(0.0588)} \\
\textit{Time}
& \multicolumn{2}{c}{0.300}
& \multicolumn{2}{c}{0.256}
& \multicolumn{2}{c}{0.327}
& \multicolumn{2}{c}{0.261}
& \multicolumn{2}{c}{0.362}
& \multicolumn{2}{c}{0.276} \\
\cmidrule(lr){1-13}

\makecell{DDIM\\\scriptsize{NFE= 1000}}
& \makecell{0.1107\\\scriptsize(0.0172)} & \makecell{0.7057\\\scriptsize(0.0188)}
& \makecell{0.1494\\\scriptsize(0.0292)} & \makecell{0.7767\\\scriptsize(0.0254)}
& \makecell{0.0920\\\scriptsize(0.0311)} & \makecell{0.1843\\\scriptsize(0.0146)}
& \makecell{0.1108\\\scriptsize(0.0276)} & \makecell{0.0860\\\scriptsize(0.0407)}
& \makecell{0.0883\\\scriptsize(0.0162)} & \makecell{0.3861\\\scriptsize(0.0300)}
& \makecell{0.1603\\\scriptsize(0.0269)} & \makecell{0.6699\\\scriptsize(0.0658)} \\
\textit{Time}
& \multicolumn{2}{c}{3.190}
& \multicolumn{2}{c}{2.720}
& \multicolumn{2}{c}{3.392}
& \multicolumn{2}{c}{2.760}
& \multicolumn{2}{c}{3.709}
& \multicolumn{2}{c}{2.934} \\
\cmidrule(lr){1-13}

\makecell{WGAN-GP\\\scriptsize{NFE= 1}}
& \makecell{0.2904\\\scriptsize(0.1112)} & \makecell{0.7316\\\scriptsize(0.0622)}
& \makecell{0.3603\\\scriptsize(0.0472)} & \makecell{0.8110\\\scriptsize(0.0769)}
& \makecell{0.4562\\\scriptsize(0.1310)} & \makecell{0.9292\\\scriptsize(0.0829)}
& \makecell{0.2470\\\scriptsize(0.0646)} & \makecell{0.3107\\\scriptsize(0.3267)}
& \makecell{0.2352\\\scriptsize(0.1335)} & \makecell{0.8212\\\scriptsize(0.1933)}
& \makecell{0.4036\\\scriptsize(0.0285)} & \makecell{0.7916\\\scriptsize(0.0337)} \\
\textit{Time}
& \multicolumn{2}{c}{\makecell{\textbf{0.001}}}
& \multicolumn{2}{c}{\makecell{\textbf{0.001}}}
& \multicolumn{2}{c}{\makecell{\textbf{0.001}}}
& \multicolumn{2}{c}{\makecell{\textbf{0.001}}}
& \multicolumn{2}{c}{\makecell{\textbf{0.001}}}
& \multicolumn{2}{c}{\makecell{\textbf{0.001}}} \\
\cmidrule(lr){1-13}

\makecell{FM\\\scriptsize{{NFE= 40}}}
& \makecell{0.0712\\\scriptsize(0.0026)} & \makecell{0.5838\\\scriptsize(0.0142)}
& \makecell{0.0954\\\scriptsize(0.0030)} & \makecell{0.6422\\\scriptsize(0.0162)}
& \makecell{0.0470\\\scriptsize(0.0191)} & \makecell{0.0556\\\scriptsize(0.0168)}
& \makecell{\textbf{0.0592}\\\scriptsize(0.0015)} & \makecell{0.0147\\\scriptsize(0.0047)}
& \makecell{0.0533\\\scriptsize(0.0072)} & \makecell{0.1673\\\scriptsize(0.0294)}
& \makecell{0.1130\\\scriptsize(0.0074)} & \makecell{0.5557\\\scriptsize(0.0268)} \\
\textit{Time}
& \multicolumn{2}{c}{0.091}
& \multicolumn{2}{c}{0.090}
& \multicolumn{2}{c}{0.091}
& \multicolumn{2}{c}{0.090}
& \multicolumn{2}{c}{0.099}
& \multicolumn{2}{c}{0.092} \\
\bottomrule
\end{tabular}%
}
\end{table}
\FloatBarrier

{In Table~\ref{tab:wd_time}, \method{} and WGAN-GP generate with
one map evaluation, whereas the diffusion and ODE baselines evolve ambient
samples over multiple steps.  M-DDIM applies 205 refinement evaluations using
the fitted \method{} transport and anchor-induced score.  Direct \method{}
records lower $W_2$ than M-DDIM on all six benchmarks.}

\begin{figure}[H]
 \centering
 \safeincludegraphics[width=0.75\linewidth]{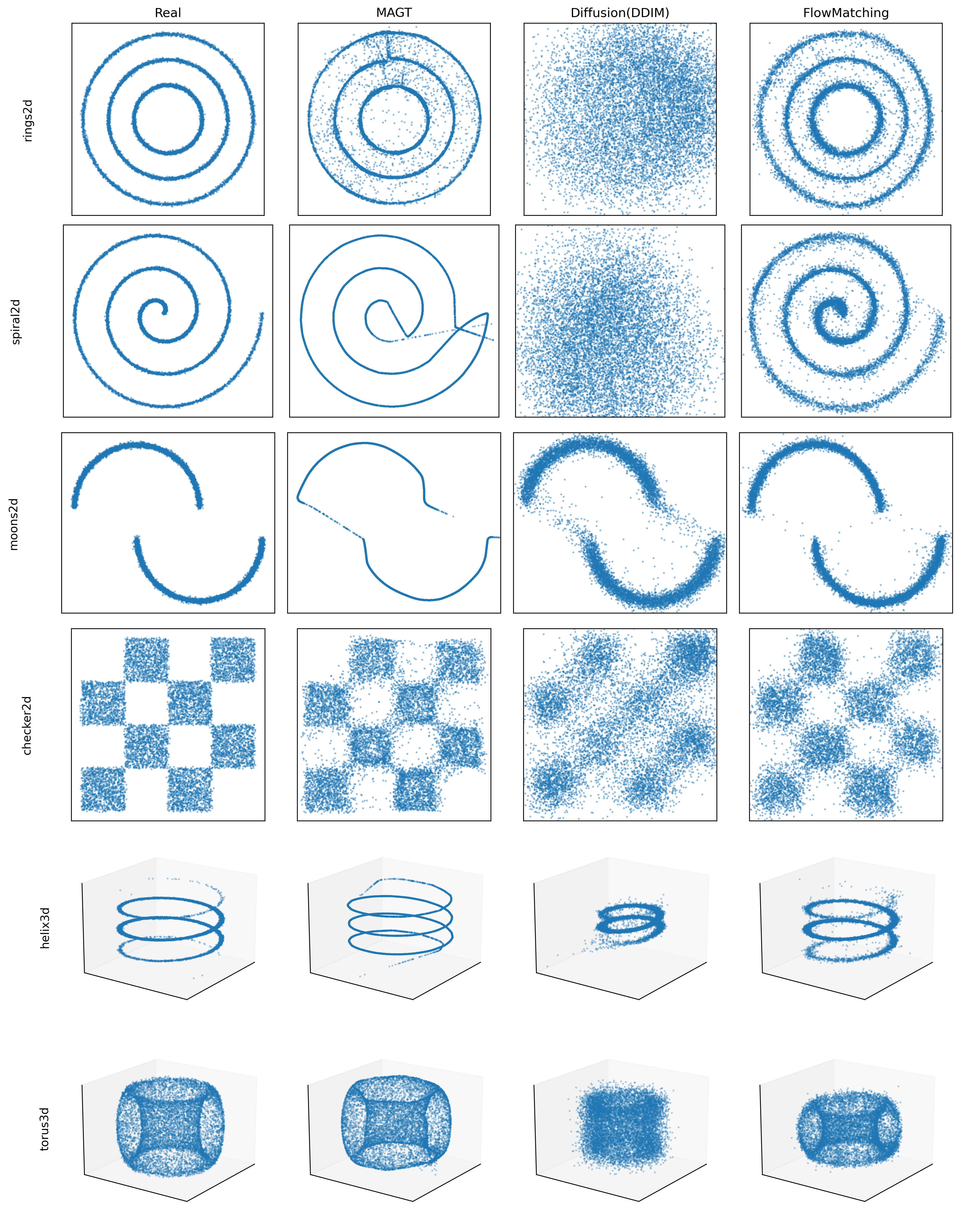}
 \caption{Qualitative comparison of generative models on six synthetic manifolds.}
 \label{fig:toyex}
\end{figure}

{Figure~\ref{fig:toyex} shows generated samples on the six synthetic benchmarks.
The samples follow the target supports, with most visible deviations near the
boundaries.}

\begin{figure}[H]
 \centering
 \safeincludegraphics[width=0.95\linewidth]{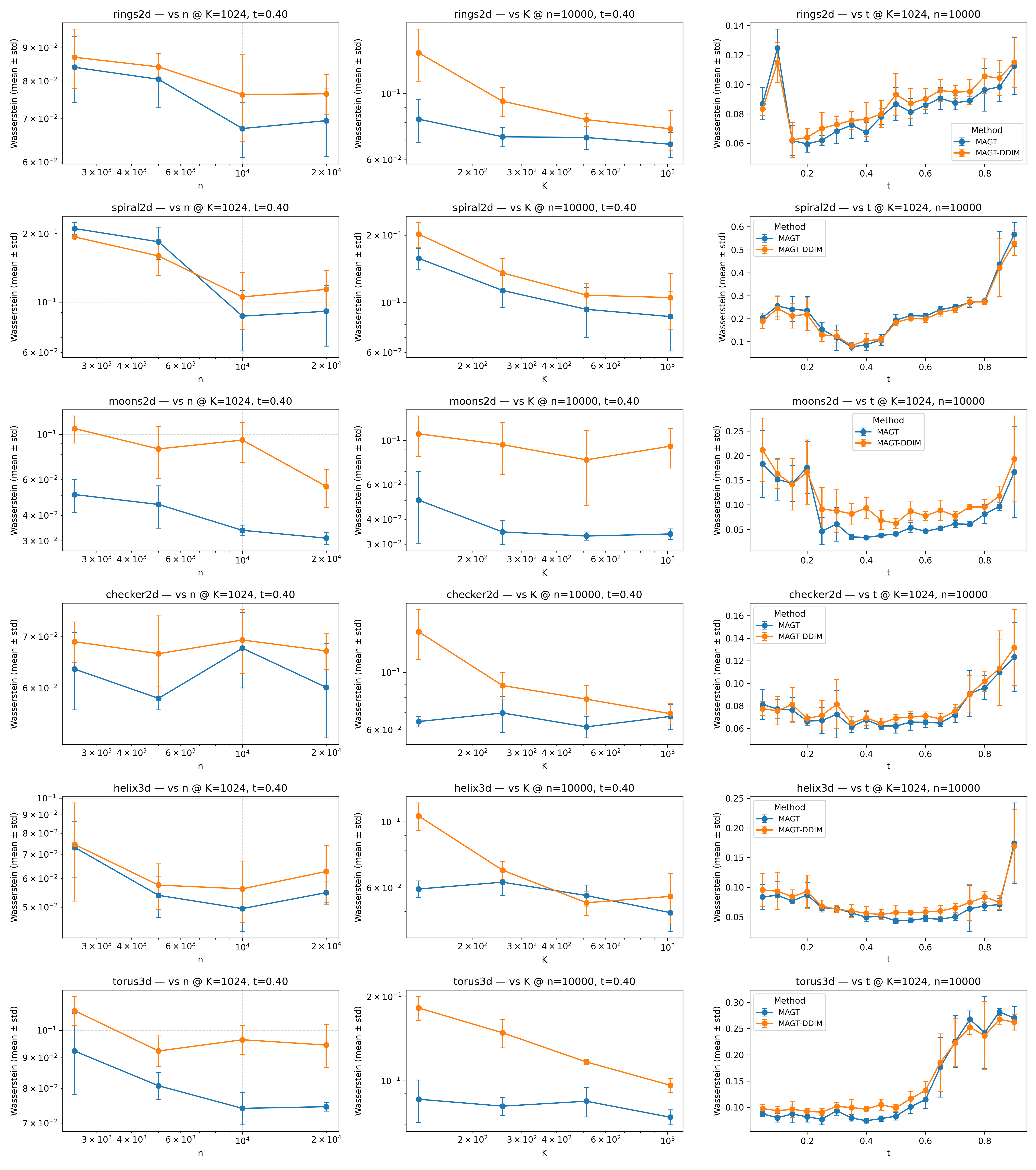}
 \caption{Effect of sample size $n$, anchor count $K$, and smoothing level $t$ on \method{} and \method{}--DDIM across six synthetic benchmarks. Curves report Wasserstein distance.}
 \label{fig:knt}
 \end{figure}
{Figure~\ref{fig:knt} examines the effects of training sample size
$n$, anchor count $K$, and smoothing level $t$.  Increasing $n$ or $K$ generally reduces
$W_2$ across the six tasks.  The dependence on $t$ is nonmonotone: small $t$
amplifies variation in the posterior-score estimate, while large $t$ removes
fine manifold structure.  The best values occur at intermediate bandwidths,
consistent with the trade-off in Section~\ref{sec:risk}; sampling uses the
fitted transport directly.}

\subsection{Real-data benchmarks}
\label{sec:experiments-real}

\paragraph*{Datasets.}
We evaluate the proposed method on image and tabular/high-dimensional benchmarks.
\textbf{MNIST} \citep{lecun1998gradient} contains $28\times 28$ grayscale handwritten digits (60{,}000 training, 10{,}000 test).
{\textbf{CIFAR-10} \citep{krizhevsky2009learning} contains $50{,}000$ training and $10{,}000$ test color images across ten classes.}

For tabular data, we use \textbf{Superconduct} \citep{hamidieh2018superconduct}, which contains $21{,}263$ samples with $D=81$ features.
We model the standardized feature vectors and evaluate distributional discrepancy between generated samples and a held-out test split (20\% of the data).

We also consider \textbf{Genomes} from the 1000 Genomes Project \citep{1000genomes2015}.
Following prior studies \citep{yelmen2023deep,ahronoviz2024genome}, we focus on $D=10{,}000$ biallelic SNPs on chromosome~6 (a 3~Mbp region including HLA genes), encoded as binary sequences.
We use 4,004 genomes for training and 1,002 for testing (stratified by continental group).

\paragraph*{Baselines and Model architectures.} We compare against DDIM~\citep{song2020ddim} and flow matching (\textsc{FM})~\citep{lipman2022flow,liu2022flow} as iterative baselines, and WGAN-GP~\citep{gulrajani2017improved} as a one-shot baseline. 
{For MNIST and CIFAR-10, we use the \method{}--MAP version;
the algorithm and implementation details are given in Supplementary
Algorithm~\ref{alg:magt-map-bank} and Appendix~\ref{sec:supp-real-data}.}
\textbf{Superconduct:} all methods use a 5-hidden-layer ReLU MLP of width 512; diffusion and flow matching additionally concatenate a 32-dimensional time embedding, and the GAN generator takes a 64-dimensional latent input. 
\textbf{Genomes:} diffusion and \method{} use closely matched 1D U-Net backbones following genomic diffusion~\citep{kenneweg2025generating}; \method{} further includes a linear projection from the low-dimensional latent input to the channel dimension expected by the U-Net.

{Detailed configurations and sensitivity analyses are reported in
Supplementary Appendix~\ref{sec:supp-real-data}.}

\paragraph*{Quantitative results and interpretation.}
{As summarized in Table~\ref{tab:mnist_cifar10}, the
\method{}--MAP configuration obtains an MNIST FID of $3.7779$,
compared with $9.9784$ for DDIM, $5.1364$ for flow matching, and $8.7263$ for
WGAN-GP.  Its measured time for the evaluation batch is $0.25$ seconds, versus
$82.11$ seconds for DDIM and $30.18$ seconds for flow matching.  On CIFAR-10,
flow matching has the lower FID ($11.10$ versus $11.66$), while
\method{}--MAP has the higher Inception Score ($8.26$ versus $7.60$).
Sampling $10{,}000$ CIFAR-10 images takes $9.04$ seconds, compared
with $237.44$ seconds for DDIM and $584.96$ seconds for flow matching.
Training costs for MNIST and CIFAR-10 are reported in NVIDIA A100 GPU-hours.
On both image datasets, \method{}--MAP and flow matching have comparable
training costs, both higher than those of WGAN-GP, while \method{}--MAP
has substantially lower inference cost for sample generation than flow matching.
WGAN-GP samples faster, whereas \method{}--MAP achieves lower FID on both datasets.}

{On the tabular benchmarks, fixed-base \method{}--MC reaches
$W_2=0.0998$ on Superconduct, a reduction of $20.98\%$
from flow matching and $74.94\%$ from DDIM.  On Genomes it reaches
$W_2=0.1688$, reducing the DDIM value by $43.60\%$ and the WGAN-GP value by
$65.63\%$.  On both datasets, \method{}--MC also samples faster than DDIM,
while WGAN-GP has the lowest sampling time.
These benchmarks evaluate a low-dimensional transport in ambient
dimensions 81 and $10{,}000$, respectively.}

\begin{table}[!t]
\centering
\captionsetup{font=small}
\caption{{Fidelity, sampling cost, and measured training cost on real-data benchmarks.
Boldface marks the best observed fidelity or discrepancy value in each row;
``--'' denotes an unavailable measurement.}}
\label{tab:mnist_cifar10}
\small{
\setlength{\tabcolsep}{4.5pt}
\renewcommand{\arraystretch}{0.85}
\begin{tabular}{l l c c c c}
\toprule
Dataset & Quantity & \method{} & DDIM & FM & {WGAN-GP} \\
\midrule
\multirow{3}{*}{MNIST} 
& FID $\downarrow$ 
& {\textbf{3.7779}} & {9.9784} & {5.1364} & {8.7263}\\
& {Training (GPU-h)}
& {0.92} & {1.54} & {0.99} & {0.19}\\
& Sampling (s) $\downarrow$ 
& 0.25 & 82.11 & 30.18 & 0.19\\
\midrule
\multirow{4}{*}{{CIFAR-10}}
& {FID $\downarrow$}
& {11.66} & {38.10} & {\textbf{11.10}} & {38.15}\\
& {IS $\uparrow$}
& {\textbf{8.26}} & {6.17} & {7.60} & {5.42}\\
& {Training (GPU-h)}
& {4.04} & {0.56} & {4.41} & {0.43}\\
& {Sampling (s) $\downarrow$}
& {9.04} & {237.44} & {584.96} & {1.68}\\
\midrule
\multirow{2}{*}{Superconduct} 
& $W_2$ $\downarrow$ 
& \textbf{0.0998} & 0.3982 & 0.1263 & 0.1443\\
& Time (s) $\downarrow$ 
& 0.0020 & 0.0974 & 0.0750 & 0.0013\\
\midrule
\multirow{2}{*}{Genomes} 
& $W_2$ $\downarrow$ 
& \textbf{0.1688} & 0.2993 & -- & 0.4911\\
& Time (s) $\downarrow$ 
& 1.63 & 980.34 & -- & 0.11\\
\bottomrule
\end{tabular}}
\end{table}

Additional representative samples generated by \method{} are provided in
Supplementary Appendix~\ref{sec:supp-generated-samples}.

\FloatBarrier
\section{Discussion}
\label{sec:discussion}

{\method{} uses a single-level denoising objective to fit a
generative transport while retaining one-step sampling.  Across the benchmarks,
its fidelity is competitive with the iterative baselines considered here.
Once $h$ is fitted, generation only draws $U\sim\pi$ and evaluates $h(U)$;
no reverse-time trajectory is integrated at sampling.}

{Gaussian smoothing gives the generated law a positive ambient
density, and the posterior identity expresses its score through $h$.  Score
matching can therefore train the transport while generation remains the direct
pushforward $P_h=h_\#\pi$.  Theorem~\ref{thm:local-erm-transfer} separates
population approximation, statistical estimation, and exceptional-event
errors.  For the localized spline-RePU class, bounding these terms yields
the minimax rate for the unsmoothed law.  The choice of smoothing level
also affects the accuracy of the finite-anchor approximation.
Larger $t$ spreads the latent posterior weights but removes fine
structure; smaller $t$ reduces smoothing bias but can concentrate the weights
and increase finite-anchor error.  Prior-sampled MC, QMC, and MAP/Laplace
proposals provide different accuracy--cost trade-offs across these regimes.
In practice, $t$ and the anchor budget $K$ can be selected jointly by held-out
performance, with effective sample size used to diagnose weight degeneracy.}

{The present theory treats regular embeddings of a compact source and a localized RePU class
whose smooth multiresolution modules can be analyzed explicitly.  The ReLU
networks used in the experiments are not covered by this exact theorem;
extending the sharp rate to unrestricted, fully trained ReLU transports remains
open.  Singular support geometry, noncompact bases, and likelihood evaluation
without latent inversion also require additional work.  Further work could
improve anchor efficiency through learned proposals and extend the framework
to conditional generation and higher-resolution images.  Hybrid samplers that
use \method{} for initialization followed by a short refinement chain could
also offer a flexible trade-off between sampling cost and generation quality.}

\FloatBarrier
\bigskip
\section*{Supplementary Materials}
\addcontentsline{toc}{section}{Supplementary Materials}

\allowdisplaybreaks[2]
\setlength{\parskip}{3pt plus 1pt minus 1pt}
\setlength{\abovedisplayskip}{5pt plus 1pt minus 2pt}
\setlength{\belowdisplayskip}{5pt plus 1pt minus 2pt}
\setlength{\abovedisplayshortskip}{2pt plus 1pt minus 1pt}
\setlength{\belowdisplayshortskip}{3pt plus 1pt minus 1pt}
\setlength{\jot}{1pt}
\setlength{\textfloatsep}{10pt plus 2pt minus 2pt}
\setlength{\floatsep}{8pt plus 2pt minus 2pt}
\setlength{\intextsep}{8pt plus 2pt minus 2pt}

\appendix
\beginsupplement
\renewcommand{\theHtheorem}{supp.\arabic{theorem}}
\renewcommand{\theHlemma}{supp.\arabic{lemma}}
\renewcommand{\theHassumption}{supp.\arabic{assumption}}
\renewcommand{\theHdefinition}{supp.\arabic{definition}}
\renewcommand{\theHproposition}{supp.\arabic{proposition}}
\renewcommand{\theHfigure}{supp.\arabic{figure}}
\renewcommand{\theHtable}{supp.\arabic{table}}
\renewcommand{\theHequation}{supp.\arabic{equation}}

Appendix~\ref{appendix_E} provides experimental details and additional
results. Appendix~\ref{sec:mcapprox} analyzes MC, QMC, and MAP approximations
of the posterior mean, score, and loss, and
Appendix~\ref{sec:implementation} presents the training algorithm.
Appendix~\ref{sec:proofs} contains the proofs of the main results, together
with the localized RePU construction, data-driven localization, minimax
lower bound, and a fixed-proposal finite-anchor extension.

\section{Experiment details and additional results}
\label{appendix_E}

\subsection{Toy Datasets}\label{sec:supp-toy-details}

For each dataset, observations are obtained by adding independent Gaussian
jitter $\varepsilon\sim\mathcal N(0,\sigma^2I_D)$ to the noiseless point
$\mathbf x\in\mathbb R^D$ defined below. The \texttt{jitter} parameter
controls $\sigma$.

\paragraph*{Rings2d.}
Draw $k\sim\Unif\{1,\dots,K\}$ and $\theta\sim\Unif[0,2\pi]$, and set
$\mathbf x=r_k(\cos\theta,\sin\theta)^\top$ for fixed radii
$r_1,\ldots,r_K$ (e.g., equally spaced in a fixed interval).
The default noise level is $\sigma=0.02$.

\paragraph*{Spiral2d.}
Draw $t\sim\Unif[t_{\min},t_{\max}]$ and set
$\mathbf x=r(t)(\cos t,\sin t)^\top$, with an increasing radius
$r(t)$, e.g., $r(t)=a+bt$ for $a,b>0$.

\paragraph*{Moons2d.}
Draw $c\sim\Unif\{0,1\}$ and $\theta\sim\Unif[0,\pi]$, and set
\[
\mathbf x=
\begin{cases}
(\cos\theta,\sin\theta)^\top, & c=0,\\
(1-\cos\theta,1-\sin\theta-\delta)^\top, & c=1,
\end{cases}
\]
where the vertical separation $\delta>0$ is fixed across experiments.

\paragraph*{Checker2d.}
Partition $[-1,1]^2$ into $m\times m$ cells of width $w=2/m$.
Draw $(i,j)$ uniformly from $\{0,\ldots,m-1\}^2$ with $i+j$ even, then
draw $\mathbf x$ uniformly from the cell
\[
[-1+iw,-1+(i+1)w]\times[-1+jw,-1+(j+1)w].
\]
The default noise level is $\sigma=0.02$.

\paragraph*{Helix3d.}
Draw $t\sim\mathcal N(0,\pi^2)$ and set
$\mathbf x=(\cos t,\sin t,t/2)^\top$, with noise level $\sigma=0.02$.

\paragraph*{Torus3d.}
For major radius $R=2$, minor radius $r=1$, and independent
$u,v\sim\Unif[0,2\pi]$, set
\[
\mathbf x=\bigl((R+r\cos v)\cos u,\,
(R+r\cos v)\sin u,\,r\sin v\bigr)^\top.
\]

The toy-data results appear in Table~\ref{tab:wd_time} and
Figure~\ref{fig:toyex}.

\subsection{Real data}\label{sec:supp-real-data}

{Table~\ref{tab:real-config} lists the settings used for the four
real-data benchmarks.}

\begin{table}[H]
\centering
\caption{{\method{} settings for the real-data benchmarks.}}
\label{tab:real-config}
\scriptsize
\setlength{\tabcolsep}{4.5pt}
\renewcommand{\arraystretch}{0.95}
\begin{tabular}{l l c c c c c}
\toprule
{Dataset} & {Variant} & {Base $\pi$} &
{$K$} & {\makecell{$t$\\candidates}} &
{\makecell{$d$\\candidates}} & {Selection metric} \\
\midrule
{MNIST} 
& {MAP}
& {\makecell{$\cN(0,I_{64})\times$\\$\mathrm{Unif}\{e_1,\ldots,e_{16}\}$}}
& {4096}
& {\makecell{$\{0.003,0.008,0.020,0.050,$\\$0.100,0.200,0.400,0.700\}$}}
& {80} 
& {FID} \\
{CIFAR-10}
& {MAP}
& {$\cN(0,I_{128})$}
& {16384}
& {$\{0.01,0.1\}$}
& {128}
& {FID} \\
{Superconduct}
& {MC}
& {$\cN(0,I_d)$}
& {49152}
& {$\{0.1,\ldots,0.9\}$}
& {$\{8,16,32,64,128\}$}
& {$W_2$} \\
{Genomes}
& {MC}
& {$\cN(0,I_{128})$}
& {2048}
& {$\{0.1,0.2,\ldots,0.9\}$}
& {128}
& {$W_2$} \\
\bottomrule
\end{tabular}
\end{table}

\paragraph*{{MAP implementation.}}
{The MNIST and CIFAR-10 \method{}--MAP variants both use the
alternating procedure in Algorithm~\ref{alg:magt-map-bank}; their
dataset-specific settings are listed in Table~\ref{tab:real-config}.}

\begin{algorithm}[H]
\caption{{Importance-corrected alternating \method{}--MAP.}}
\label{alg:magt-map-bank}
{\scriptsize\setlength{\baselineskip}{0.88\baselineskip}
\begin{algorithmic}[1]
\Require Training observations \(\{y_i\}_{i=1}^{N_B}\), prior \(\pi\), stage lengths \(E_1,E_2\), candidate budgets \(K_1,K_2\), score intervals \([t_{r,\rm lo},t_{r,\rm hi}]\), MAP level \(t_{\rm MAP}\), and proposal parameters \(\rho_{\rm IS},\lambda_{\rm MAP},\gamma_{\rm cat}\).
\Ensure Decoder \(h_\theta\) and latent bank
\(\mathcal B=\{u_i\}_{i=1}^{N_B}\).
\State Randomly initialize \(h_\theta\) and independently draw \(u_i\sim\pi\) for \(i=1,\ldots,N_B\).
\For{\(e=1,\ldots,E_1\)}
    \For{each minibatch \(\{y_b\}_{b=1}^{B}\)}
        \State Draw \(t\sim\mathrm{Unif}[t_{1,\rm lo},t_{1,\rm hi}]\), form \(y_{b,t}=\alpha_ty_b+\sigma_t\epsilon_b\), and draw \(u_k\sim q=\pi\), \(k=1,\ldots,K_1\).
        \State Compute \(\widetilde w_{bk}\propto \pi(u_k)p_\theta(y_{b,t}\mid u_k)/q(u_k)\), \(\widetilde m_b=\sum_k\widetilde w_{bk}h_\theta(u_k)\), and \(\widehat s_b=(\alpha_t\widetilde m_b-y_{b,t})/\sigma_t^2\).
        \State Update \(\theta\) by the score loss \(\|\widehat s_b-(\alpha_ty_b-y_{b,t})/\sigma_t^2\|_2^2\).
    \EndFor
\EndFor
\For{\(e=1,\ldots,E_2\)}
    \For{each indexed minibatch \(\{(y_b,u_b)\}_{b=1}^{B}\)}
        \State Form \(y_{b,t_{\rm MAP}}\) as above and refine each \(u_b\) using the MAP objective \(\Phi\) in Appendix~\ref{sec:map-proposal} at \(t=t_{\rm MAP}\), including a categorical rescan on the first alternating epoch, and persist the accepted state.
        \State Set \(q_{\rm MAP}^{(e)}(u)=B^{-1}\sum_{b=1}^{B}q(u\mid y_{b,t_{\rm MAP}})\), with each proposal constructed using the current decoder as in Appendix~\ref{sec:map-proposal}, and form \(q^{(e)}=\rho_{\rm IS}\pi+(1-\rho_{\rm IS})q_{\rm MAP}^{(e)}\).
        \State Draw \(K_2\) candidates from \(q^{(e)}\), compute \(\widetilde w_{bk}\propto\pi(u_k)p_\theta(y_{b,t}\mid u_k)/q^{(e)}(u_k)\), and update \(\theta\) with the same posterior-mean score loss at \(t\sim\mathrm{Unif}[t_{2,\rm lo},t_{2,\rm hi}]\).
    \EndFor
\EndFor
\State Select \(\sigma_{\rm jit}\) on held-out training images; draw \(I\sim\mathrm{Unif}\{1,\ldots,N_B\}\) and \(Z\sim\mathcal N(0,I_d)\); return \(h_\theta(u_I+\sigma_{\rm jit}P_{\rm cont}Z)\).
\end{algorithmic}
}
\end{algorithm}

\paragraph*{{Real-data smoothing-level sensitivity.}}
{Eight independent MNIST runs use adjacent score intervals with
lower endpoints $t\in\{0.003,0.008,0.020,0.050,0.100,0.200,0.400,0.700\}$.
Validation selects the decoder epoch and jitter before one test evaluation.
In the left panel of Figure~\ref{fig:mnist-t-sensitivity}, MNIST reaches a
minimum FID of $3.8892$ at $t=0.020$, remains at most $4.32$ through $t=0.1$,
and rises to $19.8649$ at $t=0.7$.
This curve is a sensitivity analysis; the main-table MNIST model was fitted
independently.  The right panel reports Superconduct held-out $W_2$ over
$d\in\{8,16,32,64,128\}$ and
$t\in\{0.1,\ldots,0.9\}$; the minimum is $0.1029$ at $(16,0.9)$.}

\begin{figure}[H]
\centering
\begin{minipage}[c]{0.56\linewidth}\centering
{\setlength{\fboxsep}{1pt}\fcolorbox{black}{white}{\safeincludegraphics[width=0.98\linewidth]{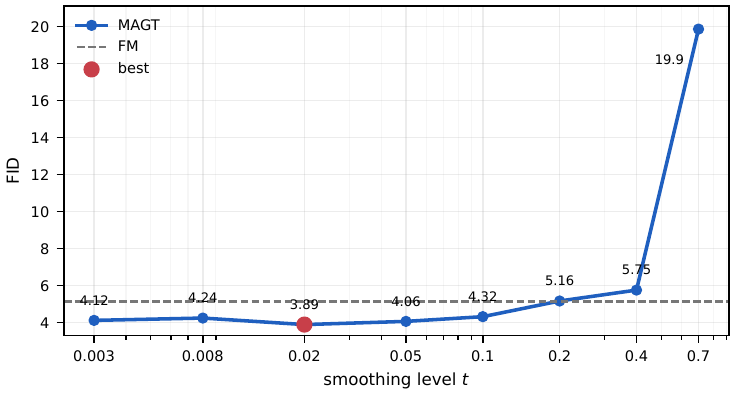}}}
\end{minipage}\hfill
\begin{minipage}[c]{0.40\linewidth}\centering
{\setlength{\fboxsep}{1pt}\fcolorbox{black}{white}{\safeincludegraphics[width=0.98\linewidth]{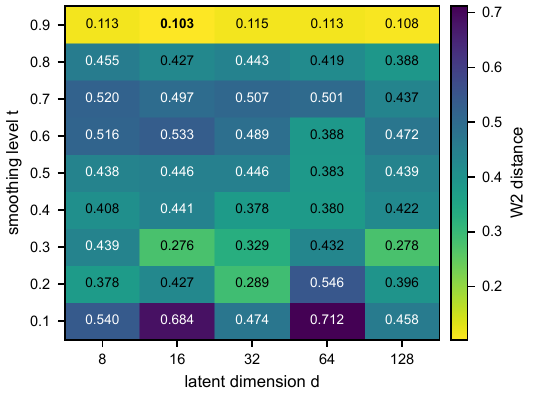}}}
\end{minipage}
\caption{{Smoothing sensitivity: MNIST FID (left) and
Superconduct held-out $W_2$ over $(d,t)$ (right).}}
\label{fig:mnist-t-sensitivity}\label{fig:superconduct-dt-sensitivity}
\end{figure}

\paragraph*{{MNIST anchor diagnostics.}}
{On 512 held-out images we compare three empirical candidate
proposals.  The prior-draw proposal uses latent anchors sampled independently
from $\pi$.  The initial-MAP proposal uses per-image MAP anchors obtained before
alternating refinement, whereas the refreshed-MAP proposal uses the persistent
MAP anchors after the decoder and latent bank have been updated alternately as
in Algorithm~1.  We evaluate each proposal over
$K\in\{2048,8192,16384,60000\}$ and $t\in\{0.003,0.005,0.008\}$.  The target
is the observable conditional score
$-\epsilon/\sigma_t$ from~\eqref{eq:ellK-def}.  The left panel of
Figure~\ref{fig:mnist-anchor-score-error} shows the full trajectories over
$K$, while Table~\ref{tab:mnist-anchor-proposal} gives the $K=60{,}000$
values.  Increasing $K$ lowers MSE, and at $K=60{,}000$ the refreshed bank
reduces MSE by $96.2\%$ relative to prior draws.}

\begin{table}[H]
\centering
\caption{{MNIST mean denoising-score MSE ($\times10^{-3}$) for
three empirical candidate banks.}}
\label{tab:mnist-anchor-proposal}
\scriptsize
\setlength{\tabcolsep}{6pt}
\renewcommand{\arraystretch}{0.95}
\begin{tabular}{c c c c}
\toprule
{\(t\)} & {Prior-draw bank} & {Initial MAP bank} & {Refreshed MAP bank} \\
\midrule
{0.003} & {182.812} & {13.526} & {\(\mathbf{6.979}\)} \\
{0.005} & {49.469} & {3.661} & {\(\mathbf{1.889}\)} \\
{0.008} & {13.436} & {0.994} & {\(\mathbf{0.513}\)} \\
\bottomrule
\end{tabular}
\end{table}

\paragraph*{{Fixed-target proposal comparison.}}
{The right panel of
Figure~\ref{fig:fixed-target-proposal-sensitivity} compares prior,
MAP-Laplace, and defensive proposals for a fixed two-dimensional target,
transport, and observation set over
$t\in\{0.05,0.10,0.25\}$ and $K\in\{64,256,1024,4096\}$, with exact $\pi/q$
correction.  Nested grids agree to $1.78\times10^{-13}$ and have boundary mass
at most $5.30\times10^{-12}$.  At $K=1024$, prior versus MAP-Laplace
(MSE, normalized ESS) equals $(37.5,.005)$ versus $(.391,.963)$ at $t=.05$,
$(2.11,.016)$ versus $(.122,.904)$ at $t=.10$, and $(.0606,.091)$ versus
$(.0298,.719)$ at $t=.25$.}

\begin{figure}[H]
\centering
\begin{minipage}[c]{0.48\linewidth}\centering
{\setlength{\fboxsep}{1pt}\fcolorbox{black}{white}{\safeincludegraphics[width=0.98\linewidth]{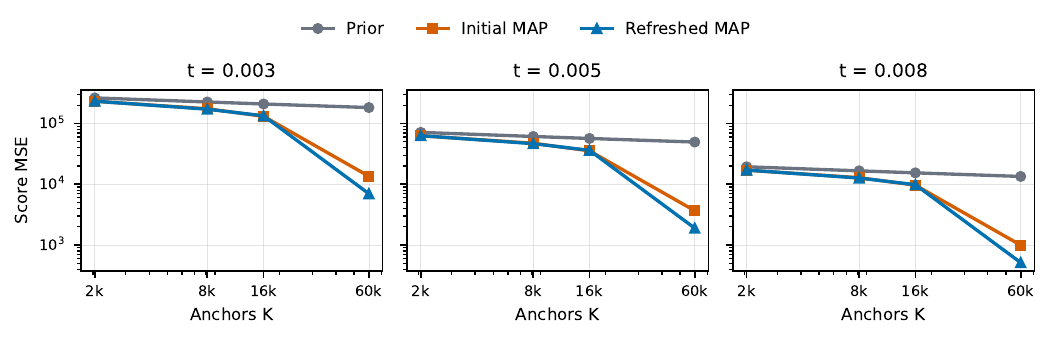}}}
\end{minipage}\hfill
\begin{minipage}[c]{0.48\linewidth}\centering
{\setlength{\fboxsep}{1pt}\fcolorbox{black}{white}{\safeincludegraphics[width=0.98\linewidth]{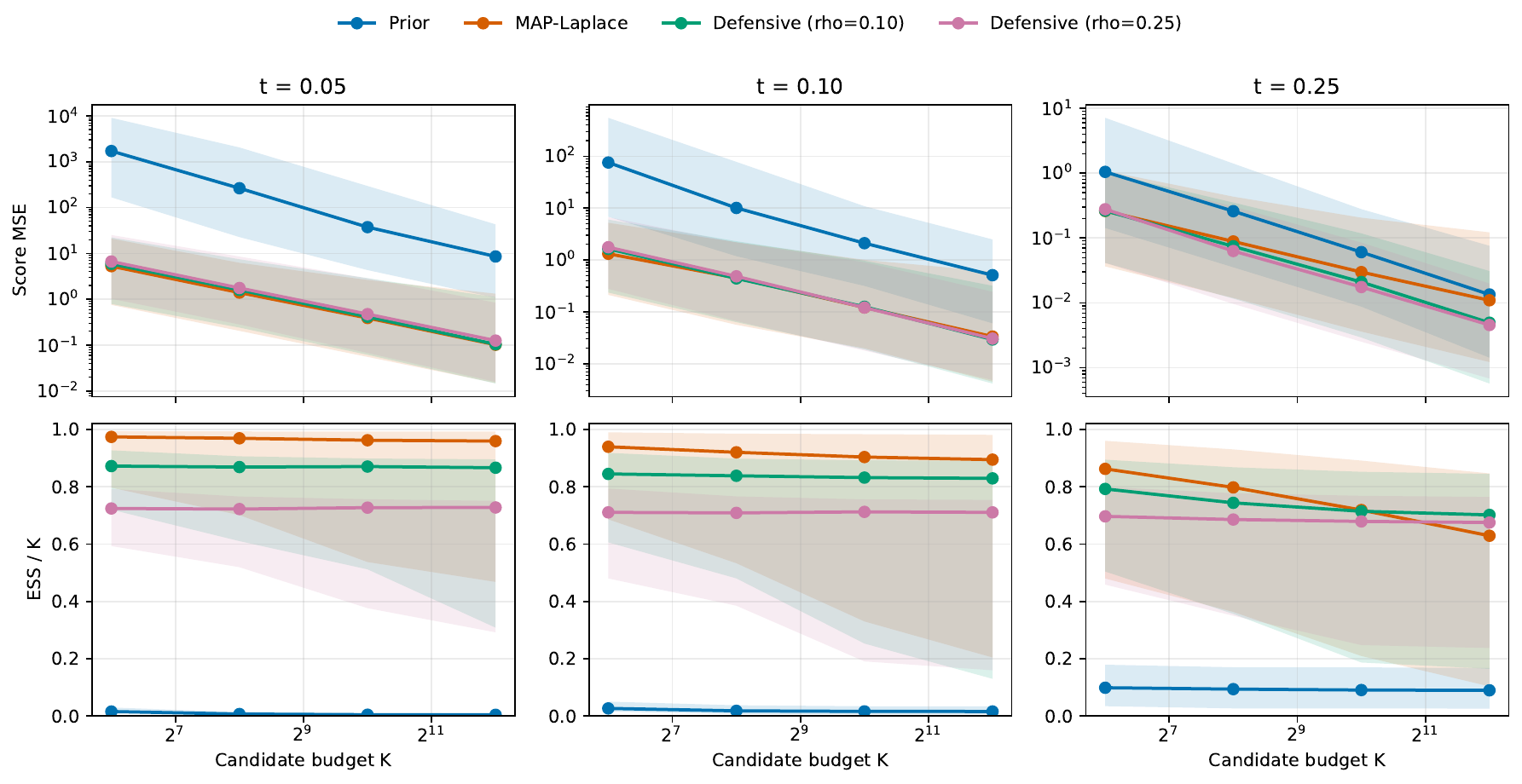}}}
\end{minipage}
\caption{{Anchor diagnostics: MNIST bank/budget MSE (left) and
fixed-target proposal MSE and normalized ESS (right).}}
\label{fig:mnist-anchor-score-error}\label{fig:fixed-target-proposal-sensitivity}
\end{figure}

\paragraph*{{Full CIFAR-10 training and evaluation.}}
{Ten class-specific five-layer decoders (latent and base width 128)
use two 200-epoch alternating stages at $t=0.01$ with anchor budgets 16384 and
8192.  Final banks use 100 jitters per center and uniform class sampling.
FID and IS use 10000 generated and 10000 test images under one Inception-v3
implementation.  Table~\ref{tab:mnist_cifar10} reports the quantitative
comparison, and Figure~\ref{fig:cifar_comparison} shows generated samples.}

\subsection{Examples of generated samples}\label{sec:supp-generated-samples}

{Figures~\ref{fig:mnist} and~\ref{fig:cifar_comparison} give the
qualitative comparisons for MNIST and CIFAR-10, respectively.}

\begin{figure}[H]
 \centering
 {\setlength{\fboxsep}{1pt}\fcolorbox{black}{white}{\safeincludegraphics[width=0.98\linewidth]{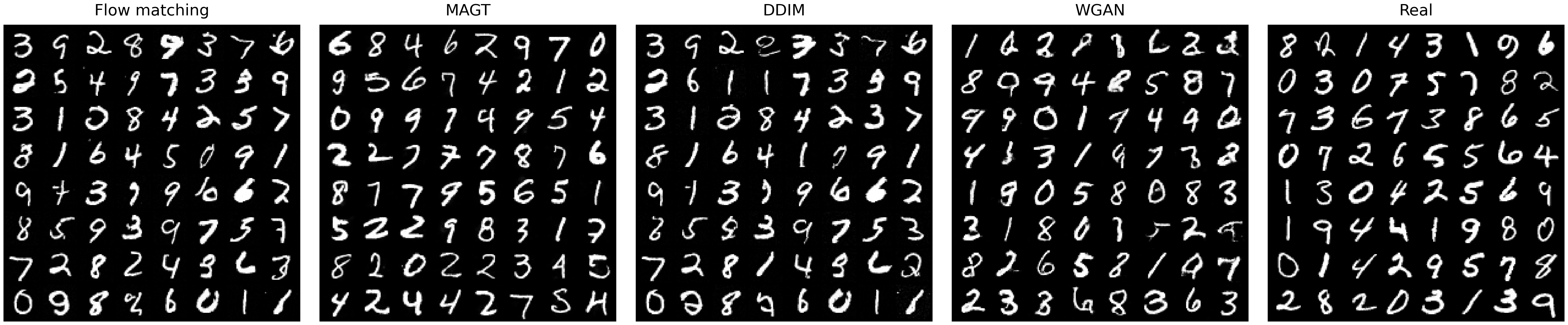}}}
 \caption{{Generated samples from \method{}, FM, DDIM, and
 WGAN-GP on MNIST, together with held-out data.}}
 \label{fig:mnist}
\end{figure}

\begin{figure}[H]
 \centering
 {\setlength{\fboxsep}{1pt}\fcolorbox{black}{white}{\safeincludegraphics[width=0.58\linewidth]{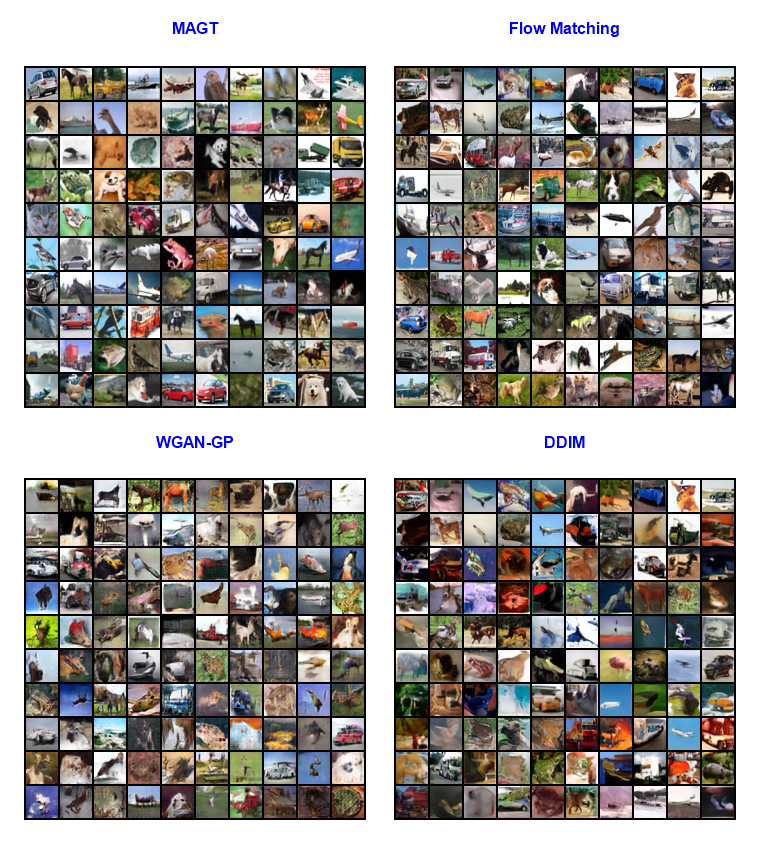}}}
 \caption{{Generated samples from \method{}, FM, DDIM, and
 WGAN-GP on CIFAR-10.}}
 \label{fig:cifar_comparison}
\end{figure}

{Figure~\ref{fig:geno} compares real and generated genomic samples
after PCA is fit separately within each class and applied to both sets.}

\begin{figure}[H]
 \centering
 \captionsetup{font=small,skip=5pt}
 \safeincludegraphics[width=0.88\linewidth]{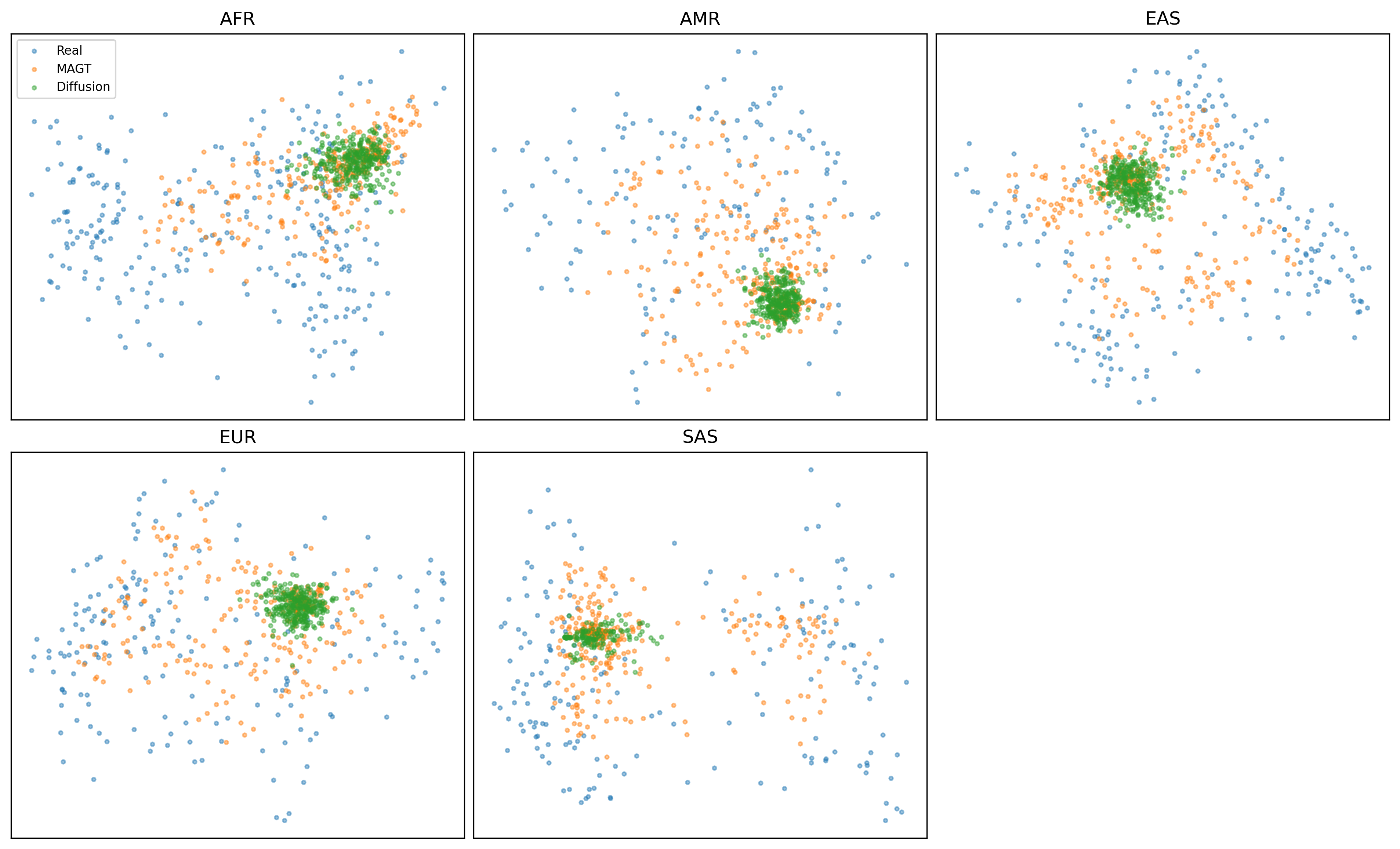}
 \caption{{Class-wise PCA projections of real and generated
 genomic samples for five classes.}}
 \label{fig:geno}
\end{figure}

\section{Practical choices for Monte Carlo approximation}
\label{sec:mcapprox}
{This appendix controls finite-anchor approximation of the
posterior mean, score, and loss for a fixed transport.  MC and MAP share a
self-normalized importance-sampling bound, whereas QMC uses Koksma--Hlawka;
the results retain smoothing, dimension, and proposal
dependence~\citep{owen2013mc,niederreiter1992random,dick2010digital}.}

{Let $m_t(y_t;h)=\mathbb E_h\{h(U)\mid Y_t^h=y_t\}$ and let
$\widetilde m_{t,K}$ be its approximation from~\eqref{eq:posterior-mean-estimator}.  By
\eqref{eq:mixture-score}--\eqref{eq:score-estimator},
\begin{equation}\label{eq:score-mean-error}
\widetilde s_{t,K}(y_t;h,\pi,\tilde\pi)-s_t(y_t;h)
=
\frac{\alpha_t}{\sigma_t^2}\Bigl(\widetilde m_{t,K}(y_t)-m_t(y_t;h)\Bigr),
\end{equation}
so posterior-mean error directly controls score error.}

{Put $r_t(y,y_0)=(\alpha_ty_0-y)/\sigma_t^2$ and
$d_{t,K}(y;h)=\widetilde s_{t,K}(y;h,\pi,\tilde\pi)-s_t(y;h)$.
If \(\|h\|_\infty\vee\|y_0\|_2\le B\), the posterior-mean identity gives
\(\bigl|\ell_K(y,y_0;h)-\|s_t(y;h)-r_t(y,y_0)\|_2^2\bigr|
\le4\alpha_tB\|d_{t,K}(y;h)\|_2/\sigma_t^2
+\|d_{t,K}(y;h)\|_2^2\).
Thus the bounds below control both integrated error and fresh-anchor training.

\paragraph*{Common SNIS bound.}
MC sets the proposal to $\pi$, MAP uses $q(\cdot\mid y)$, and the bias bound
also controls the fresh-anchor criterion.

}

{
\begin{lemma}[Common SNIS mean and loss bounds]
\label{lem:anchor-averaged-loss}
Fix $h,y$ with the candidate posterior
$p_h(\cdot\mid y)\ll\tilde\pi(\cdot\mid y)$, and set
$\omega_h=dp_h(\cdot\mid y)/d\tilde\pi(\cdot\mid y)$ and
$D_{2,h}(y)=\mathbb E_{\tilde\pi}\omega_h(U;y)^2$.  Draw $K$ independent
proposal anchors and form the self-normalized mean $\widetilde m_{t,K}(y)$
(use any vector of norm at most $B$ if its denominator vanishes).
If \(\|h\|_\infty\le B\), then
\begin{align}
\mathbb E_{\mathrm{anc}}
\|\widetilde m_{t,K}(y)-m_t(y;h)\|_2^2
&\le \frac{32B^2}{K}D_{2,h}(y), \label{eq:snis-mse-D2}\\
\left\|
\mathbb E_{\mathrm{anc}}
\{\widetilde m_{t,K}(y)-m_t(y;h)\}
\right\|_2
&\le \frac{16B}{K}D_{2,h}(y). \label{eq:snis-bias-D2}
\end{align}
Let $\ell_{K,t}^{\mathrm{av}}=\mathbb E_{\mathrm{anc}}\ell_K$.
If also \(\|y_0\|_2\le B\), then
\begin{equation}\label{eq:anchor-averaged-loss-pointwise}
\left|
\ell_{K,t}^{\mathrm{av}}(y,y_0;h)-\|s_t(y;h)-r_t(y,y_0)\|_2^2
\right|
\le
\frac{96\alpha_t^2B^2}{K\sigma_t^4}D_{2,h}(y).
\end{equation}
For any joint law $P$ of $(Y_t,Y_0)$ with $\|Y_0\|_2\le B$, if
\(\mathfrak D_t(P,\mathcal H):=
\sup_{h\in\mathcal H}\mathbb E_P D_{2,h}(Y_t)<\infty\),
then
\begin{equation}\label{eq:anchor-averaged-risk}
\sup_{h\in\mathcal H}
\left|\mathbb E_P\ell_{K,t}^{\mathrm{av}}(Y_t,Y_0;h)
-\mathbb E_P\|s_t(Y_t;h)-r_t(Y_t,Y_0)\|_2^2\right|
\le
\frac{96\alpha_t^2B^2}{K\sigma_t^4}
\mathfrak D_t(P,\mathcal H).
\end{equation}
\end{lemma}

\begin{proof}
Suppress \(h,y,t\) and put \(m=m_t(y;h)\).  For proposal draws
\(U_1,\ldots,U_K\), define
\(\widehat b_K=K^{-1}\sum_{k=1}^K\omega_h(U_k;y)\) and
\(\widehat a_K=K^{-1}\sum_{k=1}^K
\omega_h(U_k;y)\{h(U_k)-m\}\).
Then \(\mathbb E\widehat a_K=0\),
\(\mathbb E\|\widehat a_K\|_2^2\le4B^2D_{2,h}(y)/K\), and
\(\mathbb E|\widehat b_K-1|^2\le D_{2,h}(y)/K\).
Let \(E=\{\widehat b_K\ge1/2\}\).  On \(E\),
\(\widetilde m_{t,K}(y)-m=\widehat a_K/\widehat b_K\), whereas on
\(E^c\) the error has norm at most \(2B\).  Moreover,
\(\mathbb P(E^c)\le4D_{2,h}(y)/K\) by Chebyshev's inequality.  Consequently,
\(\mathbb E\|\widetilde m_{t,K}(y)-m\|_2^2
\le4\mathbb E\|\widehat a_K\|_2^2+4B^2\mathbb P(E^c)
\le32B^2D_{2,h}(y)/K\), proving~\eqref{eq:snis-mse-D2}.

For the bias, expand $\widehat a_K/\widehat b_K$ on $E$, use
$\mathbb E\widehat a_K=0$, and split off $E^c$.  Cauchy--Schwarz bounds the two
$\widehat a_K$ terms by $4BD_{2,h}(y)/K$ each and the remaining term by
$2B\mathbb P(E^c)\le8BD_{2,h}(y)/K$, proving~\eqref{eq:snis-bias-D2}.

The posterior identity gives
\(\|s_t(y;h)-r_t(y,y_0)\|_2\le2\alpha_tB/\sigma_t^2\).
Together with~\eqref{eq:score-mean-error} and
\eqref{eq:snis-mse-D2}--\eqref{eq:snis-bias-D2},
$|\mathbb E\ell_K-\|s_t-r_t\|_2^2|
\le2\|s_t-r_t\|\|\mathbb Ed\|+
\mathbb E\|d\|^2$ gives~\eqref{eq:anchor-averaged-loss-pointwise}.
Averaging over $P$ proves~\eqref{eq:anchor-averaged-risk}.
\end{proof}

The absolute value in~\eqref{eq:anchor-averaged-loss-pointwise} follows the
fresh-anchor average, so its cross term uses the $K^{-1}$ bias bound.}

\subsection{\method{}-MC}
\method{}-MC sets $\tilde\pi=\pi$ in~\eqref{eq:is-weights}; the importance
ratio cancels and weights are proportional to the Gaussian likelihood.

{Prior anchors are parallel but can collapse at small $\sigma_t$;
QMC spreads anchors deterministically and MAP concentrates them near modes.}

{
For a fixed transport \(h\), write \(\mu_h=h_\#\pi\).  The only lower-mass
condition needed for the Monte Carlo bound is
\begin{equation}\label{eq:anchor-small-ball}
\mu_h\{x':\|x'-x\|\le r\}\ge c_\mu r^d,
\qquad x\in\operatorname{supp}(\mu_h),\quad 0<r\le r_\mu,
\end{equation}
for fixed \(c_\mu,r_\mu>0\).

\begin{lemma}[$K$-approximation error]\label{thm:gen}
Consider the estimator with $\tilde\pi=\pi$.  Let
\(\Y_t\sim p_t^h\), draw i.i.d.\ anchors
$\bm U^{(1)},\dots,\bm U^{(K)}\sim\pi$ independently of $\Y_t$, and form the transport-based score
estimate $\widetilde s_{t,K}(\Y_t;h,\pi,\pi)$ via~\eqref{eq:score-estimator}.
Suppose $\|h\|_\infty\le B$ and~\eqref{eq:anchor-small-ball} holds.  Then, for
every $K\ge1$ and every VP level $t\in(0,1)$, there is a constant $C>0$
depending only on $(B,d,D,c_\mu,r_\mu)$ such that
\begin{equation}\label{eq:K-general}
\E\,
\big\|\widetilde s_{t,K}(\Y_t;h,\pi,\pi)-s_t(\Y_t;h)\big\|_2^2
\ \le\ \frac{C\,\alpha_t^2}{K\,\sigma_t^{d+4}}\,,
\end{equation}
where the expectation is over both $\Y_t$ and the anchors, and
$s_t(\cdot;h):=\nabla_{\y}\log p_t^h(\y)$ is the ideal (infinite-anchor)
mixture score at level $t$.
\end{lemma}

\begin{proof}
Put \(X=h(U)\), let \(\mu=\mu_h\), and write
\(q_x(y)=\phi_{\sigma_t}(y-\alpha_tx)\),
\(p_t^h(y)=\int q_x(y)\,d\mu(x)\),
\(\overline w_y(x)=q_x(y)/p_t^h(y)\), and
\(D_2(y)=\int\overline w_y(x)^2\,d\mu(x)\).
Lemma~\ref{lem:anchor-averaged-loss}, applied to the identity map on
\(\operatorname{supp}(\mu)\), bounds the conditional mean-squared error by
\(32B^2D_2(y)/K\).  For fixed $x,y$, put $R=\|y-\alpha_tx\|$ and
\(r(x,y)=\min\{r_\mu,\sigma_t^2/[4\alpha_t(R+\sigma_t)]\}\).
Then $\|y-\alpha_tx'\|^2-R^2\le\sigma_t^2$ for
$\|x'-x\|\le r(x,y)$, so~\eqref{eq:anchor-small-ball} implies
$p_t^h(y)\ge c\,r(x,y)^d q_x(y)$.
Hence $q_x^2/p_t^h\le Cr(x,y)^{-d}q_x$, whose integral is bounded by
$C[1+(\alpha_t/\sigma_t)^d\mathbb E(1+\|Z\|)^d]\le C\sigma_t^{-d}$.
Tonelli therefore gives $\mathbb E_{Y_t}D_2(Y_t)\le C\sigma_t^{-d}$.
Finally,
\(\widetilde s_{t,K}(y;h,\pi,\pi)-s_t(y;h)
=\alpha_t\{\widetilde m_{t,K}(y)-m_t(y;h)\}/\sigma_t^2\).
Combining the last three bounds proves~\eqref{eq:K-general}.
\end{proof}
}

{If the smoothed data density $p_t^*:=p_t^{h^*}$ satisfies
$p_t^*\le C_{\rm dom}p_t^h$, the same argument gains factor
$C_{\rm dom}$.  These are fixed-transport approximation bounds;
Theorem~\ref{thm:acc} analyzes the exact loss.}

\subsection{\method{}-QMC}
{
\method{}-QMC replaces the i.i.d.\ anchors by a low-discrepancy point set
\(P_K=\{z_1,\ldots,z_K\}\subset[0,1]^d\), mapped to the base law by
\(T_\#\operatorname{Unif}([0,1]^d)=\pi\).  Write \(D^*(P_K)\) for its star
discrepancy and \(V_{\rm HK}\) for Hardy--Krause variation.

\begin{lemma}[\method{}-QMC score discrepancy bound]\label{lem:qmc-discrepancy}
Fix \(t\in(0,1)\), \(y\in\mathbb R^D\), and \(h\), and use anchors
\(U^{(j)}=T(z_j)\) in~\eqref{eq:score-estimator}.  Define
\(f_0(z)=\phi(y;\alpha_th(T(z)),\sigma_t^2I_D)\),
\(f_1(z)=h(T(z))f_0(z)\),
\(I_0=\int_{[0,1]^d}f_0(z)\,dz\),
\(I_1=\int_{[0,1]^d}f_1(z)\,dz\), and
\(V_{\rm HK}(f_1)=\sum_{r=1}^D V_{\rm HK}(f_{1,r})\).
If these variations are finite and
\(V_{\rm HK}(f_0)D^*(P_K)\le I_0/2\), then
\begin{equation}\label{eq:qmc-score-bound}
\|\widetilde s^{\rm QMC}_{t,K}(y)-\nabla_y\log p_t^h(y)\|_2
\le
\frac{2\alpha_t}{\sigma_t^2I_0}
\{V_{\rm HK}(f_1)+\|m_t(y;h)\|_2V_{\rm HK}(f_0)\}D^*(P_K).
\end{equation}
Hence classical low-discrepancy constructions give the deterministic rate
\(O\{K^{-1}(\log K)^d\}\) for fixed \(h,y\), provided the displayed
variations are finite \citep{niederreiter1992random,dick2010digital}.
\end{lemma}

\begin{proof}
Let \(I_{r,K}=K^{-1}\sum_{j=1}^Kf_r(z_j)\), for \(r=0,1\).
Koksma--Hlawka gives
\(|I_{0,K}-I_0|\le V_{\rm HK}(f_0)D^*(P_K)\) and
\(\|I_{1,K}-I_1\|_2\le V_{\rm HK}(f_1)D^*(P_K)\).
The denominator condition implies \(I_{0,K}\ge I_0/2\).  Since
\(m_t(y;h)=I_1/I_0\), the ratio identity yields
\(\|I_{1,K}/I_{0,K}-I_1/I_0\|_2
\le2\{V_{\rm HK}(f_1)+
\|m_t(y;h)\|_2V_{\rm HK}(f_0)\}D^*(P_K)/I_0\).
Multiplying by \(\alpha_t/\sigma_t^2\) proves the claim.
\end{proof}
}

\subsection{\method{}-MAP}
\label{sec:map-proposal}
{\method{}-MAP uses a data-dependent Gaussian proposal constructed
from a MAP--Laplace--Gauss--Newton approximation.}

For a fixed noise level $t$ and observation $\y_t\in\R^D$, the latent posterior induced by the base density $\pi$ and the map $h$ is
\begin{equation}
\label{eq:latent-posterior}
p_h(\bm u\mid \y_t)\ \propto\ \pi(\bm u)\,\phi\!\left(\y_t;\alpha_t h(\bm u),\sigma_t^2 \bm I_D\right).
\end{equation}
We compute a MAP estimate $\hat{\bm u}\in\arg\max_{\bm u}p_h(\bm u\mid \y_t)$, equivalently
$\hat{\bm u}\in\arg\min_{\bm u}\Phi(\bm u)$ for the negative log-posterior
$
\Phi(\bm u)
\ :=\ 
\frac{1}{2\sigma_t^2}\big\|\y_t-\alpha_t h(\bm u)\big\|_2^2\;-\;\log \pi(\bm u)$.

A Laplace approximation of~\eqref{eq:latent-posterior} yields a Gaussian proposal of the form
$q(\bm u\mid \y_t)=\mathcal{N}(\hat{\bm u},\tilde{\bm\Sigma})$.
{In \method{}-MAP this Gaussian is the proposal
$\tilde\pi(\cdot\mid \y_t)=q(\cdot\mid \y_t)$ in~\eqref{eq:is-weights}; the
importance ratio corrects it to the posterior~\eqref{eq:latent-posterior}.}

\paragraph*{Gauss--Newton Laplace proposal.}
{For the standard Gaussian base used in this implementation, the
Gauss--Newton approximation uses the Jacobian matrix $Dh(\hat{\bm u})$ and
requires only first derivatives of $h$.}
Concretely, we take
\(\hat{\bm\Lambda}:=\bm I_d+
\alpha_t^2Dh(\hat{\bm u})^{\top}Dh(\hat{\bm u})/\sigma_t^2\) and
\(\tilde{\bm\Sigma}:=(\kappa_q\hat{\bm\Lambda}
+\delta_q\bm I_d)^{-1}\),
where $\bm I_d$ is the $d\times d$ identity matrix,
$\kappa_q>0$ scales the proposal precision, and $\delta_q\ge0$ provides damping.

{
\begin{lemma}[\method{}-MAP self-normalized IS error]\label{lem:snis-map}
Fix \(t\in(0,1)\) and \(y\in\mathbb R^D\).  Suppose
\(p_h(\cdot\mid y)\ll q(\cdot\mid y)\), \(\|h\|_\infty\le B\), and the
second weight moment
\(D_2\{p_h(\cdot\mid y)\|q(\cdot\mid y)\}<\infty\).
For \(K\) i.i.d.\ anchors from \(q(\cdot\mid y)\), the self-normalized MAP
score estimator satisfies
\[
\mathbb E\!\left[
\|\widetilde s_{t,K}(y;h,\pi,q)-\nabla_y\log p_t^h(y)\|_2^2
\mid y\right]
\le
\frac{32\alpha_t^2B^2}{K\sigma_t^4}
D_2\{p_h(\cdot\mid y)\|q(\cdot\mid y)\}.
\]
\end{lemma}

\begin{proof}
Apply Lemma~\ref{lem:anchor-averaged-loss} with proposal \(q(\cdot\mid y)\)
and use the score-to-posterior-mean identity~\eqref{eq:score-mean-error}.
\end{proof}
}

\section{Implementation}
\label{sec:implementation}

{The anchor budget $K$ enters $L_{n,K}$ through the self-normalized score
estimator~\eqref{eq:is-weights}.  Algorithm~\ref{algorithm_1} implements the
resulting SGD update with a stabilized softmax and chunked backpropagation.}

\begin{algorithm}[H]
\caption{{Single-level transport selection for fixed $(K,d)$} \label{algorithm_1} }
{\footnotesize \setlength{\baselineskip}{0.85\baselineskip}
\begin{algorithmic}[1]
\Require Training data $\mathcal{D}_{\text{train}}$, validation data $\mathcal{D}_{\text{val}}$, latent dimension $d$, anchor count $K$, VJP chunk size $K_c$, evaluation metric $\mathcal{E}(\cdot,\cdot)$, candidate times $\mathcal{T}=\{t_1,\dots,t_L\}$, learning rate $r$
\Ensure Trained transport \(\hat h_{\hat\lambda}\)
\For{$t \in \mathcal{T}$}
\For{$e = 1,\dots,\text{max epoch}$}
\For{each minibatch $\{\Y^{(i)}\}_{i=1}^{B} \subset \mathcal{D}_{\text{train}}$}
\State Sample noises $\{ \Z^{(i)}\}_{i=1}^{B}\!\sim\!\mathcal N(\mathbf 0,\mathbf I)$; set $t_b\!\gets\!t$
\State $\alpha_b\!=\!\alpha(t_b)$, $\sigma_b\!=\!\sigma(t_b)$; $\ \Y^{(b)}_t \gets \alpha_b \Y^{(b)} + \sigma_b \Z^{(b)}$
\State \textbf{Parameter update:}\quad $\theta \gets \textsc{Update}(\theta, r;\ \{(\Y^{(b)},\Y^{(b)}_t,t_b)\}_{b=1}^B,\ K,\ K_c)$
\EndFor
\EndFor
\State Set $\hat h_{(t,d,K)} \gets h_{\theta}$
\State Sample anchors $\{\U^j\}_{j=1}^{|\mathcal{D}_{\text{val}}|} \sim \pi$;
Generate $\widehat{\mathcal{D}}^{(t)} \gets \{\hat h_{(t,d,K)}(\U^j)\}_{j=1}^{|\mathcal{D}_{\text{val}}|}$
\State Compute $e_{\rm val}(t) \gets \mathcal E\!\big(\widehat{\mathcal{D}}^{(t)}, \mathcal{D}_{\text{val}}\big)$
\EndFor
\State $\hat t \gets \arg\min_{t \in \mathcal{T}} e_{\rm val}(t)$, $\hat\lambda \gets (\hat t,d,K)$, $\hat h_{\hat\lambda} \gets \hat h_{(\hat t,d,K)}$
\State \Return $\hat h_{\hat\lambda}$

\Function{Update}{$\theta, r;\ \{(\Y^{(b)},\Y^{(b)}_t,t_b)\}_{b=1}^B,\ K,\ K_c$}
\State Sample latents $\bm u_k \sim \pi$ for $k=1,\dots,K$
\State Compute centers \textbf{without grad}: $\tilde{\Y}^{(k)} = h_{\theta}(u_k)$
\State \textbf{{Phase 1 (no-grad forward \& exact center-gradient direction).}}
\For{$b=1$ \textbf{to} $B$}
\State $z_{b,k} \gets -\|\Y^{(b)}_t-\alpha_b \tilde{\Y}^{(k)}\|^2/(2\sigma_b^2)$,\quad
$w_{b,k} \gets \mathrm{softmax}_k(z_{b,\cdot})$
\State $m_b \gets \sum_{k=1}^K w_{b,k} \tilde{\Y}^{(k)}$,\quad
$sP_b \gets (\alpha_b m_b - \Y^{(b)}_t)/\sigma_b^2$
\EndFor
\State Set $T_b\gets(\alpha_b\Y^{(b)}-\Y_t^{(b)})/\sigma_b^2$, $g_b \gets sP_b - T_b$, and $c_b \gets (\alpha_b/\sigma_b^2)\,g_b$,
$\Delta_{b,k} \gets \langle \tilde{\Y}^{(k)} - m_b,\, c_b\rangle$
\State For each $k$, set
$g_k \;\gets\; \sum_{b=1}^B \Big[\, w_{b,k}\, c_b \;+\; w_{b,k}\,\Delta_{b,k}\big(\tfrac{\alpha_b}{\sigma_b^2}\Y^{(b)}_t - \tfrac{\alpha_b^2}{\sigma_b^2}\tilde{\Y}^{(k)}\big)\,\Big]
$
\State \textbf{Phase 2 (chunked VJP; $g$ frozen).}
\For{chunks $\mathcal{K}\subset\{1,\dots,K\}$ of size $\le K_c$}
\State $S_{\mathcal{K}}(\theta) \gets \sum_{k\in\mathcal{K}} \langle h_{\theta}(u_k),\, g_k\rangle$ \ \textit{(treat $g$ as constant)}
\State Backprop $\nabla_\theta S_{\mathcal{K}}(\theta)$ and accumulate
\EndFor
\State Gradient step: $\theta \leftarrow \theta - r \sum_{\mathcal{K}} \nabla_\theta S_{\mathcal{K}}(\theta)$
\State \Return $\theta$
\EndFunction
\end{algorithmic}}
\end{algorithm}

Algorithm~\ref{algorithm_1} evaluates the stabilized-softmax gradient
up to a constant absorbed into the learning rate. Let $L$ denote the
summed minibatch loss (the mean loss adds the factor $B^{-1}$).
After computing {$g_k=\tfrac12\,\partial L/\partial h_\theta(u_k)$} without an anchor
activation graph, it accumulates
{$\nabla_\theta L=2\sum_k\nabla_\theta h_\theta(u_k)^\top g_k$} in chunks of size
$K_c$.  Computing the weights for all $B\times K$ observation--anchor pairs
costs $O(BKD)$.  Chunked backpropagation does not reduce this computation, but
stores the activation graph for only $K_c$ anchors at a time.

\clearpage
\section{Proofs for Section~\ref{sec:risk}}
\newcommand{\thetah}{\theta_{\rm h}}
\newcommand{\thetao}{\theta_{\rm o}}
\label{sec:proofs}

\begingroup

\setlength{\parskip}{1pt plus 1pt minus 1pt}

Throughout, set $d=d^*$, $s_\star=\eta+1$, $\beta=(\eta+2)/2$, and
$P^*=P_{h^*}$.

\paragraph*{Probability and loss notation.}\enspace
For the $\mathcal N(0,\tau I_D)$ density $\phi_\tau$, write
$K_\tau\mu:=\phi_\tau*\mu$, $p_\tau^P:=K_\tau P$,
$s_{\tau,P}:=\nabla\log p_\tau^P$, and $p_\tau^*:=p_\tau^{P^*}$.
With $t=\tau/(1+\tau)$, put $\ell_\tau(x;h):=(1-t)\ell_t(x;h)$
and $R_\tau(P_h):=R_\tau(h)$.
For a function $f$ of one clean observation, write
$Pf:=\mathbb E_{X\sim P^*}f(X)$ and
$P_nf:=n^{-1}\sum_{i=1}^nf(X_i)$.
For positive densities $q,p$ on $\mathbb R^D$, let
$\mathcal J(q\Vert p):=\int\|\nabla\log q-\nabla\log p\|_2^2q$
denote relative Fisher divergence.

\paragraph*{Function-space conventions.}\enspace
For $\mu>0$, write $C^\mu=B_{\infty,\infty}^\mu$ for the
H\"older--Zygmund scale; $C^0$ denotes the sup-norm space.
We use $W^{k,p}$ for Sobolev spaces and $H^\sigma$ for the
$L^2$ Sobolev scale, with negative orders defined by duality.
Manifold and bundle norms use the fixed atlas and local frames.

\paragraph*{Geometric notation.}\enspace
For a closed set $\mathcal M\subset\mathbb R^D$, write
$\operatorname{dist}(y,\mathcal M):=\inf_{x\in\mathcal M}\|y-x\|$ and
\[
\operatorname{reach}(\mathcal M)
:=\sup\bigl\{r>0:\text{ every $y$ with $\operatorname{dist}(y,\mathcal M)<r$
has a unique nearest point in $\mathcal M$}\bigr\}.
\]
If $\operatorname{reach}(\mathcal M)\ge\rho$, the nearest-point projection
$\Pi_{\mathcal M}(y):=\arg\min_{x\in\mathcal M}\|y-x\|$ is well defined on
$\mathcal M_\rho:=\{y:\operatorname{dist}(y,\mathcal M)<\rho\}$.

\paragraph*{Uniform regularity constants.}\enspace
For fixed $B_{\rm reg},B_*,M_{\rm reg},m_{\rm reg},\rho_{\rm reg}>0$, the
baseline regularity inequalities underlying Assumption~\ref{G2} are
\begin{equation}\label{eq:uniform-regularity-bounds}
\begin{aligned}
 \|h\|_{C^{\beta+1}}&\le B_{\rm reg},
 &m_{\rm reg}\le\sigma_{\min}\{Dh(u)\}
 &\le\sigma_{\max}\{Dh(u)\}\le M_{\rm reg},\\
 \operatorname{reach}\{h(\mathcal U)\}&\ge\rho_{\rm reg}.
\end{aligned}
\end{equation}
Uniform interiority means that some $\delta_{\rm reg}>0$ is common to the
truth class and, for every $h^*\in\mathcal H_*$,
\begin{equation}\label{eq:uniform-interiority-bounds}
\begin{aligned}
 \|h^*\|_{C^{s_\star}}&\le B_*,
 &\|h^*\|_{C^{\beta+1}}&\le B_{\rm reg}-\delta_{\rm reg},\\
 m_{\rm reg}+\delta_{\rm reg}
 &\le\sigma_{\min}\{Dh^*(u)\}
 \le\sigma_{\max}\{Dh^*(u)\}
 \le M_{\rm reg}-\delta_{\rm reg},\\
 \operatorname{reach}\{h^*(\mathcal U)\}
 &\ge\rho_{\rm reg}+\delta_{\rm reg}.
\end{aligned}
\end{equation}
For the candidates, let $\mathcal H_{\rm reg}$ denote the class of
$C^{\beta+1}$ embeddings satisfying~\eqref{eq:uniform-regularity-bounds}
with separate constants
$\widetilde B_{\rm reg},\widetilde M_{\rm reg},
\widetilde m_{\rm reg},\widetilde\rho_{\rm reg}>0$.
These constants are determined by the finite construction in
Proposition~\ref{prop:canonical-law-sieve} and are independent of $n$ and $J$.
This candidate envelope may have larger upper bounds
and smaller positive lower bounds.  The truth class and
\eqref{eq:uniform-interiority-bounds} retain the original constants.
Set $\rho_{\mathcal M}:=\min(\rho_{\rm reg},\widetilde\rho_{\rm reg})$.

\paragraph*{Proof roadmap.}\enspace
Appendix~\ref{sec:vs-main-proof} proves the general score-ERM transfer bound.
Appendix~\ref{sec:mlsn-proof} applies it to prove Theorem~\ref{thm:acc} by
controlling the population approximation, quadratic statistical, and
exceptional-event terms. The construction, analytic estimates, and
localization and root bounds needed for this step are developed in
Appendices~\ref{sec:multiscale-tools}--\ref{sec:pilot-verification}, respectively.
Appendix~\ref{sec:minimax-lower-bound} gives the matching lower bound, while
Appendix~\ref{sec:finite-anchor-specialization} treats the optional finite-anchor
extension separately from the exact-loss result.
Figure~\ref{fig:proof-structure} summarizes these dependencies.

\begin{figure}[H]
\centering
\begingroup
\renewcommand{\baselinestretch}{1}
\begin{tikzpicture}[
  x=1pt,y=1pt,
  every node/.style={font=\fontsize{10}{12}\selectfont,align=center},
  proofbox/.style={draw=black!65,line width=0.5pt,fill=black!2,
    inner xsep=5pt,inner ysep=6pt,outer sep=0pt},
  support/.style={proofbox,text width=132pt,minimum height=64pt},
  error/.style={proofbox,text width=132pt,minimum height=70pt},
  wide/.style={proofbox,text width=440pt,fill=black!5},
  arrow/.style={->,>=stealth,line width=0.6pt,draw=black!70}
]
\node[support] (geometry) at (72,0) {
  \textbf{\ref{sec:multiscale-tools}: Sieve construction}\\
  Law charts; RePU features;\\
  approximation and\\
  $W_2$ comparison
};
\node[support] (analysis) at (226,0) {
  \textbf{\ref{sec:weighted-local-estimation}: Analytic estimates}\\
  Inverse bounds; weak tests;\\
  noise and curvature
};
\node[support] (roots) at (380,0) {
  \textbf{\ref{sec:pilot-verification}: Pilot and roots}\\
  Pilot; interior minimizers;\\
  uniform coordinate errors
};
\draw[arrow] (geometry.east) -- (analysis.west);
\draw[arrow] (analysis.east) -- (roots.west);
\draw[draw=black!70,line width=0.6pt]
  (geometry.south) -- (72,-43) -- (380,-43) -- (roots.south);
\draw[draw=black!70,line width=0.6pt] (analysis.south) -- (226,-43);
\node at (226,-61) (verification) {
  \textbf{\ref{sec:mlsn-proof}: Bound the three errors at admissible resolutions}
};
\draw[arrow] (226,-43) -- (verification.north);
\node[error] (bias) at (72,-111) {
  \textbf{Population approximation}\\[3pt]
  $b_{\tau_J}\lesssim 2^{-J(\eta+1)}$
};
\node[error] (statistical) at (226,-111) {
  \textbf{Statistical estimation}\\
  Square-root quadratic term\\[2pt]
  $\lesssim 2^{-J(\eta+1)}$\\
  \({}+n^{-1/2}2^{J(d^*-2)/2}\)
};
\node[error] (exception) at (380,-111) {
  \textbf{Exceptional events}\\
  Pilot and fitting failures\\[3pt]
  $O(n^{-2})$
};
\draw[draw=black!70,line width=0.6pt]
  (bias.south) -- (72,-158) -- (380,-158) -- (exception.south);
\draw[draw=black!70,line width=0.6pt] (statistical.south) -- (226,-158);
\node[wide,minimum height=48pt] (transfer) at (226,-198) {
  \textbf{Theorem~\ref{thm:local-erm-transfer} (\ref{sec:vs-main-proof}):
    Transfer to Wasserstein error}\\
  Apply the conditional bound, then average over the pilots:\\[2pt]
  $\mathbb E W_2(P^*,P_{\widehat h_J^{\rm NN}})
    \lesssim 2^{-J(\eta+1)}+n^{-1/2}2^{J(d^*-2)/2}+n^{-2}$
};
\draw[arrow] (226,-158) -- (transfer.north);
\node[wide,minimum height=46pt] (rate) at (226,-284) {
  \textbf{Theorem~\ref{thm:acc}: Intrinsic-dimension rate}\\[3pt]
  $\displaystyle
   \sup_{h^*\in\mathcal H_*}\mathbb E
   W_2(P_{h^*},P_{\widehat h_J^{\rm NN}})
   \lesssim n^{-(\eta+1)/(2\eta+d^*)}$
};
\draw[arrow] (transfer.south) -- (rate.north);
\node[anchor=west,align=left] at (234,-243) {
  Balance: $2^J\asymp n^{1/(2\eta+d^*)}$
};
\node at (226,-326) {\textit{Separate results}};
\node[proofbox,text width=209pt,minimum height=50pt] at (110.5,-362) {
  \textbf{\ref{sec:minimax-lower-bound}: Minimax comparison}\\
  Matching lower bound\\
  for the interior subclass
};
\node[proofbox,text width=209pt,minimum height=50pt] at (341.5,-362) {
  \textbf{\ref{sec:finite-anchor-specialization}: Optional finite anchors}\\
  Fixed-proposal, fixed-hidden\\
  anchor-averaged extension
};
\end{tikzpicture}
\captionsetup{font=small,skip=6pt}
\caption{Proof structure. Appendix~\ref{sec:mlsn-proof} combines the supporting
estimates with Theorem~\ref{thm:local-erm-transfer} to prove
Theorem~\ref{thm:acc}, for $\eta>2$ and $d^*>2$.
The last row is separate from the upper-bound argument.}
\label{fig:proof-structure}
\endgroup
\end{figure}

\subsection{Proof of Theorem~\ref{thm:local-erm-transfer}}
\label{sec:vs-main-proof}

\paragraph*{Conditional setup.}
As in the main text, $\theta=(\thetah,\thetao)$ comprises hidden-feature
parameters and output coefficients, fitted jointly.
Let $\Theta_{\rm h}$ be a compact set of hidden parameters and let
$\Theta_{\rm o}(\thetah)\subset\mathbb R^p$ be a compact convex output
set for each $\thetah\in\Theta_{\rm h}$.  The joint parameter set consists
of all pairs $(\thetah,\thetao)$ with $\thetao\in\Theta_{\rm o}(\thetah)$.
Let $\theta_{\rm o}^\dagger(\thetah)$ be an interior population minimizer
of $R_\tau(h_{(\thetah,\thetao)})$ over this output set, and write
$h^\dagger_{\thetah}:=h_{(\thetah,\theta_{\rm o}^\dagger(\thetah))}$.
Choose $\widehat\theta_{\rm o}(\thetah)$ to minimize
$L_{n,t}(h_{(\thetah,\thetao)})$ over $\Theta_{\rm o}(\thetah)$.
The selection is jointly measurable in the data and $\thetah$, with
$\widehat\theta_{\rm o}(\widehat\theta_{\rm h})=\widehat\theta_{\rm o}$.
Joint optimality ensures that the fitted output is a conditional minimizer.

\paragraph*{Score gradient and path Hessian.}
With the notation of Theorem~\ref{thm:local-erm-transfer}, fix
$\thetah\in\Theta_{\rm h}$.
Along the output-coordinate segment, put
\[
 \theta(s;\thetah):=
 \bigl(\thetah,\theta_{\rm o}^\dagger(\thetah)
 +s\{\widehat\theta_{\rm o}(\thetah)
          -\theta_{\rm o}^\dagger(\thetah)\}\bigr),
 \qquad 0\le s\le1,
\]
and define
\begin{align*}
 g_{n,\tau}(\thetah)
 &:=(1-t)(P_n-P)\nabla_{\thetao}\ell_t
 (X;h^\dagger_{\thetah}),\\
 \widetilde{\bm H}_{n,\tau}(\thetah)
 &:=\int_0^1
 \left.\nabla_{\thetao}^2\{(1-t)L_{n,t}(h_\theta)\}
 \right|_{\theta=\theta(s;\thetah)}\,ds.
\end{align*}
Both derivatives are taken only in the output coordinates, with the hidden
configuration held fixed.

\paragraph*{Uniform conditional version.}
Choose positive-definite matrices $\bm M(\thetah)$ and deterministic
positive-semidefinite matrices $\bm A_\tau(\thetah)$.
Here $\bm M(\thetah)$ defines the weighted norm used to translate
output-coefficient error into Wasserstein error.
For the spline-RePU construction one may take the fixed matrix
$\bm M=\bm W_J$; Appendix~\ref{sec:mlsn-proof} verifies this choice.
Suppose
$W_2(P_{h^*},P_h)\le B$ over the candidate class and
\[
\begin{aligned}
 b_\tau&:=\sup_{\thetah\in\Theta_{\rm h}}
 W_2(P_{h^*},P_{h^\dagger_{\thetah}})<\infty,\\
 W_2(P_{h_\theta},P_{h^\dagger_{\thetah}})
 &\le C_{\bm M}
 \|\thetao-\theta_{\rm o}^\dagger(\thetah)\|_{\bm M(\thetah)}
 \quad\text{uniformly over the displayed parameters.}
\end{aligned}
\]
On an event $\mathcal E_n$ of probability at least $1-\varepsilon_n$,
assume that the population and empirical output minimizers for every
hidden configuration are interior score roots, every
$\widetilde{\bm H}_{n,\tau}(\thetah)$ is invertible, and
\[
 \|\widetilde{\bm H}_{n,\tau}(\thetah)^{-1}z\|_{\bm M(\thetah)}^2
 \le z^\top\bm A_\tau(\thetah)z,
 \qquad z\in\mathbb R^p,\quad\thetah\in\Theta_{\rm h}.
\]
Then, for any measurable hidden selection followed by its chosen
conditional empirical output minimizer, the fitted transport satisfies
\[
 \mathbb EW_2(P_{h^*},P_{\widehat h})
 \le b_\tau+C_{\bm M}
 \left[\mathbb E\left\{
 \mathbf1_{\mathcal E_n}\sup_{\thetah\in\Theta_{\rm h}}
 g_{n,\tau}(\thetah)^\top\bm A_\tau(\thetah)
 g_{n,\tau}(\thetah)\right\}\right]^{1/2}
 +B\varepsilon_n.
\]

\paragraph*{Joint-ERM specialization.}
For the joint estimator in the main text, the unindexed quantities are
the corresponding conditional quantities evaluated at its fitted hidden
configuration:
\[
\begin{aligned}
 \theta^\dagger
   &=(\widehat\theta_{\rm h},
       \theta_{\rm o}^\dagger(\widehat\theta_{\rm h})),&
 \bm M&=\bm M(\widehat\theta_{\rm h}),\\
 g_{n,\tau}&=g_{n,\tau}(\widehat\theta_{\rm h}),&
 \bm A_\tau&=\bm A_\tau(\widehat\theta_{\rm h}),\\
 \widetilde{\bm H}_{n,\tau}
   &=\widetilde{\bm H}_{n,\tau}(\widehat\theta_{\rm h}).
\end{aligned}
\]
Thus the uniform conditional assumptions imply the conditions of
Theorem~\ref{thm:local-erm-transfer}.  The main-text quadratic score term is
bounded pointwise by the conditional supremum above.  Its expectation
retains the data dependence of the fitted hidden configuration.
Theorem~\ref{thm:local-erm-transfer} also allows a measurable matrix
$\bm A_\tau$ depending on the full fitting data. The proof below is
pathwise and applies to that choice; Appendix~\ref{sec:mlsn-proof}
uses the empirical path Hessian to define this matrix.

\begin{proof}[Proof of Theorem~\ref{thm:local-erm-transfer} and the conditional version]
For the joint estimator, output optimality and population stationarity
give, on $\mathcal E_n$,
\[
 0=g_{n,\tau}+\widetilde{\bm H}_{n,\tau}
      (\widehat\theta_{\rm o}-\theta_{\rm o}^\dagger).
\]
Using~\eqref{eq:general-path-stability} and then
\eqref{eq:general-chart-w2},
\[
\begin{aligned}
 \|\widehat\theta_{\rm o}-\theta_{\rm o}^\dagger\|_{\bm M}^2
   &\le g_{n,\tau}^{\top}\bm A_\tau g_{n,\tau},\\
 W_2(P_{h^*},P_{\widehat h})
   &\le b_\tau+C_{\bm M}
        (g_{n,\tau}^{\top}\bm A_\tau g_{n,\tau})^{1/2}.
\end{aligned}
\]
Cauchy--Schwarz bounds the expectation on $\mathcal E_n$, and its
complement contributes at most $B\varepsilon_n$.
This proves~\eqref{eq:general-local-erm}.

For the conditional version, the same argument applies uniformly.
Fix $\thetah\in\Theta_{\rm h}$.  Population stationarity at
$\theta_{\rm o}^\dagger(\thetah)$ and the two interior score equations give
\[
 0=g_{n,\tau}(\thetah)
   +\widetilde{\bm H}_{n,\tau}(\thetah)
    \{\widehat\theta_{\rm o}(\thetah)
          -\theta_{\rm o}^\dagger(\thetah)\}.
\]
Hence, on $\mathcal E_n$,
\[
 \|\widehat\theta_{\rm o}(\thetah)
       -\theta_{\rm o}^\dagger(\thetah)\|_{\bm M(\thetah)}^2
 \le g_{n,\tau}(\thetah)^\top\bm A_\tau(\thetah)g_{n,\tau}(\thetah).
\]
Apply this inequality at the measurable data-selected
$\widehat\theta_{\rm h}$.  The triangle inequality and
the uniform transport comparison bound the Wasserstein error on $\mathcal E_n$
by $b_\tau$ plus $C_{\bm M}$ times the square root of the supremum of the
displayed quadratic form over $\thetah\in\Theta_{\rm h}$.
Cauchy--Schwarz controls its expectation, while the complement contributes
at most $B\varepsilon_n$.  This proves the conditional bound.
\end{proof}

\paragraph*{Fixed-hidden covariance interpretation.}
For one fixed hidden configuration and a deterministic
$\bm A_\tau(\thetah)$, let $\bm G_\tau(\thetah)$ be the
covariance of the one-observation vector
$(1-t)\nabla_{\thetao}\ell_t(X;h^\dagger_{\thetah})$ under $P_{h^*}$.
Independence of the fitting observations gives
\[
 \mathbb E\!\left[
 g_{n,\tau}(\thetah)^\top\bm A_\tau(\thetah)g_{n,\tau}(\thetah)
 \right]
 =n^{-1}\operatorname{tr}
 \{\bm A_\tau(\thetah)\bm G_\tau(\thetah)\}.
\]
Since the quadratic form is nonnegative, inserting
$\mathbf1_{\mathcal E_n}$ can only decrease this expectation.

\subsection{Proof of Theorem~\ref{thm:acc}}
\label{sec:mlsn-proof}
\begingroup
\setlength{\abovedisplayskip}{6pt plus 2pt minus 2pt}
\setlength{\belowdisplayskip}{6pt plus 2pt minus 2pt}

We verify the population approximation, statistical estimation, and
exceptional-event terms of Theorem~\ref{thm:local-erm-transfer}.
The law-coordinate construction and RePU realization are given in
Appendix~\ref{sec:multiscale-tools}, the analytic estimates in
Appendix~\ref{sec:weighted-local-estimation}, and the pilot, interior roots,
and weak-coordinate bounds in Appendix~\ref{sec:pilot-verification}.

\paragraph*{Resolution and law coordinates.}
In the RePU upper-bound analysis, $n=n_{\rm fit}=|I_{\rm fit}|$ is the
fitting-sample size, comparable to the total size under the three-way split.
Relabel its clean observations as $X_i=Y_0^i$, so $P_{n,{\rm fit}}=P_n$.
Put $h_J=2^{-J}$, $s=\tau_J=c_th_J^2$, and $\epsilon=\sqrt{s}$,
with sufficiently small fixed $c_t>0$; the theorem uses
$h_J\asymp n^{-1/(2\eta+d)}$.

On branch $m$, let $\bar{\mathcal M}_m$, $\bar\nu_m$, and $\bar f_m$
denote the reference manifold, volume, and density from
Appendix~\ref{sec:multiscale-tools}.
The field $r_\theta$ is a zero-integral density perturbation and
$v_\theta$ a normal displacement. Write
$q_\theta=\bar f_m+r_\theta$, $\Phi_\theta=I+v_\theta$, and
$P_\theta=(\Phi_\theta)_\#(q_\theta\bar\nu_m)$.
The fields $(r^*,v^*)$ represent the truth and need not lie in the
retained space. When coordinates $\theta$ or $e$ induce $h$, write
$\ell_\tau(x;\theta)$ or $\ell_\tau(x;e)$ for $\ell_\tau(x;h)$, and set
$L_{n,J}(\theta):=P_n\ell_{\tau_J}(\,\cdot\,;h_\theta)$.

\paragraph*{Error norms.}
On this reference branch define
\begin{align}
 \mathcal W(r,v)&:=\|r\|_{H^{-1}(\bar{\mathcal M}_m)}
                         +\|v\|_{L^2(\bar\nu_m)},\nonumber\\
 \|(a,b)\|_{E_p}&:=\|a\|_{W^{1,p}}+s^{-1}\|b\|_{L^p},
 \qquad E:=E_2,\quad 1<p<\infty.
 \label{eq:r25-energy}
\end{align}
These norms act on fields. For an independent output-coefficient
increment $a=(a^{\rm den},a^{\rm nor})$, use
\begin{equation}\label{eq:vs-W-matrix}
 \|a\|_{\bm W_J}^2
 :=\sum_{j=-1}^J2^{-2j}\|a_j^{\rm den}\|_2^2
       +\sum_{j=-1}^J\|a_j^{\rm nor}\|_2^2.
\end{equation}
Let $q_j^{\rm den},q_j^{\rm nor}$ count the level-$j$ coefficients
and $Q_J:=\sum_{j=-1}^J(q_j^{\rm den}+q_j^{\rm nor})\asymp h_J^{-d}$.

\paragraph*{Conditioning and comparison fits.}
Fix a realization $z$ of both pilot samples in the good event
$\mathcal G_{n,J}^{\rm pil}$ of Lemma~\ref{lem:vs-observable-pilot}.
The selected reference branch $m$ and the joint
domain~\eqref{eq:mlsn-joint-domain} are then fixed, and the fitting
observations remain independent with law $P^*$.
Let $\mathcal E_n$ be the observation-space event~\eqref{eq:r25-noise-event},
with conditional failure probability at most $Cn^{-2}$.
All constants below are uniform over good pilot realizations, admissible
branches, and hidden configurations.

Let $\theta_{\rm o}^\dagger(\thetah)$ be the population output minimizer
from Lemma~\ref{lem:r26-population-root}, and let
$\widehat\theta_{\rm o}(\thetah)$ be a measurable empirical output
minimizer in the same original fiber. On $\mathcal E_n$ the latter is
the unique interior root by Lemma~\ref{lem:pm-newton-margins}.
At the jointly fitted hidden configuration, set
\[
 \theta^\dagger=
 (\widehat\theta_{\rm h},\theta_{\rm o}^\dagger(\widehat\theta_{\rm h})),
 \qquad
 \Delta=\widehat\theta_{\rm o}
          -\theta_{\rm o}^\dagger(\widehat\theta_{\rm h}).
\]
The estimator remains the joint ERM; fixing its hidden configuration is
only a step in the comparison argument.

\paragraph*{Population approximation.}
Lemma~\ref{lem:mlsn-architecture} supplies a retained projection comparator.
Lemma~\ref{lem:r26-population-root} separately constructs the population
minimizer, and Lemma~\ref{lem:mlsn-consequences} gives
\[
 \mathcal W(r^\dagger_{\thetah}-r^*,v^\dagger_{\thetah}-v^*)
 \le Ch_J^{\eta+1}
 \quad\text{for every }\thetah.
\]
The transport comparison in Lemma~\ref{lem:normal-density-chart} therefore
permits the deterministic choice
\begin{equation}\label{eq:r26-population-bias}
 b_{\tau_J}=Ch_J^{\eta+1}.
\end{equation}

\paragraph*{Statistical estimation.}
Choose the fixed positive-definite matrix $\bm M=\bm W_J$
defined in~\eqref{eq:vs-W-matrix}.
For an independent output-coefficient increment $a$, the synthesis norm
equivalence in Lemma~\ref{lem:normal-density-chart} and the uniform
pullback bounds of Lemma~\ref{lem:mlsn-architecture} give
\[
 \|a\|_{\bm W_J}\asymp
 \mathcal W(S_{J,m,\thetah}a),
\]
where $S_{J,m,\thetah}$ synthesizes the density and normal fields
in~\eqref{eq:mlsn-profile-synthesis}.
The law-chart comparison therefore verifies~\eqref{eq:general-chart-w2}
for the empirical and population fits in the same fiber, with uniform
$C_{\bm M}$, for every fitting realization at the fixed good pilot,
not only on $\mathcal E_n$.

Let $g_{n,\tau_J}$ and $\widetilde{\bm H}_{n,\tau_J}$ be the selected
gradient and path Hessian of Appendix~\ref{sec:vs-main-proof}.
Population centering is performed before evaluating the gradient at
the data-selected hidden configuration. On $\mathcal E_n$, the two
score-root identities and whole-fiber curvature imply
\[
 g_{n,\tau_J}+\widetilde{\bm H}_{n,\tau_J}\Delta=0.
\]
For this application choose the measurable matrix
\begin{equation}\label{eq:r26-path-matrix}
 \bm A_{\tau_J}:=
 \begin{cases}
 \widetilde{\bm H}_{n,\tau_J}^{-\top}\bm W_J
       \widetilde{\bm H}_{n,\tau_J}^{-1},&\mathcal E_n,\\
 0,&\mathcal E_n^c.
 \end{cases}
\end{equation}
It is positive semidefinite and satisfies
\eqref{eq:general-path-stability} with equality on $\mathcal E_n$.
In particular,
\begin{equation}\label{eq:r26-quadratic-identity}
 \mathbf1_{\mathcal E_n}g_{n,\tau_J}^{\top}\bm A_{\tau_J}g_{n,\tau_J}
 =\mathbf1_{\mathcal E_n}\|\Delta\|_{\bm W_J}^2.
\end{equation}
The uniform coordinate second moment in
Lemma~\ref{lem:mlsn-consequences} consequently gives
\begin{equation}\label{eq:r26-quadratic-bound}
 \mathbb E_{\rm fit}\!\left[
 \mathbf1_{\mathcal E_n}g_{n,\tau_J}^{\top}\bm A_{\tau_J}g_{n,\tau_J}
 \mid z\right]
 \le C\{h_J^{2\eta+2}+n^{-1}h_J^{2-d}\}.
\end{equation}
The approximation remainder comes from comparing the empirical and
population coordinates through the truth. At the chosen resolution it
has the same order as the variance term. The matrix in
\eqref{eq:r26-path-matrix} depends on the empirical path, so we apply the
measurable-matrix statement of Theorem~\ref{thm:local-erm-transfer}.
The uniform coordinate second moment in Lemma~\ref{lem:mlsn-consequences}
supplies the required quadratic bound.

\paragraph*{Exceptional events and conclusion.}
On the good pilot event the observable safety test holds, and the root
construction proves nonemptiness of every output fiber. Thus the estimator
uses the prescribed joint minimization. Proposition~\ref{prop:canonical-law-sieve}
and the construction in Appendix~\ref{sec:repu-class} give a common
bounded support for all feasible laws and the fixed fallback on every
pilot realization. Let $B$ bound their Wasserstein distances to the truth.
Theorem~\ref{thm:local-erm-transfer}, applied conditionally on $z$, gives
\begin{align*}
 \mathbb E_{\rm fit}[W_2(P^*,P_{\widehat h_J^{\rm NN}})\mid z]
 &\le b_{\tau_J}+C_{\bm M}
 \left[\mathbb E_{\rm fit}\{
 \mathbf1_{\mathcal E_n}g_{n,\tau_J}^{\top}\bm A_{\tau_J}g_{n,\tau_J}
 \mid z\}\right]^{1/2}+Cn^{-2}\\
 &\le C\{h_J^{\eta+1}+n^{-1/2}h_J^{1-d/2}+n^{-2}\}.
\end{align*}
Integrating over good pilot realizations and using
$B\mathbb P\{(\mathcal G_{n,J}^{\rm pil})^c\}\le Cn^{-2}$ proves
\begin{equation}\label{eq:r25-final-bound}
 \sup_{h^*\in\mathcal H_*}\mathbb E
 W_2(P_{h^*},P_{\widehat h_J^{\rm NN}})
 \le C\{h_J^{\eta+1}+n^{-1/2}h_J^{1-d/2}+n^{-2}\}.
\end{equation}
The same conditional argument holds for a fixed hidden configuration.
More generally, this bound applies at resolutions where the pilot
accuracy, correction-ball margins, and smallness conditions in
Lemmas~\ref{lem:vs-observable-pilot} and~\ref{lem:pm-newton-margins}
are satisfied.

Choose $J_n=\lfloor\log_2n/(2\eta+d)\rfloor$.
The pilot condition $Q_{J_n}\log n/n\to0$ and all three limits in
\eqref{eq:r25-three-smallness} hold for every fixed $\eta>2,d>2$.
Both leading terms in~\eqref{eq:r25-final-bound} then have order
$n^{-(\eta+1)/(2\eta+d)}$, and $n^{-2}$ is smaller.
The factors $h_J^2$ in the weak-density step multiply an energy error
or a noise norm, so there is no standalone $h_J^2$ smoothing bias.
The degree and dual filters are chosen at the actual $\eta$, preserving
the Jackson orders also for $\eta>3$.
This proves Theorem~\ref{thm:acc}. The matching lower bound is given
separately in Appendix~\ref{sec:minimax-lower-bound}.
\endgroup

\subsection{Law coordinates and spline-RePU construction}\label{sec:multiscale-tools}
Recall the geometric notation and uniform regularity and interiority
bounds~\eqref{eq:uniform-regularity-bounds}--\eqref{eq:uniform-interiority-bounds}
from the beginning of Appendix~\ref{sec:proofs}.
We develop law coordinates, reference approximations, and RePU features
in turn, supplying the candidate families and approximation bounds
used in Appendix~\ref{sec:mlsn-proof}.

\paragraph*{Technical conventions.}\enspace
In mesh-dependent estimates below, abbreviate $h_J$ by $h$.
Fix an admissible reference branch, suppress $m$, and write
$\nu_0=\bar\nu_m$ and $\bar f=\bar f_m$ in the law coordinates
of Appendix~\ref{sec:mlsn-proof}.
For local calculations the hidden component is fixed and $\theta$
abbreviates the output-coordinate vector; these derivatives are not
with respect to the full neural parameter. The fields
$(r_\theta,v_\theta)$ are affine in these coefficients.
The notation $\theta^*$ in a law-coordinate difference means the
full fields representing the truth; it does not posit an identifiable
true hidden parameter or a retained truth.
The norms $\mathcal W$, $E_p$, and $\bm W_J$ are defined in
Appendix~\ref{sec:mlsn-proof}; reference and corresponding physical
norms are uniformly equivalent.
We also use the variance budget
\begin{equation}\label{eq:vs-variance-budget}
 \mathfrak B_J:=\sum_{j=-1}^J(2^{-2j}+\tau_J)
                    (q_j^{\rm den}+q_j^{\rm nor}).
\end{equation}
For the chart-transition estimates alone, fix
$1<\nu_{\rm mix}<\beta-1$ and $2<\nu_{\rm diag}<\beta$.
These orders exist for every $\eta>2$; no stronger weighted-inverse
decay condition is used below.
\subsubsection{Intrinsic law coordinates}
We establish the multiscale basis and density--normal coordinates that
represent law perturbations and control their Wasserstein distance.

\paragraph*{Reference-frame and normal--density coordinates.}

{Fix a finite smooth cubulation of the compact source manifold
$\mathcal U$. Choose one
biorthogonal composite spline multiwavelet system, of regularity and
cancellation above $s_\star+2$, by the ordered collar
restriction--extension construction of
\citet{dahmen1999wavelets,jouini2007wavelet}.  Its analysis and synthesis maps
are bounded inverses on every $H^s$, $-1\le s\le s_\star+2$; hence coordinate
truncation defines a finite-rank projector $\Pi_J$.  Let $\mathsf P_j$
denote its physical level-$j$ block; in coefficient representations, the same
notation denotes the corresponding coordinate cutoff.  Uniformly,}
\begin{equation}\label{eq:vs-wavelet-properties}
\begin{aligned}
|\mathcal K_j|&\asymp2^{jd},
&\|\Pi_Jg\|_{L^2}&\le C\|g\|_{L^2},\\
\|(I-\Pi_J)g\|_{H^\mu(\mathcal U)}
&\le C2^{-J(\mu'-\mu)}\|g\|_{C^{\mu'}(\mathcal U)},
&&-1\le \mu<\mu'\le s_\star .
\end{aligned}
\end{equation}
{The last inequality follows from the
$B_{\infty,\infty}^{\mu'}$ coefficient bound and geometric summation for
$\mu<\mu'$.  The same coefficients give the asserted $H^{-1}$ and
$C^\alpha$ tails, while the
finite spaces satisfy the corresponding Bernstein inequalities.
More explicitly, the ordered collar construction supplies bounded analysis
and synthesis maps $\mathcal D_s:H^s\to\ell_2^s$ and
$\mathcal A_s:\ell_2^s\to H^s$ with
$\mathcal D_s\mathcal A_s=I$ and $\mathcal A_s\mathcal D_s=I$ on $H^s$.
Thus
\begin{equation}\label{eq:vs-collar-projector-factorization}
 \Pi_J=\mathcal A_s\mathbf1_{\{j\le J\}}\mathcal D_s,
 \qquad
 I-\Pi_J=\mathcal A_s\mathbf1_{\{j>J\}}\mathcal D_s.
\end{equation}
The dual system is obtained from the transposed collar recursion, so the same
factorization with $\mathcal D_s^*$ and $\mathcal A_s^*$ proves the negative
Sobolev bounds by duality.  Equation~\eqref{eq:vs-collar-projector-factorization}
also shows that no redundant local family is truncated.
Write $\mathsf P_{\le J}=\sum_{j=-1}^J\mathsf P_j=\Pi_J$ and
$\Lambda=(I-\Delta_{\mathcal M^*})^{1/2}$.

On each reference image $\bar{\mathcal M}_m$, apply the same construction
componentwise in fixed local frames of its smooth normal bundle.  The bundle
transition matrix enters the collar synthesis and its transpose enters the
dual analysis.  Therefore the resulting primal and dual normal atoms are
biorthogonal, and their coordinate truncation is a uniformly stable
projector on bundle Sobolev spaces.  They obey the same Jackson, Bernstein,
localization, and almost-diagonal estimates as the scalar bank.  At level
$j$ their indices consist of a fixed macro-patch, a scalar cell, and one of
$D-d$ bundle components, so
$q_j^{\rm nor}\asymp(D-d)2^{jd}$.  This fixed bundle-valued bank, rather than
a finite truncation of the redundant family
$P_{N_x}\{\psi_{jk}(x)e_a\}$, is the normal bank in
Definition~\ref{def:repu-construction}.}

\paragraph*{Fixed generator orders and terminal projection bounds.}
The common primal--dual range can be obtained without raising the specified
RePU degree $q$.  Set $N=M=q+1$, choose an integer $s_D>\eta+3$, and take
the spline pair of \citet[Section~6.A]{cohen1992biorthogonal}, with dual
order $\widetilde N\ge N$ of the same parity as $N$.  Choose it so that
\[
 (1-c_{\rm CDF})\widetilde N>c_{\rm CDF}N+s_D+1,
 \qquad c_{\rm CDF}=\tfrac14\log_2(256/27)<1.
\]
Here is an explicit regularity check.  For $k=(N+\widetilde N)/2$ the dual
mask, apart from its phase, is
$(\cos(\xi/2))^{\widetilde N}P_k(\sin^2(\xi/2))$, where
$P_k(t)=\sum_{n<k}\binom{k+n-1}{n}t^n$.
Since $P_k(1/2)=2^{k-1}$,
$P_k(t)\le\max(2,4t)^k$ on $[0,1]$.  Under frequency doubling
$t$ becomes $4t(1-t)$, and
$\max_t\max(2,4t)\max(2,16t(1-t))=256/27$.
Pairing consecutive factors in the refinement product bounds the dual
Fourier transform by
$C(1+|\xi|)^{-\widetilde N+c_{\rm CDF}(N+\widetilde N)}$.
The low-frequency product is bounded because $P_k(t)=1+O(t)$.
Fourier inversion therefore gives dual $C^{s_D}$ regularity.
The residual-mask bounds in \citet[Proposition~4.9]{cohen1992biorthogonal}
are $B_1=1<2^{N-1/2}$ and
$\widetilde B_2\le2^{c_{\rm CDF}(N+\widetilde N)}
<2^{\widetilde N-s_D-1}<2^{\widetilde N-1/2}$.
Together with the spline-pair reconstruction identity and the Fourier
decay, these verify the biorthogonal Riesz pair; the primal remains a
degree-$q$ spline, with regularity
$C^{q-1,1}$.  The complementary-boundary construction of
\citet[Sections~4--5]{dahmen1998complementary} uses finite combinations of
these same translates and next-level generators.  It preserves the
primal degree and the two regularity ranges.  Tensorization and the
scale-modified collar construction then give the declared common
Besov range above $R_0=\eta+3$.  Only fixed filter support and fixed
prototype width change with the selected dual order, not $q$ or $J$.

\begin{lemma}[Low-order kernel of the composite cutoff]
\label{lem:r25-cutoff}
For the full scalar and normal-bundle cutoffs, fix $m_0=5/2$.
Uniformly over the reference branches and hidden pullbacks,
\begin{align}
 |D_x^aD_y^bK_h(x,y)|&\le C h^{-d-|a|-|b|},
       &&\operatorname{dist}(x,y)\le C_0h,\label{eq:r25-kernel-near}\\
 |D_x^aD_y^bK_h(x,y)|&\le C h^{m_0}
       \operatorname{dist}(x,y)^{-d-m_0-|a|-|b|},
       &&\operatorname{dist}(x,y)\ge C_0h,\label{eq:r25-kernel-far}
\end{align}
for $|a|+|b|\le2$. In particular,
\begin{equation}\label{eq:r25-kernel-convenient}
 |D_x^aD_y^bK_h(x,y)|\le C h^{-d-|a|-|b|}
       (1+\operatorname{dist}(x,y)/h)^{-d-m_0}.
\end{equation}
The corresponding Jackson and Bernstein bounds hold at the declared
orders. No compact support of the global dual atoms is required.
\end{lemma}
\begin{proof}
Let $S_i$ be the synthesis extension from the local inflow trace space
on macro-patch $\Gamma_i$, and $T_i$ the ordered residual restriction.
Include physical volume factors in their adjoints. The ordered
restriction--extension construction gives
\[
 T_i u=(u-\sum_{k<i}S_kT_ku)|_{\Gamma_i},\qquad
 T_iS_k=\delta_{ik}I,\qquad \sum_iS_iT_i=I.
\]
For the local complementary-boundary cutoff $Q_{i,h}$, therefore,
\begin{equation}\label{eq:r25-ordered-cutoff}
 \Pi_h=\sum_iS_iQ_{i,h}T_i,\qquad
 \Pi_h-I=\sum_iS_i(Q_{i,h}-I)T_i.
\end{equation}
The local cutoffs and their adjoints have paired-support locality at
distance $Ch$. The ordered identities and bounded primal and physical
adjoint extensions on their respective trace spaces follow from
\citet[Theorems~2.4.1, 2.5.1 and 4.3.1]{dahmen1999wavelets}.
Here the extension is the scale-modified one in their equation~(4.3.4):
an atom not meeting the outflow boundary is extended by zero, and only
boundary atoms undergo Hestenes extension. This choice is essential.
An unmodified reflection extension can have a distributional kernel on
a reflected diagonal, away from $x=y$; boundedness on Besov spaces alone
would not prove the estimates below. We give the additional pointwise
argument for the prescribed scale-modified construction.

Write $N_{\rm pat}$ for the fixed number of patches, and choose
\[
 L_N=100(N_{\rm pat}+1)2^{N_{\rm pat}+1},\qquad
 \delta>0,\qquad m_0+2+L_N\delta<\eta+3.
\]
Choose an auxiliary $p_0>2$ with $d/p_0<\delta$, avoiding the finitely
many exceptional trace endpoints. This is possible for every $\eta>2$.
It changes no filter or retained space. Expansion of the ordered residuals
has at most $2^{N_{\rm pat}}$ paths with at most $N_{\rm pat}+1$ factors;
two lift estimates per composition and the final Jackson estimate use
fewer than $L_N$ such margins.

For a compatible radius-$r$ molecule $f$ of amplitude $M_f$, whose
scaled derivatives through order $a+3\delta$ are bounded by $M_f$,
put $\vartheta=1/4<\min(1/p_0',d/p_0')$. The positive and negative
Besov estimates are
\[
 \|f\|_{B^{a+2\delta}_{p_0,p_0}}\le C M_f r^{d/p_0-a-2\delta},
 \qquad
 \|f\|_{B^{-\vartheta}_{p_0,p_0}}\le C M_f r^{d/p_0+\vartheta}.
\]
The negative estimate follows by dual Sobolev embedding on its support;
there is no extra boundary trace at this negative order. For a typed
lift $A=S_i$ or $T_i^*$, apply its positive and negative operator bounds
and Bernstein to each output block. Splitting at $2^\ell\asymp r^{-1}$ gives
\begin{equation}\label{eq:r25-lift}
 \begin{split}
 \|D^aAf\|_\infty
 &\le C M_f\sum_{\ell\ge0}2^{\ell(a+d/p_0)}
 \min\{2^{\ell\vartheta}r^{d/p_0+\vartheta},
        2^{-\ell(a+2\delta)}r^{d/p_0-a-2\delta}\}\\
 &\le C M_f r^{-a}.
 \end{split}
\end{equation}
Both sums converge since $d/p_0<\delta$. Slightly larger fractional
output orders use another fixed small margin, rather than $d/2$ derivatives.

Separate the diagonal identity and restriction distributions. For a
cross-interface piece of $S_i$, expand in local boundary wavelets.
At separated points of distance $r$, only $2^{-\ell}\ge cr$ contribute,
and bounded overlap gives
$|D_x^aD_z^bK_{S_i,\ell}(x,z)|\le C2^{\ell(d+a+b)}$.
Summing gives $Cr^{-d-a-b}$, also at the allowed fractional orders.
Indeed, both a local dual boundary atom and the extended primal boundary
atom are within $C2^{-\ell}$ of the same interface cell. Thus their
paired support cannot contain separated points unless $2^{-\ell}\ge cr$.
The Hestenes maps act on these atoms, not on an arbitrary input kernel;
their fixed dilation factors and the smooth chart changes preserve the
per-level estimate.

To compose two such typed maps $A,B$, freeze $x,y$ with distance $r$
and take a smooth cutoff $\chi$ near $x$, of radius $cr$, excluding $y$.
The near input $\chi(z)D_y^bK_B(z,y)$ is a compatible molecule of
amplitude $Cr^{-d-b}$; \eqref{eq:r25-lift} applied to $A$ gives
$Cr^{-d-a-b}$. In the far part, move to $B^*$ and split the smooth
row $D_x^aK_A(x,z)$ into annuli of radii $R\ge cr$. Each contributes
$CR^{-d-a-b}$ by the same lift estimate, and the sum is geometric.
Cutoffs are frozen at the point pair when derivatives are taken.
Induction in the formula for $T_i$ proves the estimates for every
ordered map. Residuals vanish on processed patches; adjoint rows have
the complementary outflow traces. Smooth cutoffs preserve these zero
jets. Regularized point masses can be used before passing to separated
kernel rows, so no incompatible zero extension is inserted.

For a term $A(Q_h-I)B$ in~\eqref{eq:r25-ordered-cutoff}, fix
$r\ge C_0h$ and $|a|+|b|\le2$. The near input
$\chi D_y^bK_B(\cdot,y)$ is a radius-$r$ molecule. Differentiated local
Jackson and locality make $(Q_h-I)$ of this input a radius-$Cr$
molecule with amplitude $Ch^{m_0}r^{-d-b-m_0}$. Its negative-order
norm retains the same factor: paired-support locality keeps its support
inside a fixed enlargement of the radius-$r$ support, and the support
Sobolev estimate above applies to its improved amplitude. To use the
positive norm in~\eqref{eq:r25-lift}, apply the differentiated Jackson
estimate through order $a+3\delta$ as well. Equation~\eqref{eq:r25-lift} then gives
$Ch^{m_0}r^{-d-a-b-m_0}$. For the far part transpose $Q_h-I$.
Its locality allows replacing the row of $A^*D_x^a\delta_x$ by a
smooth cutoff equal to that row on the entire $h$-neighborhood of the
far support. On a radius-$R$ annulus the adjoint Jackson error has
amplitude $Ch^{m_0}R^{-d-a-m_0}$; apply $B^*$ and sum over $R\ge cr$.
The derivatives requested are at most $m_0+a+b+L_N\delta<\eta+3$.
The off-diagonal kernel of $I$ is zero, proving~\eqref{eq:r25-kernel-far}.

Each global primal atom is its actual scale-modified extension of a
local atom. It has support diameter $C2^{-j}$, bounded overlap, and
$a$th derivative at most $C2^{j(d/2+a)}$. Its global dual is
$T_i^*\widetilde\psi_{i,j,k}$, whose $b$th derivative is bounded
everywhere by $C2^{j(d/2+b)}$ using~\eqref{eq:r25-lift}.
At fixed $x$ only a bounded number of primal atoms per level contribute.
Summation through $J$ proves~\eqref{eq:r25-kernel-near} without dual
locality. Physical adjoints give the reversed estimates.

Smooth bundle transition matrices, their transposes and volume factors
preserve the estimates componentwise. Conjugation by the prescribed
$O(h^2)$-time hidden flow preserves distances, Jacobians and all the
derivatives used above. Jackson follows from the original analysis and
synthesis factorization, and Bernstein from the same finite bank.
The mean-zero constraint is imposed only after this full-space argument;
the corrected projector need not have a local kernel.
\end{proof}
\paragraph*{The mean-zero coefficient bank.}
Normalize one fixed coarse primal atom to $\varphi_0$ with physical
integral one; one fixed coarse enlargement provides such an atom.
For every other full-bank index $i$, set
$\mu_i=\int\psi_i\,d\nu$ and
$\varphi_i=\psi_i-\mu_i\varphi_0$.
These are the density prototypes selected in
Definition~\ref{def:repu-construction}; their free coefficients are
$c_i=\langle g,\widetilde\psi_i\rangle$ for mean-zero $g$.
There is no trainable mass coefficient and no separate box on a discarded
coarse pivot.  Biorthogonality gives the exact identities
\[
 g=\sum_{i\ne0}c_i\varphi_i,\qquad
 \Pi_J^0g=\sum_{\substack{i\ne0\\j_i\le J}}c_i\varphi_i
 =Q\Pi_Jg,\qquad Qg=g-\varphi_0\int g\,d\nu .
\]
Thus the independent boxes apply to the same free coordinates used in
the proof.
The swapped primal--dual Besov characterization applied to the smooth
physical function $1$ gives, with $t=\eta+2$,
$|\mu_{jk}|\le C2^{-j(t+d/2)}$.
The corrected atom is its local spline plus the explicit smooth coarse
tail $-\mu_{jk}\varphi_0$; higher exact moments and strict support are not
asserted for this tail.  It has a fixed-width RePU realization by adding
one fixed coarse branch.  Sobolev basis changes are uniformly bounded:
their only extra row is $(\mu_i)$, and
$\sum_i2^{2j_i}|\mu_i|^2<\infty$.
This also controls both order-two conjugations after density information
normalization.  The coefficient tail's coarse correction is
$O(h^{\eta+t})$ for a $C^\eta$ truth.
The augmented full cutoff is
$\Pi_J+\varphi_0\langle(I-\Pi_J)(\,\cdot\,),1\rangle$;
dual smooth Jackson bounds its extra kernel and derivatives by
$Ch^{t-|b|}$.  It therefore obeys the total-order-two bound in~\eqref{eq:r25-kernel-convenient} as well.
The constrained cutoff has the same Jackson and operator bounds, but
need not have an exactly local kernel.  Below, its density block is
denoted by $\Pi_J$; the full cutoff is used first in every kernel argument.
Neither cutoff is the orthogonal Galerkin projector or an information
inverse.
For $\alpha\in\{0,\beta\}$, write
\begin{equation}\label{eq:vs-Calpha-norms}
 \|(r,v)\|_{\mathcal C_\alpha}
 :=\|r\|_{C^\alpha}+\|v\|_{C^{\alpha+1}},
 \qquad \mathcal C_0=: \mathcal C.
\end{equation}

Let $\mathcal V_{J,m}$ contain the mean-zero density and normal-section
wavelets through level $J$.  It has dimension $Q_J$.  Fix uniformly stable
dual frames and suppress $m$ within one admissible branch.

Proposition~\ref{prop:canonical-law-sieve} constructs finitely many law sieves
with bounded coefficient domains and uniformly Lipschitz charts.  If
\(\theta_m(h)\in\mathbb R^{Q_J}\) is the density--normal coefficient vector in
branch \(m\), define
\(d_{\rm coef}(h,h'):=\|\theta_m(h)-\theta_m(h')\|_2\) for
\(h,h'\in\mathcal H_{J,m}\); the frame-analysis constants make it uniform over the fixed
branches.  Write $N(u,\mathcal A,\mathfrak d)$ for the $u$-covering number of a
metric space $(\mathcal A,\mathfrak d)$.  Then
\begin{equation}\label{eq:vs-sieve-entropy}
\max_{1\le m\le M_0}\log N(u,\mathcal H_{J,m},d_{\rm coef})
\le CQ_J\log(C/u),\qquad 0<u\le1.
\end{equation}
The fixed slack leaves room for the comparators; Lemma~\ref{lem:pm-newton-margins} verifies the separate constraints for the retained root construction.
For an embedding $h$, let $J_h(u)=\{
\det(Dh(u)^\top Dh(u))\}^{1/2}$.  The area formula~\citep{federer1969gmt} gives
\begin{equation}\label{eq:vs-area-density}
P_h=f_h\,\nu_h,\qquad
f_h(h(u))=\frac{\pi(u)}{J_h(u)},
\qquad \nu_h=\mathcal H^d|_{h(\mathcal U)}.
\end{equation}
Assumptions~\ref{G1}--\ref{G2} therefore imply uniformly
\begin{equation}\label{eq:vs-density-bounds}
0<m_f\le f^*\le M_f<\infty,\qquad
\|f^*\|_{C^\eta}\le L_f.
\end{equation}
Candidate densities are uniformly bounded and positive in $C^\beta$; no
uniform candidate $C^\eta$ bound is used.

Fix a true law $P^*=f^*\nu^*$ on
$\mathcal M^*=h^*(\mathcal U)$.  If $v$ is a normal vector field on
$\mathcal M^*$, set $F_v(x)=x+v(x)$.  If $r$ has zero $\nu^*$ integral and
$f^*+r>0$, define
\begin{equation}\label{eq:vs-law-chart}
P_{r,v}:=(F_v)_\#\{(f^*+r)\nu^*\}.
\end{equation}
For a nontrivial normal bundle, use stable local frames and a partition of
unity.

{We differentiate each affine reference chart directly.  If
$P_{r,v}=(F_v)_\#\{(f_0+r)\nu_0\}$, the truth in that chart is
$(r_*,v_*)$, $F_*=F_{v_*}$, and $w=(a,b)$, then for every smooth test
function $\varphi$
\begin{equation}\label{eq:vs-direct-reference-derivative}
 \dot P_w\varphi
 =\int_{\mathcal M_0}\!\left[
 a(x)\varphi\{F_*(x)\}+
 (f_0+r_*)(x)\nabla\varphi\{F_*(x)\}^{\!\top}b(x)\right]d\nu_0(x).
\end{equation}
On $\mathcal M^*$, decompose $b\circ F_*^{-1}=b_T+b_N$.  If $J_*$ is the
surface Jacobian and
$f_*=\{(f_0+r_*)/J_*\}\circ F_*^{-1}$, integration by parts identifies this
direction with
\begin{equation}\label{eq:vs-direct-triangular-map}
 D_*(a,b)=\left\{
 (a/J_*)\circ F_*^{-1}-\operatorname{div}_{\mathcal M^*}(f_*b_T),\
 b_N\right\}.
\end{equation}
The graph radius makes $b\mapsto b_N$ invertible.  Multiplication,
composition, and the single divergence show that $D_*$ and its inverse are
bounded and almost diagonal in the $H^{-1}\oplus L^2$ and matched dyadic
norms.  All Gaussian numerator derivatives below follow from
\eqref{eq:vs-direct-reference-derivative}; no nonlinear chart identity is
needed for them.}

Write $\mathcal M_Q=\operatorname{supp}Q$,
$\nu_Q=\mathcal H^d|_{\mathcal M_Q}$, and let $\iota_Q$ and $\iota_*$ be the
ambient inclusion maps of $\mathcal M_Q$ and $\mathcal M^*$.  A candidate law
$Q=f_Q\nu_Q$ is in the $C^1\times C^0$ law-neighborhood of
$P^*=f^*\nu^*$ of radius $\varepsilon$ if there is a $C^1$ diffeomorphism
$\varphi:\mathcal M^*\to\mathcal M_Q$ such that
\(\|\iota_Q\circ\varphi-\iota_*\|_{C^1}
+\|d\{\varphi^*(f_Q\nu_Q)\}/d\nu^*-f^*\|_{C^0}<\varepsilon\).
Here the $C^1$ norm is computed in the fixed finite atlas and stable normal
frame; changing that atlas gives a uniformly equivalent neighborhood.  This
definition is at the law level and therefore quotients source
reparameterizations.

\begin{lemma}[Normal--density chart and removal of the gauge]
\label{lem:normal-density-chart}
There are $\varepsilon_0,c,C>0$, uniform over the class, with the following
properties.

If a regular candidate law is in the $C^1\times C^0$ law-neighborhood of $P^*$ of
radius $\varepsilon_0$, then it has a unique representation~\eqref{eq:vs-law-chart}
with $\|v\|_{C^1}+\|r\|_{C^0}\le C\varepsilon_0$.  The pair $(r,v)$ is unique
at the law level.  At the differential level, two map directions give the same
law direction exactly when their difference is a $P^*$-preserving tangential
reparameterization.
Moreover,
\begin{equation}\label{eq:vs-chart-w2}
W_2(P_{r,v},P_{r',v'})
\le C\left\{
\|r-r'\|_{H^{-1}(\mathcal M^*)}
+\|v-v'\|_{L^2(\nu^*)}
\right\}.
\end{equation}
In wavelet coefficients through level $J$, the squared norm on the right is
equivalent to $\|\theta\|_{\bm W_J}^2$ defined in
\eqref{eq:vs-W-matrix}.
For unrestricted law coordinates define the full norm using the physical
$L^2$ norms of the reconstructed level blocks:
\begin{equation}\label{eq:vs-full-W-norm}
\|(r,v)\|_{\mathcal W}^2
:=\sum_{j=-1}^{\infty}2^{-2j}\|\mathsf P_jr\|_{L^2}^2
+\sum_{j=-1}^{\infty}\|\mathsf P_jv\|_{L^2}^2.
\end{equation}
On $\mathcal V_J$, the Riesz bounds give
$\|\cdot\|_{\mathcal W}\asymp\|\cdot\|_{\bm W_J}$ uniformly.
\end{lemma}

\begin{proof}
Positive reach gives a unique nearest-point projection on the
$\rho_{\mathcal M}$-tube.  If $\varphi$ is the diffeomorphism in the
$C^1\times C^0$ neighborhood definition, then
$\Pi_{\mathcal M^*}\circ\varphi$ is uniformly $C^1$-close to the identity.
The diffeomorphism group of a compact manifold is open in the $C^1$
topology, so this map is a global diffeomorphism.  Thus
$\Pi_{\mathcal M^*}|_{\mathcal M_Q}$ is invertible.
Its inverse is $F_v(x)=x+v(x)$ with
$v(x)\perp T_x\mathcal M^*$; pulling back the density then uniquely determines
$r$.

For a tangential-normal map perturbation $a+v$, differentiation of the area
formula gives
\begin{equation}\label{eq:vs-map-law-derivative}
(a,v)\longmapsto
\left\{-\operatorname{div}_{\mathcal M^*}(f^*a),\ v\right\}.
\end{equation}
Its kernel is $\{(a,0):\operatorname{div}(f^*a)=0\}$, the infinitesimal
$P^*$-preserving reparameterizations.  Every mean-zero density direction $r$
is realized by the unique solution of
\begin{equation}\label{eq:vs-gauge-elliptic}
-\operatorname{div}_{\mathcal M^*}(f^*\nabla\phi)=r,
\qquad \int\phi\,d\nu^*=0,
\end{equation}
and $a=\nabla\phi$.  Thus~\eqref{eq:vs-map-law-derivative} is an isomorphism
modulo this gauge.

For~\eqref{eq:vs-chart-w2}, set $g_0=f^*+r$, $g_1=f^*+r'$ and
$g_a=(1-a)g_0+ag_1\ge m_f/2$.  Lax--Milgram gives
\begin{equation}\label{eq:vs-density-transport-pde}
 -\operatorname{div}(g_a\nabla\phi_a)=g_1-g_0,
 \qquad \int\phi_a\,d\nu^*=0.
\end{equation}
Uniform ellipticity gives
$\int g_a|\nabla\phi_a|^2d\nu^*\le C\|g_1-g_0\|_{H^{-1}}^2$.
The continuity equation and Benamou--Brenier formula~\citep{villani2003topics}
yield
\(W_2\{(f^*+r)\nu^*,(f^*+r')\nu^*\}
\le C\|r-r'\|_{H^{-1}}\).
Push this coupling through $F_v\times F_v$ and couple $F_v(x)$ with
$F_{v'}(x)$ under $(f^*+r')\nu^*$.  Uniform Lipschitz and density bounds give
the two terms in~\eqref{eq:vs-chart-w2}.  The wavelet characterization of
$H^{-1}$~\citep{jouini2007wavelet,meyer1992wavelets} and normal-frame
stability give~\eqref{eq:vs-W-matrix}.
\end{proof}

\begin{lemma}[Transition between law charts]
\label{lem:vs-chart-transition}
Let two of the preceding normal--density charts overlap, write
$\Gamma_{ba}$ for the coordinate change from chart $a$ to chart $b$, and put
$X_s=C^s\times C^{s+1}$. On a fixed smaller overlap, for every fixed
$0<\delta<\beta-2$, $\Gamma_{ba}$ and its inverse are continuously
differentiable from $X_\beta$ to $X_{\beta-1-\delta}$ and twice continuously
differentiable to $X_{\beta-2-\delta}$. The strict loss concerns continuity
of derivatives in the H\"older--Zygmund topology, not the fixed-base bounds
below. The first differentials also extend as bounded operators to the weak chart norms used
below and satisfy the displayed tame bounds; no second derivative is asserted
as a bounded bilinear map on $H^{-1}\times L^2$ or on $X_\beta$ itself.
At every coordinate $\zeta$ in this overlap their differentials have the
triangular form
\begin{equation}\label{eq:vs-chart-transition-linearization}
 D\Gamma_{ba}(\zeta)(\dot r,\dot v)
 =\bigl(A_\zeta\dot r+\operatorname{div}(B_\zeta\dot v),
 C_\zeta\dot v\bigr),
\end{equation}
where $A_\zeta,B_\zeta,C_\zeta$ and the corresponding inverse coefficients
are uniformly bounded in the regularities allowed by Assumption~\ref{G2}.
Consequently,
\begin{align}
 \|D\Gamma_{ba}(\zeta)w\|_{H^{-1}\times L^2}
 &\le C\|w\|_{H^{-1}\times L^2},
 \label{eq:vs-chart-transition-W}\\
 \|D\Gamma_{ba}(\zeta)w\|_{C^0\times C^1}
 &\le C\|w\|_{C^0\times C^1},
 \label{eq:vs-chart-transition-C0}\\
 \|(D\Gamma_{ba}(\zeta)w)^{\rm den}\|_{C^{\beta-1}}
 +\|(D\Gamma_{ba}(\zeta)w)^{\rm nor}\|_{C^{\beta}}
 &\le C\|w\|_{C^\beta\times C^{\beta+1}}.
 \label{eq:vs-chart-transition-C}
\end{align}
The same inequalities hold for the inverse.  Their wavelet blocks satisfy
\begin{align}
 \|\mathsf P_jD\Gamma^{aa}_{ba}(\zeta)\mathsf P_\ell\|_{2\to2}
 &\le C2^{-\nu_{\rm diag}|j-\ell|},
 \qquad a\in\{\mathrm{den},\mathrm{nor}\},
 \label{eq:vs-chart-transition-ad-diag}\\
 \|\mathsf P_jD\Gamma^{\rm den,nor}_{ba}(\zeta)
       \mathsf P_\ell\|_{2\to2}
 &\le C2^\ell2^{-\nu_{\rm mix}|j-\ell|},
 \qquad D\Gamma^{\rm nor,den}_{ba}(\zeta)=0.
 \label{eq:vs-chart-transition-ad-mix}
\end{align}
If $w\in\mathcal V_J$, then
\begin{equation}\label{eq:vs-chart-transition-E}
 \|(D\Gamma_{ba}(\zeta)w)^{\rm den}\|_{H^1}
 +\tau_J^{-1}\|(D\Gamma_{ba}(\zeta)w)^{\rm nor}\|_2
 \le C\{\|w^{\rm den}\|_{H^1}+\tau_J^{-1}\|w^{\rm nor}\|_2\}.
\end{equation}
For the omitted $C^\eta\times C^{\eta+1}$ tail
$\rho_J=-(I-\Pi_J)\zeta$, the left side of
\eqref{eq:vs-chart-transition-E}, with
$D\Gamma_{ba}(\zeta)\rho_J$ in place of $D\Gamma_{ba}(\zeta)w$, is
$O(2^{-J(\eta-1)})$.  Finally, writing
$R_{ba}(\zeta,w)=\Gamma_{ba}(\zeta+w)-\Gamma_{ba}(\zeta)
-D\Gamma_{ba}(\zeta)w$, the weak and low strong remainders obey the mixed
derivative-loss estimates
\begin{align}
 \|R_{ba}(\zeta,w)^{\rm den}\|_{H^{-1}}
 +\|R_{ba}(\zeta,w)^{\rm nor}\|_{L^2}
 &\le C\|w\|_{C^0\times C^1}
 \{\|w^{\rm den}\|_{L^2}+\|w^{\rm nor}\|_{H^1}\},
 \label{eq:vs-chart-transition-tame-W}\\
 \|R_{ba}(\zeta,w)\|_{C^0\times C^1}
 &\le C\|w\|_{C^1\times C^2}\|w\|_{C^0\times C^1}.
 \label{eq:vs-chart-transition-tame-C}
\end{align}
On the strong scale, Taylor's remainder instead obeys the two-loss estimate
\[
 \|\Gamma_{ba}(\zeta+w)-\Gamma_{ba}(\zeta)
      -D\Gamma_{ba}(\zeta)w\|_{X_{\beta-2}}
 \le C\|w\|_{X_\beta}^2.
\]
\end{lemma}

\begin{proof}
{Let $\pi_a,\pi_b$ be the nearest-point projections onto the two
reference supports.  If a support is the graph $F_a$ over the first
reference, then $\phi_{ba}:=\pi_b\circ F_a$ is a uniformly regular
diffeomorphism on the smaller overlap.  Its graph over the second reference
is $F_b=F_a\circ\phi_{ba}^{-1}$, and its normal coordinate is
$v_b=F_b-\operatorname{id}$; no zero difference
$F_a-F_b\circ\phi_{ba}$ is used.  Differentiating $v_b$ gives
$\dot v_b=C_\zeta\dot v$.  Pulling the law back by $\phi_{ba}$ and
differentiating its density and surface Jacobian gives
$\dot r_b=A_\zeta\dot r+\operatorname{div}(B_\zeta\dot v)$, proving}
\eqref{eq:vs-chart-transition-linearization}.  Multiplication by a uniformly
regular coefficient is bounded on the displayed spaces, and
$\|\operatorname{div}z\|_{H^{-1}}\le\|z\|_2$; this proves
\eqref{eq:vs-chart-transition-W}--\eqref{eq:vs-chart-transition-C} and the
inverse bounds.  The divergence is the reason that the strong density norm in
\eqref{eq:vs-chart-transition-C} is $C^{\beta-1}$ rather than $C^\beta$.

Wavelet vanishing moments applied to multiplication by $A_\zeta$ and
$C_\zeta$ give every diagonal decay order below $\beta$.  For
$\operatorname{div}(B_\zeta\dot v)$, extract the input derivative $2^\ell$;
the remaining coefficient has $C^{\beta-1}$ regularity and gives every order
below $\beta-1$.  This proves
\eqref{eq:vs-chart-transition-ad-diag}--
\eqref{eq:vs-chart-transition-ad-mix}; the inverse has the same triangular
form and estimates.

For $w\in\mathcal V_J$, Bernstein's inequality and
$\tau_J^{-1}\asymp2^{2J}$ absorb the single derivative in the divergence,
which proves~\eqref{eq:vs-chart-transition-E}.  For $\rho_J$, the Jackson
bounds~\eqref{eq:vs-wavelet-properties} give
$\|\rho_J^{\rm den}\|_{H^1}+2^{2J}\|\rho_J^{\rm nor}\|_2
=O(2^{-J(\eta-1)})$; the preceding formula gives the same bound after the
transition.  Taylor's formula in the fixed finite atlas, the product rule,
and the uniform coefficient bounds prove
\eqref{eq:vs-chart-transition-tame-W}--\eqref{eq:vs-chart-transition-tame-C}
and the displayed two-loss strong estimate.  In particular, the density
pullback term has the form $\operatorname{div}(w^{\rm nor}w^{\rm den})$
plus lower-order products, while the normal pullback uses one $H^1$
derivative; this gives exactly the weak norms in
\eqref{eq:vs-chart-transition-tame-W}.  For the omitted tail, Jackson's
inequality also gives
$\|\rho_J^{\rm den}\|_2+\|\rho_J^{\rm nor}\|_{H^1}
\le C2^{-J\eta}$, so its mixed weak remainder is
$O(2^{-2J\eta})=o(2^{-J(\eta+1)})$.  Each differentiation of the
composition and Jacobian formulas costs the corresponding source derivative.
Their endpoint product bounds give the displayed quadratic remainder
directly, by smooth approximation. Interpolation with any strict
$\delta$ loss gives continuity of the first and second differentials;
endpoint Fr\'echet continuity in the big H\"older--Zygmund spaces is not used.
\end{proof}

\subsubsection{Reference multiscale approximation}

We construct a finite collection of truth-independent reference families.
They supply law charts and a Jackson-rate comparator for the RePU-coordinate proof.
The unwarped reference profile realizes the same approximation spaces, while
the jointly trained hidden feature map is introduced later.
The pilot
construction then selects and localizes one of the finitely many chart
branches.

\begin{lemma}[Quantitative prescribed-Jacobian reconstruction]
\label{lem:quantitative-prescribed-jacobian}
Let $\mu$ be a fixed smooth probability volume on the compact manifold
$\mathcal U$, and let $s>1$.  On every set
\[
 \mathcal F_{s}(m,B)=\{f\in C^s(\mathcal U):
 \textstyle\int f\,d\mu=1, m\le f\le m^{-1},\ \|f\|_{C^s}\le B\}
\]
there is a choice of $C^{s+1}$ diffeomorphisms $S_f$ such that
$(S_f)_\#\mu=f\mu$, with
\begin{equation}\label{eq:quantitative-prescribed-jacobian}
 \sup_{f\in\mathcal F_s(m,B)}
 (\|S_f\|_{C^{s+1}}+\|S_f^{-1}\|_{C^{s+1}})\le C_{m,B}.
\end{equation}
Moreover, if $f_0$ is positive and $C^{s+1}$ and
$\int f\,d\mu=\int f_0\,d\mu=1$, then, for
$\|f-f_0\|_{C^s}$ sufficiently small, one may choose a diffeomorphism
$R_f$ satisfying
\begin{equation}\label{eq:quantitative-prescribed-jacobian-local}
 (R_f)_\#(f_0\mu)=f\mu,\qquad R_{f_0}=I,\qquad
 \|R_f-I\|_{C^{s+1}}\le C_{f_0}\|f-f_0\|_{C^s}.
\end{equation}
The same statement holds at integer $s$ in the H\"older--Zygmund scale.
\end{lemma}

\begin{proof}
We give the endpoint argument, writing $C_*^r=B_{\infty,\infty}^r$
(and $C_*^r=C^r$ at noninteger orders).  Choose
$0<\alpha<\min(1,s-1)$ and a smooth elliptic right inverse
$L=\nabla\Delta_\mu^{-1}$ of divergence on mean-zero functions.
Its local order-$-1$ kernel, with a smooth remainder on the compact
manifold, gives $\|Lb\|_{C_*^{r+1}}\le C_r\|b\|_{C_*^r}$.
For the pointwise exponential map $\Phi_V(x)=\exp_xV(x)$,
\[
 J_\mu\Phi_V=1+\operatorname{div}_\mu V+N(V).
\]
The remainder is a smooth function of $(x,V,DV)$ vanishing quadratically.
The dyadic product estimate therefore gives
\[
 \begin{aligned}
 \|N(V)\|_{C_*^s}
 &\le C_s\|V\|_{C^{1+\alpha}}\|V\|_{C_*^{s+1}},\\
 \|N(V)-N(W)\|_{C^\alpha}
 &\le C_\alpha(\|V\|_{C^{1+\alpha}}+\|W\|_{C^{1+\alpha}})
                  \|V-W\|_{C^{1+\alpha}}.
 \end{aligned}
\]
For small $\|V\|_{C^{1+\alpha}}$, $\Phi_V$ is a diffeomorphism, so
$\int N(V)\,d\mu=0$.  If $\|q-1\|_{C^\alpha}$ is sufficiently small,
the iteration
\[
 V_0=0,\qquad V_{k+1}=L\{q-1-N(V_k)\}
\]
contracts in a fixed small $C^{1+\alpha}$ ball.  Choose that ball small
enough also to absorb the high-norm term.  Then
\[
 \|V_{k+1}\|_{C_*^{s+1}}
 \le C_s\|q-1\|_{C_*^s}+\tfrac12\|V_k\|_{C_*^{s+1}}.
\]
Dyadic Fatou applied to the low-norm limit yields a $C_*^{s+1}$ field
$V$ with $\|V\|_{C_*^{s+1}}\le2C_s\|q-1\|_{C_*^s}$ and
$J_\mu\Phi_V=q$.  Thus $Q=(\Phi_V)^{-1}$ satisfies $Q_\#\mu=q\mu$.
Tame composition and the uniformly positive singular values give the same
regularity for $Q$; when $\|q-1\|_{C_*^s}$ is small, they give
$\|Q-I\|_{C_*^{s+1}}\le C_s\|q-1\|_{C_*^s}$.
These estimates follow at integer orders by smooth approximation and
Fatou on the dyadic blocks.

For a general $f\in\mathcal F_s(m,B)$, take the positive heat smoothing
$g=e^{\epsilon\Delta_\mu}f$.  It preserves mass and the bounds
$m\le g\le m^{-1}$.  The smooth Moser field
\[
 g_t=1+t(g-1),\qquad v_t=L(1-g)/g_t
\]
has uniformly bounded $C^{1+\alpha}$ norm, independently of $\epsilon$;
its flow $T_g$ satisfies $(T_g)_\#\mu=g\mu$ and has uniformly bounded
Lipschitz norm.  Consequently
\[
 q=(f/g)\circ T_g,\qquad
 q\mu=(T_g^{-1})_\#(f\mu),\qquad
 \|q-1\|_{C^\alpha}\le C_{m,B}\|f-g\|_{C^\alpha}.
\]
One fixed $\epsilon=\epsilon(m,B,s)>0$ makes the last norm smaller
than the contraction threshold, uniformly in $f$.  At this fixed scale
$T_g,T_g^{-1}$ have bounded norms at every required finite order, and
$\|q\|_{C_*^s}\le C_{m,B,s}$.  Hence $S_f=T_g\circ Q$ proves
\eqref{eq:quantitative-prescribed-jacobian}, including integer $s$.
This also gives a direct compact-manifold construction corresponding to
\citet{dacorogna1990jacobian}.

For the local assertion, apply the established construction at
regularity $s+1$ to the fixed $f_0$ and choose
$T_0$ with $(T_0)_\#\mu=f_0\mu$ and
$T_0,T_0^{-1}\in C^{s+2}$.  Define $q_f$ by
\[
 q_f\mu=(T_0^{-1})_\#(f\mu).
\]
Then $q_{f_0}=1$ and fixed-map composition gives
$\|q_f-1\|_{C^s}\le C_{f_0}\|f-f_0\|_{C^s}$.  The small-data step of the
prescribed-Jacobian construction at the constant density supplies $Q_f$ with
$(Q_f)_\#\mu=q_f\mu$, $Q_{f_0}=I$, and
\[
 \|Q_f-I\|_{C^{s+1}}\le C\|q_f-1\|_{C^s}.
\]
Set $R_f=T_0\circ Q_f\circ T_0^{-1}$.  Its pushforward identity is the first
part of~\eqref{eq:quantitative-prescribed-jacobian-local}; the fixed
$C^{s+2}$ conjugating maps and the preceding estimate prove its norm bound.
Thus all inversion and composition differences are taken against the fixed
identity, not between two rough variable-density solutions.
\end{proof}

\begin{proposition}[Canonical law sieve]
\label{prop:canonical-law-sieve}
Under Assumptions~\ref{G1}--\ref{G2}, there is a deterministic sequence
of finite families
$\{\mathcal H_{J,m}:1\le m\le M_0\}_{J\ge J_0}$ and fixed reference laws
$\bar P_1,\ldots,\bar P_{M_0}$ with $C^\infty$ supports and positive
$C^\eta$ densities, all independent of the unknown truth, with
each $\mathcal H_{J,m}$ contained in the theoretical candidate envelope and
with the following properties.

(i)  The induced-law set $\{h_\#\pi:h\in\mathcal H_{J,m}\}$ is parametrized
by a compact convex coefficient domain, denoted by $\Theta_{J,m}$, whose
origin is the reference law.  In each reference chart
the free coordinates are exactly all mean-zero density coefficients and the
{normal-bundle multiwavelet coefficients} through
level $J$, subject to fixed $C^\beta\times C^{\beta+1}$ bounds.  Its dimension is
$Q_J$, and \(\log N(u,\mathcal H_{J,m},d_{\rm coef})
\le C Q_J\log(C/u)\) for \(0<u\le1\).

(ii) There is a fixed $a_0>0$ such that
$\min_mW_2(P^*,\bar P_m)\le a_0$ for every truth.  For every index satisfying
$W_2(P^*,\bar P_m)\le3a_0$, the truth lies with a uniform margin inside the
outer reference chart of branch $m$, and there is a comparator
$P_{J,m}^0\in\{h_\#\pi:h\in\mathcal H_{J,m}\}$.  If
$\zeta^*=(r^*,v^*)$ is the truth coordinate in this reference chart and
$\rho_J:=\Pi_J\zeta^*-\zeta^*=-(I-\Pi_J)\zeta^*$, then, with $s_\star=\eta+1$ and
$\tau_J=c_t2^{-2J}$,
\begin{align}
 \|\rho_J^{\rm den}\|_{H^{-1}}+\|\rho_J^{\rm nor}\|_2
 &\le C2^{-Js_\star},\label{eq:canonical-sieve-W}\\
 \|\rho_J^{\rm den}\|_\infty+\|\rho_J^{\rm nor}\|_{C^1}
 &\le C2^{-J\eta},\label{eq:canonical-sieve-C}\\
 \|\rho_J^{\rm den}\|_{H^1}
 +\tau_J^{-1}\|\rho_J^{\rm nor}\|_2
 &\le C2^{-J(s_\star-2)}.\label{eq:canonical-sieve-E}
\end{align}
The comparator and a fixed $C^\beta\times C^{\beta+1}$ coefficient
neighborhood around it remain in the strict regularity interior uniformly in
$J$ and in the truth.  Relative to the truth, every member of this
neighborhood has the exact reference-coefficient representation
$\rho_J+u$, $u\in\mathcal V_J$.  Lemma~\ref{lem:vs-chart-transition} gives
the transition to the truth-centered chart.  The later risk calculation uses
its mixed derivative-loss bounds and direct affine-reference differentiation;
it does not
invoke a no-loss second derivative of the nonlinear chart transition.

(iii) Every law in every branch is induced by a $C^{s_\star}$ map from the fixed source
with base law $\pi$; no truth-dependent base distribution is introduced.
\end{proposition}

\begin{proof}
Let $\mathfrak P_s$ be the laws induced by the fixed truth class
$\mathcal H_*$.  Fix a smooth probability volume $\mu$ and
$A_\pi=S_{d\pi/d\mu}^{-1}$ from
Lemma~\ref{lem:quantitative-prescribed-jacobian}, so $(A_\pi)_\#\pi=\mu$.
Smoothly approximate $h^*\circ A_\pi^{-1}$ by embeddings $k_m$ and set
$\bar h_m=k_m\circ A_\pi$.  Their induced laws $(k_m)_\#\mu$ have smooth
supports and smooth positive densities, even when $\pi$ is not smooth.
The approximations converge in $C^{\beta+1}$ and retain a common
$C^{\eta+1}$ bound and positive singular-value and reach bounds.
Choose the references using only low-order geometry.
Compactness in $C^2$ implies the following uniform implication: sufficiently
small $W_2$ distance gives normal--density coordinates satisfying
$\|r^*\|_\infty\le\lambda/8$ and
$\|v^*\|_{C^1}\le\varepsilon/8$.
Here $\lambda>0$ is a common reference-density lower bound and
$\varepsilon>0$ is a common graph-tube radius.  Indeed, a contradicting
sequence has limiting embeddings with the same positive induced law.
Their supports coincide, and their relative normal graphs and pulled-back
densities converge to zero in $C^1$ and $C^0$, respectively.
Choose $a_0$ for this implication at distance $4a_0$, then a finite
$a_0$-net of smooth reference laws
$\bar P_m=(\bar h_m)_\#\pi$.  The maps $\bar h_m$ retain the original
half-margin.  Fix their atlas and spline banks now; this finite family
will not be refined after the coefficient amplitudes are chosen.

In reference chart $m$, write
$P_{r,v}^{(m)}=(I+v)_\#\{(\bar f_m+r)\bar\nu_m\}$.
All admissible truth coordinates have a uniform
$C^\eta\times C^{\eta+1}$ bound on this fixed finite atlas.
For coefficients write
\[
 \|\thetao\|_{b_\beta}:=
 \max_{j,k}\left\{
 2^{j(\beta+d/2)}|\theta_{jk}^{\rm den}|,
 2^{j(\beta+1+d/2)}|\theta_{jk}^{\rm nor}|\right\}.
\]
The inverse geometric transports in~\eqref{eq:mlsn-geometric-pullbacks}
are uniformly bounded in the same regularity classes on this fixed atlas;
this follows directly from their $O(\tau_J)$ flow and fixed smooth
geometry.  Thus the adapted truth coefficients have a common finite
$b_\beta$ bound $K_*$.  Choose $A>8K_*$.  This ordering uses no uniform
high-order frame constant over a sequence of refining atlases.

Choose a fixed coarse cutoff $L$ so large that every allowed coefficient
tail satisfies
\[
 \left\|\mathsf U_{\thetah,J}
       \sum_{j>L,k,s}\theta_{jk}^s\psi_{jk}^s\right\|_{C^0\times C^1}
 \le CA2^{-\beta L}\le\min(\lambda,\varepsilon)/16.
\]
Use the norm that sums the two components.  Here and below the identity
transport gives the canonical unwarped bank.  Form the coarse field by
transporting the input sum $j\le L$, and impose
\[
 \bar f_m+r_{\thetah,\le L}\ge\lambda/2,\qquad
 \|v_{\thetah,\le L}\|_{C^1}\le\varepsilon/2.
\]
These are closed convex constraints in each output fiber.  Full candidates
then have density at least $7\lambda/16$, normal $C^1$ norm at most
$9\varepsilon/16$, and a fixed $C^\beta\times C^{\beta+1}$ bound.
They realize the fixed coarse constraints in
Definition~\ref{def:repu-construction}.  Finite-dimensional coefficient
covering gives the entropy bound in (i).

We next reconstruct the laws from the fixed source, without requiring
Lipschitz dependence of reconstructed maps on their coefficients.
Use the already fixed $\mu,A_\pi$ and smooth
$k_m=\bar h_m\circ A_\pi^{-1}$, and write
$k_m^*\bar\nu_m=j_m\mu$.  For each candidate put
\[
 f_r=1+j_m(r\circ k_m),\qquad
 h_{m,r,v}=(I+v)\circ k_m\circ S_{f_r}\circ A_\pi.
\]
Choose $S_{f_r}$ using the global clause at $s=\beta$.
The preceding positivity and regularity bounds imply uniform
$C^{\beta+1}$ bounds on this map and its inverse on its image, and
$(h_{m,r,v})_\#\pi=P_{r,v}^{(m)}$ exactly.
Small $C^1$ normal graphs with a common $C^2$ bound have uniformly positive
reach: local chord-to-tangent deviations are uniformly quadratic, and
distant points remain separated by projection onto the reference support.
We therefore choose the candidate-envelope upper constants larger, and
positive lower constants smaller, than these derived bounds and the original
truth-envelope constants.  All generated maps lie strictly inside
$\mathcal H_{\rm reg}$, while $\mathcal H_*$ is unchanged.
For each finite spline law, applying the reconstruction separately at
$s=\eta$ gives a $C^{\eta+1}$ representative of that same law, proving (iii).
No compatibility between these two selections is used in score derivatives,
which are taken in the law integral.

Finally, project the truth coordinates through level $J$.  Their
$C^\eta\times C^{\eta+1}$ regularity gives the three Jackson bounds
\eqref{eq:canonical-sieve-W}--\eqref{eq:canonical-sieve-E}.
The adapted comparators use less than a quarter of the coefficient box.
For sufficiently large fixed $J_0\ge L$, their low-order Jackson error and
the coarse-tail bound place them strictly inside the positivity and tube
constraints, with margins at least $\lambda/4$ and $\varepsilon/4$.
A common coefficient neighborhood preserves these margins and the larger
candidate envelope.  The error is exactly
$\rho_J+u$, $u\in\mathcal V_J$, with the stated chart transition.
This proves (ii).  The atlas, $A,L,J_0$, and candidate-envelope constants
are fixed independently of $n,J$.
\end{proof}

Set
\(\mathcal I_*(P^*):=
\{1\le m\le M_0:W_2(P^*,\bar P_m)\le3a_0\}\).

\paragraph*{Branchwise convention for the core analysis.}
Call $m$ admissible if $m\in\mathcal I_*(P^*)$.  Work branchwise without
conditioning on the selector, intersect empirical events over $1\le m\le M_0$,
and write $\mathcal H_J=\mathcal H_{J,m}$ and $P_J^0=P_{J,m}^0$.
Appendix~\ref{sec:pilot-verification} verifies this for the optional selector.

In fixed reference coefficients, $e=\rho_J+u$, $u\in\mathcal V_J$, is affine.
Lemma~\ref{lem:vs-chart-transition} expresses the differential to the
truth-centered chart as a triangular density--normal operator. Its full-field
weak bounds and retained-input energy bounds apply before projection.
No inverse of a compressed chart differential is used. We differentiate the
Gaussian law integral directly in the affine reference coordinates below.

\subsubsection{RePU features and the joint parameter domain}
\label{sec:repu-class}
The localized parameter domains use the notation of Appendix~\ref{sec:mlsn-proof}.

The following construction is fixed before the data are observed.
We specify the RePU features, joint domain, and approximation properties;
the pilot and score-root arguments are in Appendix~\ref{sec:pilot-verification}.
Fix an admissible branch $m$ and write
$\mathfrak h_m(r,v)=h_{m,r,v}$ for the canonical representative of
Proposition~\ref{prop:canonical-law-sieve}.  Let
$\{\psi_{jk}^{\rm den},\psi_{jk}^{\rm nor}\}$ be the order-$M$ spline bank,
including its fixed-coarse density correction, from Appendix~\ref{sec:multiscale-tools}, with
$q_j^s$ elements of type $s\in\{{\rm den},{\rm nor}\}$ at level $j$ and
$q_j^s\le C2^{jd}$.  The density elements have zero integral, and the normal
elements are written in the fixed stable normal frame.

{In local coordinates, each order-$M=q+1$ cardinal B-spline satisfies
$B_M(x)=q!^{-1}\sum_{\ell=0}^{q+1}(-1)^\ell
\binom{q+1}{\ell}\rho_q(x-\ell)$, up to fixed translation
and dilation, and is therefore a one-hidden-layer $\rho_q$ network.
Exact RePU multiplication
blocks \citep{li2020powernet} form the tensor-product multiwavelets with fixed
depth and width.  This claim is chartwise; atlas assembly, pullbacks, and
$\mathfrak h_m$ are fixed geometric operations.}

\paragraph*{Trainable hidden feature map.}
Write the full parameter in Definition~\ref{def:repu-construction} as
$\theta=(\thetah,\thetao)$.  Both components are optimized jointly, while
fixing the hidden component $\thetah$ leaves an optimization problem over
a convex output-coefficient domain.
{Let $F_{\thetah}$, $\thetah\in\Theta_{\rm h}$, be the tangent feature field in
Definition~\ref{def:repu-construction}, assembled by the fixed atlas and partition of unity from
a fixed-depth, fixed-width chartwise $\rho_q$ network.  Every affine weight and bias is
a component of $\thetah$, and $\Theta_{\rm h}\subset\mathbb R^{p_0}$ is its
fixed
compact box.}  Direct differentiation of this finite architecture gives
an explicit constant $C_{\rm arch}$ such that
\begin{equation}\label{eq:mlsn-architecture-normalization}
 \sup_{\thetah\in\Theta_{\rm h}}\left\{
 \|F_{\thetah}\|_{C^{\eta+3}}
 +\|D_{\thetah}F_{\thetah}\|_{{\rm op};C^{\eta+2}}\right\}
 \le C_{\rm arch},
\end{equation}
where the derivative norm is the operator norm into $C^{\eta+2}$.
Choose the fixed constant $c_{\rm flow}>0$ in
Definition~\ref{def:repu-construction} so that $c_{\rm flow}C_{\rm arch}$
is sufficiently small, with the smallness bound depending only on the
regularity envelope.  Let $T_{\thetah}$ be the time-one flow of
$c_{\rm flow}\tau_JF_{\thetah}$, where $\tau_J=c_t2^{-2J}$.  In each chart, let
$J_{\thetah}$ be the Riemannian Jacobian of $T_{\thetah}$ and define the two
geometric pullbacks
\begin{equation}\label{eq:mlsn-geometric-pullbacks}
 U_{\thetah}^{\rm den}r
 :=J_{\thetah}(r\circ T_{\thetah}),
 \qquad
 (U_{\thetah}^{\rm nor}v)(x)
 :=P_{N_x}^{(m)}\{v(T_{\thetah} x)\}.
\end{equation}
Write $\mathsf U_{\thetah}=\operatorname{diag}
(U_{\thetah}^{\rm den},U_{\thetah}^{\rm nor})$.
These definitions are global.  The first identity preserves total mass by
change of variables; the second keeps
the output in the reference normal bundle.  Both operations are fixed
geometric layers around the trainable $T_{\thetah}$.

Set $B_{jk,\thetah}^s=U_{\thetah}^s\psi_{jk}^s$ and define the output-coefficient
synthesis operator
\begin{equation}\label{eq:mlsn-profile-synthesis}
 S_{J,m,\thetah}\thetao
 :=\left(
 \sum_{j,k}\theta_{jk}^{\rm den}B_{jk,\thetah}^{\rm den},
 \sum_{j,k}\theta_{jk}^{\rm nor}B_{jk,\thetah}^{\rm nor}
 \right)=:(r_{\theta},v_{\theta}).
\end{equation}
Define the
transported coarse field by truncating the input indices:
\[
 S_{\le L,m,\thetah}\thetao:=\mathsf U_{\thetah}
    \left(\sum_{j\le L,k}\theta_{jk}^{\rm den}\psi_{jk}^{\rm den},
          \sum_{j\le L,k}\theta_{jk}^{\rm nor}\psi_{jk}^{\rm nor}\right)
       =:(r_{\thetah,\le L},v_{\thetah,\le L}).
\]
The density atoms here include their fixed coarse mass corrections;
the discarded pivot is not an additional free coefficient.  Set
\[
 \mathcal D_{J,m}(\thetah):=
 \left\{\thetao:\begin{array}{l}
 |\theta_{jk}^{\rm den}|\le A2^{-j(\beta+d/2)},\quad
 |\theta_{jk}^{\rm nor}|\le A2^{-j(\beta+1+d/2)},\\
 \bar f_m+r_{\thetah,\le L}\ge\lambda/2
       \text{ on }\bar{\mathcal M}_m,\quad
 \|v_{\thetah,\le L}\|_{C^1}\le\varepsilon/2
 \end{array}\right\}.
\]
These are exactly the deterministic constraints in
Definition~\ref{def:repu-construction} and
Proposition~\ref{prop:canonical-law-sieve}.  For every pilot realization
and every hidden parameter, define the possibly empty fiber
\begin{equation}\label{eq:mlsn-profile-domain}
 \Theta_{n,J,m}^{\rm o}(\thetah):=
 \left\{\thetao\in\mathcal D_{J,m}(\thetah):
 \|S_{J,m,\thetah}\thetao-\widetilde\zeta_{J,m}\|_{\mathcal C,m}
                    \le r_{n,J}^{\rm pil}\right\}.
\end{equation}
Each fiber is compact and convex.  Empty fibers contribute no points
to the joint domain
\begin{equation}\label{eq:mlsn-joint-domain}
 \begin{aligned}
 \Theta_{n,J,m}^{\rm NN}:=
 &\{\theta=(\thetah,\thetao):\thetah\in\Theta_{\rm h},\quad
                    \thetao\in\Theta_{n,J,m}^{\rm o}(\thetah)\},\\
 h_{m,\theta}:={}&\mathfrak h_m(S_{J,m,\thetah}\thetao).
 \end{aligned}
\end{equation}
No singleton is inserted into a fiber and no switch uses the good-pilot
event.  On $\mathcal S_n\cap
\{\Theta_{n,J,\widehat m}^{\rm NN}\ne\varnothing\}$ minimize the exact
fit loss jointly over this domain.  Otherwise output the fixed
$h_{\rm ref}=\mathfrak h_1(0,0)$.  In particular, failure of one profile
to be feasible does not cause fallback when another profile is feasible.

Here are existence and measurability details.  At fixed $J$ all
parameters range over a finite-dimensional compact product box.
Synthesis, the transported coarse fields, and their displayed norms are
continuous in $\theta$.  All constraints are Borel in the pilot
data and continuous in those parameters.  Suprema over the reference
supports, including the finite-atlas $C^1$ norms, may equivalently be
taken over fixed countable dense sets.  Consequently the joint domain
has a Borel graph and compact sections.  Its nonemptiness event is
Borel, and it is a measurable compact-valued correspondence on that
event.  These claims follow also by projecting the Borel graph with
compact sections; graph measurability alone without the compact-section
argument would not suffice.

Every feasible pair, on every pilot realization, defines the
probability law
\[
 P_{m,\theta}=(I+v_{\theta})_\#
                \{(\bar f_m+r_{\theta})\bar\nu_m\}.
\]
Indeed mass preservation gives $\int r_{\theta}\,d\bar\nu_m=0$,
and the deterministic tail bound gives
$\bar f_m+r_{\theta}\ge7\lambda/16$ and
$\|v_{\theta}\|_{C^1}\le9\varepsilon/16$.
Thus these laws have a common deterministic compact support, and the
canonical construction places their representatives in the same
candidate envelope irrespective of the pilot event.
The induced law is continuous in the parameters in $W_2$.
For fixed $t_J>0$, its Gaussian-convolution score is continuous in the
parameters and is bounded by $C_{t_J}(1+|y|)$ on this compact family.
The exact corruption-averaged fit loss is therefore finite and
continuous in the parameters, by Gaussian dominated convergence, for
every finite real fit sample.  It is Borel in the data.
The measurable maximum theorem now gives a nonempty compact measurable
argmin correspondence whenever the joint domain is nonempty.
Successively minimizing its parameter coordinates in a fixed order
gives a unique Borel lexicographic selection.  The selected parameters,
fallback indicator, and induced law are therefore data-only measurable
statistics.  This argument uses the law integral and does not assume
continuity of a choice of prescribed-Jacobian representatives.

\begin{lemma}[Properties derived from the RePU construction]
\label{lem:mlsn-architecture}
The prescribed hidden pullbacks, their inverses and the transported
cutoffs are uniformly bounded at every spatial order used in the
analytic lemmas in Appendix~\ref{sec:weighted-local-estimation}. For every hidden parameter there is an original
retained comparator $\theta^0$ such that, in affine reference fields,
\begin{align}
 \|\theta^0-\theta^*\|_{E_p}&\le C_ph^{\eta-1},
          &&1<p<\infty,\label{eq:r25-comparator-energy}\\
 \|r^0-r^*\|_{H^{-1}}+\|v^0-v^*\|_2&\le Ch^{\eta+1},
 \qquad \|r^0-r^*\|_\infty+\|v^0-v^*\|_{C^1}\le Ch^\eta.
          \label{eq:mlsn-profile-population-bias}
\end{align}
Its normalized coefficient magnitudes are at most $K_*<A/8$,
and its input-coarse positivity and tube margins are at least
$\lambda/4$ and $\varepsilon/4$ for sufficiently large fixed $J_0$.
The inverse, Gaussian, product and empirical bounds of Appendix~\ref{sec:weighted-local-estimation} hold uniformly
over the hidden family.
\end{lemma}
\begin{proof}
The normalized flow's variational equations and Gronwall give
$T_{\thetah}-I=O(\tau_J)$, uniform flow and inverse bounds through
$C^{\eta+3}$, and $C^{\eta+2}$ bounds after one hidden-parameter derivative.
For this small flow, the normal-fiber projection and its inverse have
the same spatial bounds. The resulting pullbacks preserve the declared
Sobolev and Besov orders, norm equivalences, and total density mass.

Pull back the full truth coordinates, apply the original mean-zero
scalar and bundle cutoffs, and transport back. Jackson approximation
in $W^{1,p}\times L^p$, $H^{-1}\times L^2$, and $C^0\times C^1$ gives
\eqref{eq:r25-comparator-energy}--\eqref{eq:mlsn-profile-population-bias},
using $s^{-1}=c_t^{-1}h^{-2}$ for the normal energy term.
Proposition~\ref{prop:canonical-law-sieve} supplies $K_*<A/8$ and the
coarse margins, which the small flow preserves.
The conjugated cutoff inherits Lemma~\ref{lem:r25-cutoff}'s bounds.
The analytic lemmas in Appendix~\ref{sec:weighted-local-estimation}
depend only on these uniform bounds and the fixed candidate envelope,
so their constants are uniform over the hidden family.
\end{proof}

\subsection{First-law information and weak Gaussian estimates}
\label{sec:weighted-local-estimation}\label{sec:weighted-estimation}
This subsection establishes the analytic bounds used for the score roots
and coordinate errors in Appendix~\ref{sec:pilot-verification}.
We derive the law equation and retained-space inverse, then develop
weak Gaussian, noise, and whole-fiber curvature estimates.

Choose $0<\alpha<\min\{1,\beta-2\}$. The candidate bounds imply uniform
$C^{2,\alpha}$ reference densities and $C^{3,\alpha}$ support maps.
In physical coordinates their densities are positive and uniformly
$W^{2,\infty}$, since differentiating the surface Jacobian twice uses
at most three derivatives of the embedding. These conclusions also
hold at integer smoothness in the stated H\"older--Zygmund convention.
Only these low-order bounds are used in the Gaussian estimates.

For a current candidate $P$ put $p_s=K_sP$ and define, for signed laws
or distributional law directions,
\begin{equation}\label{eq:r25-information}
 \mathcal A_{s,P}(U,V):=2\int p_s
    \nabla\frac{K_sU}{p_s}\cdot\nabla\frac{K_sV}{p_s},
 \qquad I_\theta(w,z):=\mathcal A_{s,P_\theta}(DP_\theta[w],DP_\theta[z]).
\end{equation}
All convolutions and gradients in this expression are ambient.
Subscripts $s,P$ are suppressed when fixed.

\subsubsection{Exact law equation and first-law bounds}
The exact score equation expresses the estimation problem through law
perturbations; the forward and remainder bounds control its linear
and nonlinear parts.

For the clean empirical measure $P_n$ let $p_n=K_sP_n$.
Gaussian integration by parts in the corruption-averaged loss gives
\begin{align}
 L_{n,J}(\theta)&=\int p_n
       (|s_\theta|^2+2\operatorname{div}s_\theta)+D/s,
       \qquad s_\theta=\nabla\log p_\theta,\label{eq:r25-exact-loss}\\
 DL_{n,J}(\theta)[w]&=2\int(p_ns_\theta-\nabla p_n)
             \cdot\nabla B_w
       =\mathcal A_\theta(P_\theta-P_n,DP_\theta[w]),
       \quad B_w=K_sDP_\theta[w]/p_\theta.
       \label{eq:r25-exact-gradient}
\end{align}
Gaussian tails justify the integrations by parts. No sampled corruption
or finite-anchor objective is substituted for this exact criterion.

\begin{lemma}[Gaussian first-law bounds]
\label{lem:r25-gaussian-forward}
Fix a candidate $P=f\nu$ on its physical support $\mathcal M$,
$X\sim P$, $Z\sim\mathcal N(0,I_D)$ independently, and
$\epsilon=\sqrt s$. For every fixed $1<p<\infty$, uniformly in the
candidate and $s$, let
$T_Da=Ds[a\nu]$ and $T_Nb=Ds[-s\operatorname{div}(bP)]$.
Then
\begin{align}
 \|T_Da\|_{L^p(p_s)}&\le C_p\|a\|_{W^{1,p}},
       \label{eq:r25-density-forward}\\
 \|T_Da(X+\epsilon Z)-\nabla_{\mathcal M}(a/f)(X)\|_p
       &\le C_p\epsilon\|a\|_{W^{2,p}},
       \label{eq:r25-density-comparison}\\
 \|T_Nb\|_{L^p(p_s)}&\le C_p\|b\|_p,
       \label{eq:r25-normal-forward}\\
 \|T_Nb(X+\epsilon Z)-b(X)\|_p
       &\le C_p\epsilon\|b\|_{W^{1,p}}\quad
           \text{if $b$ is normal to $\mathcal M$}.
       \label{eq:r25-normal-comparison}
\end{align}
The forward estimates also allow fixed exponential-linear radial weights
in the clean corruption. They do not require retained directions.
\end{lemma}
\begin{proof}
Let $Q_y$ denote the posterior on $\mathcal M$ and put
$\phi_i(y)=\phi_s(y-x_i)$, $z_i=(y-x_i)/\epsilon$.
The source ball of radius $c\epsilon/(1+|z_i|)$ about $x_i$ gives
\[
 p_s(y)\ge c\epsilon^d(1+|z_i|)^{-d}\phi_i(y).
\]
For a clean anchor $x_0$ and $k$ posterior replicas, take the geometric
mean of these $k+1$ bounds before raising to the denominator power $k$.
Their joint density is bounded by
\begin{equation}\label{eq:r25-replica}
 \frac{\prod_{i=0}^k\phi_i(y)}{p_s(y)^k}
 \le C_k\epsilon^{-dk-D}\prod_{i=0}^k(1+|z_i|)^{dk/(k+1)}
       \exp\left\{-\frac{\sum_{i=0}^k|z_i|^2}{2(k+1)}\right\}.
\end{equation}
After integrating $y$, any fixed polynomial residual factors and
$\exp(c|z_0|)$ are bounded by the source kernel
$C_k\epsilon^{-dk}\exp\{-c_k\sum_{i=1}^k|x_i-x_0|^2/\epsilon^2\}$.
It has bounded integrals with any one source fixed, by upper volume growth.
Reach gives uniformly controlled graphs for close chords; remote pairs
have exponentially small contributions. If a residual exceeds
$c\epsilon^{-\kappa}$, half the Gaussian exponent bounds that event
by $\exp(-c\epsilon^{-2\kappa})$, while the other half integrates
every free $L^p$ source factor. Thus the tails are smaller than any
prescribed power of $\epsilon$.

Conditional Jensen for a two-replica covariance is valid for every
real $p\ge1$, without increasing the number of denominator factors.
On a near pair in source coordinates,
\[
 |u(\xi_1)-u(\xi_2)|^p\le |\xi_1-\xi_2|^p
 \int_0^1|Du((1-t)\xi_1+t\xi_2)|^p\,dt.
\]
For $t\le1/2$ change $\xi_1$ to the intermediate point, and for
$t\ge1/2$ change $\xi_2$. Their Jacobians are bounded below by
$2^{-d}$, up to a fixed graph constant. Applying this to the replica
kernel shows, for any fixed polynomial or exponential-linear residual
weight $W$,
\begin{equation}\label{eq:r25-increment}
 \mathbb E_{X,Z}Q_Y^{\otimes k}
     [W|u(x_i)-u(x_j)|^p]\le C_{p,W}\epsilon^p\|u\|_{W^{1,p}}^p.
\end{equation}
Second-order integral remainders have the analogous
$\epsilon^{2p}\|u\|_{W^{2,p}}^p$ bound. Long pairs use endpoint norms
and Gaussian decay, not paths outside a controlled chart.

For $u=a/f$, exact posterior differentiation gives
$T_Da(y)=\epsilon^{-2}\operatorname{Cov}_{Q_y}(x,u(x))$.
The replica form has one chord and one density increment, supplying
two powers of $\epsilon$ by~\eqref{eq:r25-increment}. This proves
\eqref{eq:r25-density-forward}, including the weights.
For the comparison use the graph $x+t+g_x(t)$ at the clean anchor,
with $t=\epsilon A$ and flat posterior $Q_0=\mathcal N(z_T,I_d)$.
The first unnormalized weight correction is
\[
 P_1(A,z)=D\log f(x)\cdot A+\tfrac12 z_N\cdot\mathrm{II}_x[A,A].
\]
On $|z|\le\epsilon^{-1/16}/8$, $|A|\le\epsilon^{-1/16}/2$,
the normalized weight is
$1+\epsilon(P_1-Q_0P_1)+O\{\epsilon^2\operatorname{poly}(A,z)\}$.
The remainder uses $D^2f,D^3g_x$; its polynomial bound is integrated,
not maximized on the growing region. The exact tails are covered by
\eqref{eq:r25-replica}. Since $\operatorname{Cov}_{Q_0}(A,A)=I$,
the linear part of $u$ gives its intrinsic gradient with error
$C\epsilon\operatorname{poly}(z)|Du|$. The replica remainder has
one chord and a second-order test remainder, yielding
$C_p\epsilon\|u\|_{W^{2,p}}$. This proves
\eqref{eq:r25-density-comparison} without $D^3f$.

For an ambient field $b$,
\[
 T_Nb(y)=Q_yb+\epsilon^{-2}
      \operatorname{Cov}_{Q_y}(x,b(x)\cdot(y-x)).
\]
The chord and residual moments prove~\eqref{eq:r25-normal-forward}
without differentiating $b$. For the comparison write
$b(x')=b(x)+[b(x')-b(x)]$ at the clean anchor.
The variable average and covariance cost
$C_p\epsilon\|b\|_{W^{1,p}}$. For the constant term, normality gives
$b(x)\cdot(x'-x)=b(x)\cdot g_x(t)=O(|t|^2)|b(x)|$.
Its chord times quadratic graph increment has order $\epsilon^3$;
division by $\epsilon^2$ proves~\eqref{eq:r25-normal-comparison}.
Smooth approximation gives the Sobolev assertions.
\end{proof}

\begin{lemma}[Exact geometric law remainder]
\label{lem:r25-law-remainder}
Let $P=\mu$, $P_1=(I+b)_\#\{(1+u)\mu\}$, where
$\int u\,d\mu=0$, $1+u\ge0$, $\|u\|_\infty\le A_0$ and
$\|b\|_\infty/\epsilon\le\rho_0$ for fixed constants. For
$R=P_1-P-u\mu+\operatorname{div}(b\mu)$ and $1\le p<\infty$,
\begin{equation}\label{eq:r25-law-remainder}
 \left\|\nabla\frac{K_sR}{p_s}\right\|_{L^p(p_s)}
 \le C_p\{s^{-1}\|ub\|_p+s^{-3/2}\||b|^2\|_p\}
 \le C_p(\|u\|_\infty+\|b\|_\infty/\epsilon)s^{-1}\|b\|_p.
\end{equation}
The displacement need not be normal to the current candidate support.
\end{lemma}
\begin{proof}
Uniform small-ball bounds give
$\mathbb E[|Z'|^k\mid Y=x+\epsilon z]\le C_k(1+|z|)^k$,
where $Z'=(Y-X')/\epsilon$. Indeed the denominator integrated on
$B(x,\epsilon)$ is at least
$c\epsilon^d e^{-(|z|+1)^2/2}$, whereas the numerator tail beyond
$R\epsilon$ is bounded by
$C\epsilon^d\sum_{j\ge\lfloor R\rfloor}(j+1)^de^{-j^2/2}$.
Split the moment integral at $R=C_k(1+|z|)$.

Put $v=b/\epsilon$ and
$M_l(x,z)=\exp\{lv\cdot z-l^2|v|^2/2\}$. The exact second derivative
of the path contribution, divided by its unshifted Gaussian, is
\[
 F_l=M_l\{2u\,v\cdot(z-lv)
       +(1+lu)([v\cdot(z-lv)]^2-|v|^2)\}.
\]
Taylor's integral remainder gives
$K_sR/p_s=\int_0^1(1-l)Q_yF_l\,dl$.
Both $|F_l|$ and $|D_zF_l|$ are at most
$C(|u||v|+|v|^2)e^{\rho_0|z|}(1+|z|)^3$.
For $m_s(y)=Q_yZ'$, differentiation of the posterior gives exactly
\[
 \nabla_y Q_yF_l=\epsilon^{-1}
 Q_y\{D_zF_l-F_l(Z'-m_s(y))\}.
\]
Conditional Jensen, the moment bound pointwise at each clean anchor,
and Gaussian integration bound its $L^p(p_s)$ norm by
$C_p\epsilon^{-1}\||u||v|+|v|^2\|_p$.
Integrate in $l$. No source derivative of $u$ or $b$ is taken, and
the normal-squared term in $F_l$ has not been discarded.
\end{proof}
\subsubsection{The same-bank finite-\texorpdfstring{$p$}{p} inverse}
The retained-space inverse converts score residuals into coordinate
bounds, as needed for the interior-root construction.

\begin{lemma}[Finite-$p$ first-law inverse]
\label{lem:vs-finite-p-inverse}
For each fixed $1<p<\infty$, sufficiently small fixed $c_t>0$ gives,
on the actual retained density and normal space,
\begin{equation}\label{eq:r25-inverse}
 \|I_\theta^{-1}F\|_{E_p}\le C_p\|F\|_{E_{p'}^*},
 \qquad I_\theta(w,w)\ge c\|w\|_E^2.
\end{equation}
Constants are uniform in $J$, the reference branch, and all hidden
parameters in their fixed compact box. The functional norm is on the
retained test space; no inverse on arbitrary Sobolev directions is asserted.
\end{lemma}
\begin{proof}
We first establish the scalar elliptic inverse without any Gaussian
perturbation. Let $V_h$ be the full, not mean-corrected, scalar space.
Subtract fixed-cutoff coordinate polynomials in its kernel integrals.
The moments of~\eqref{eq:r25-kernel-near}--\eqref{eq:r25-kernel-far}
and signed polynomial Jackson errors imply, for $1\le p\le\infty$,
\begin{align}
 \|\Pi_hu\|_{W^{k,p}}&\le C_p\|u\|_{W^{k,p}},\quad k=0,1,2,
       \label{eq:r25-projector-sobolev}\\
 \|(I-\Pi_h)u\|_{W^{t,p}}&\le C_ph^{2-t}\|u\|_{W^{2,p}},
       \quad t=0,1,\qquad
 \|(I-\Pi_h)u\|_p\le C_ph\|u\|_{W^{1,p}}.\nonumber
\end{align}
In particular, $\|D\Pi_h1\|_\infty+\|1-\Pi_h1\|_{H^1}\le Ch^2$.
These statements use $m_0>2$, not mixed derivatives of total order four.

Let $A=-\operatorname{div}(c\nabla)+I$ on the fixed smooth compact
reference, with uniformly elliptic symmetric $c\in C^{2,\alpha}$.
Write $a_0$ for its form and $R_h$ for its actual Ritz projector.
Continuous local elliptic regularity and a finite atlas give uniform
$W^{2,p}$ regularity and $W^{-1,p}\to W^{1,p}$ invertibility.
The scalar Green bound followed by rescaled interior estimates in
both variables gives
\[
 |D_x^aD_y^bG_A(x,y)|\le C\operatorname{dist}(x,y)^{2-d-a-b},
 \qquad a+b\le3.
\]
Only $D_x^2D_yG_A$, not a fourth mixed derivative, is needed below.

\paragraph*{Localized loads and annular error.}
Suppose $Ag=-\operatorname{div}\xi$, where $\xi$ is supported in
$B_{C\rho}(x_0)$, $\rho\ge h$, and
\[
 \|\xi\|_1\le C a_0^*,\qquad
 \|\xi\|_2\le C a_0^*\rho^{-d/2},\qquad
 \|\operatorname{div}\xi\|_2\le C a_0^*\rho^{-d/2-1}.
\]
Here $a_0^*$ is a load amplitude, unrelated to the form $a_0$.
Put $E_g=g-R_hg$, $\zeta=g-\Pi_hg$ and $\chi=E_g-\zeta\in V_h$.
Energy quasioptimality and continuous $H^2$ duality give
\begin{equation}\label{eq:r25-ritz-global}
 \|g\|_{H^1}+\|E_g\|_{H^1}\le Ca_0^*\rho^{-d/2},\qquad
 \|E_g\|_2+\|\zeta\|_2+\|\chi\|_2\le Ca_0^*h\rho^{-d/2}.
\end{equation}
For example, solve $Az=E_g$ and use
$\|E_g\|_2^2=a_0(E_g,z-\Pi_hz)$ to obtain the second bound.
On a radius-$R$ annulus with $R\ge C\rho$, the Green estimates give
$|D^kg|\le Ca_0^*R^{1-d-k}$, $k\le2$, and
$\int_{B_R}|g|\le Ca_0^*R$. A cutoff separates local $H^2$ approximation
from remote kernel tails, yielding
\begin{equation}\label{eq:r25-ritz-jackson}
 \|\zeta\|_{H^1(D_R)}\le Ca_0^*hR^{-d/2-1},\quad
 \|\zeta\|_{2,D_R}\le Ca_0^*h^2R^{-d/2-1},\quad
 \|\zeta\|_{1,D_R}\le Ca_0^*h^2/R.
\end{equation}
The remote projected derivative of order $a=0,1$ has bound
$Ca_0^*h^{m_0-a}R^{1-d/2-m_0}$, which is no larger than these terms.

Take a core of radius $R_0=M\rho$ and exterior dyadic annuli $D_j$
of radii $R_j=2^jR_0$, with fixed-overlap enlargements $\Omega_j$.
If $R_0$ exceeds a fixed geometric radius,~\eqref{eq:r25-ritz-global}
already suffices. Otherwise define finite sums
\[
 G=R_0^{d/2}\|E_g\|_{H^1(\mathrm{core})}
        +\sum_jR_j^{d/2}\|E_g\|_{H^1(D_j)},\qquad
 L=h^{-1}\|E_g\|_{1,\mathrm{core}}+\sum_jR_j^{-1}\|E_g\|_{1,D_j}.
\]
Then $\|E_g\|_{W^{1,1}}\le CG$. Core contributions and the analogous
$L$ sum for $\zeta$ are at most $C_Ma_0^*$ by the preceding bounds.
Macroscopic annuli are covered by a fixed finite atlas with subordinate
cutoffs, so none is assumed to lie in a single coordinate ball.

For $\omega=1$ on $D_j$, supported on an enlargement of radius $R=R_j$,
write $w=\omega^2$. Since $\chi$ is retained,
\[
 C_\omega=w\chi-\Pi_h(w\chi)
       =\int K_h(x,y)[w(x)-w(y)]\chi(y)\,dy.
\]
Differentiate before subtracting $\chi(x)$. The exact identity is
\begin{align*}
 DC_\omega={}&\int D_xK_h(x,y)[w(x)-w(y)][\chi(y)-\chi(x)]\,dy\\
 &+Dw(x)\int K_h(x,y)[\chi(y)-\chi(x)]\,dy
       +\chi(x)D\{w(x)\Pi_h1(x)-\Pi_hw(x)\}.
\end{align*}
The first two terms use the moments
$\int|D_xK_h|\operatorname{dist}^2\le Ch$ and
$\int|K_h|\operatorname{dist}\le Ch$ and local path changes of variables.
The scalar term contains all three errors
$D(w-\Pi_hw)+(\Pi_h1-1)Dw+wD\Pi_h1$, bounded by $Ch^2R^{-3}$.
For $M_j(u)=\int_{\rm remote}\operatorname{dist}(y,D_j)^{-d-m_0}|u(y)|\,dy$
with a fixed radius-$R$ buffer, the resulting commutator estimate is
\begin{align*}
 \|C_\omega\|_{H^1}
 \le{}&C(h/R)\|\chi\|_{H^1(\Omega_j)}+ChR^{-2}\|\chi\|_{2,\Omega_j}\\
 &+Ch^{m_0-1}R^{-d/2-m_0}\|\chi\|_{1,\Omega_j}
       +Ch^{m_0-1}R^{d/2}M_j(\chi).
\end{align*}
Orthogonality gives $a_0(E_g,w\chi)=a_0(E_g,C_\omega)$.
In the genuinely far part $C_\omega=-\Pi_h(w\chi)$. With an additional
buffer cutoff $\vartheta$, integrate by parts in the error variable:
$a_0(E_g,\vartheta C_\omega)=\langle E_g,A(\vartheta C_\omega)\rangle$.
Only two evaluation derivatives of $K_h$, and $c,Dc$, occur. The local
and remote contributions are bounded, respectively, by
\[
Ch^{m_0-2}R^{-d/2-m_0}\|\chi\|_{2,\Omega_j}\|E_g\|_{1,\Omega_j},
\qquad Ch^{m_0-2}R^{d/2}\|\chi\|_{2,\Omega_j}M_j(E_g).
\]
The norms of $\chi$ here are local because the source of the far
projection is $w\chi$. Replacing them by its global norm would lose
the local estimate needed for the annular sum.
Use Young's inequality and the scaled interpolation
$R^{-1}\|E_g\|_2\le\varepsilon_0\|E_g\|_{H^1}
+C_{\varepsilon_0}R^{-d/2-1}\|E_g\|_1$ on enlargements.
For arbitrarily small fixed $\varepsilon_0$, once $h/R$ is sufficiently
small, the cutoff energy identity thus implies
\begin{align}
 \|E_g\|_{H^1(D_j)}\le{}&\varepsilon_0\|E_g\|_{H^1(\Omega_j)}
  +C_{\varepsilon_0}\bigl\{\|\zeta\|_{H^1(\Omega_j)}
       +R^{-1}\|\zeta\|_{2,\Omega_j}\notag\\
 &+R^{-d/2-1}(\|E_g\|_{1,\Omega_j}+\|\zeta\|_{1,\Omega_j})
  +h^{m_0-1}R^{d/2}M_j(E_g-\zeta)\notag\\
 &+h^{m_0-2}R^{d/2+1}M_j(E_g)\bigr\}.
 \label{eq:r25-local-energy}
\end{align}
Constants here do not depend on $M,J$ or the error sums. In particular,
the $\sqrt{h/R}$ term from the commutator is made small before interpolation.
More explicitly, the energy identity first produces
$C(h/R)\|\chi\|_{H^1(\Omega_j)}\|E_g\|_{H^1(\Omega_j)}$.
Taking its square root gives the stated small factor. Each far term
has the form $B_j\|\chi\|_{2,\Omega_j}$; using
$\|\chi\|_{2,\Omega_j}\le\|E_g\|_{2,\Omega_j}
+\|\zeta\|_{2,\Omega_j}$, scaled interpolation and Young's inequality
leaves $R B_j$ and an arbitrarily small local energy term. This accounts
for the extra factor $R$ in the last term
of~\eqref{eq:r25-local-energy}.

Let $\ell_j=R_j^{-1}\|E_g\|_{1,D_j}$ and
$\ell_c=h^{-1}\|E_g\|_{1,\mathrm{core}}$. The remote term satisfies
\[
 h^{m_0-2}R_j^{d+1}M_j(E_g)\le C(h/R_j)^{m_0-2}(T\ell)_j,
\]
where
\[
 (T\ell)_j\le C\left\{2^{-j}\ell_c+\sum_{k<j}2^{k-j}\ell_k
       +\sum_{|k-j|\le C}\ell_k
       +\sum_{k>j}2^{-(k-j)(d+m_0-1)}\ell_k\right\}.
\]
Both row and column sums are bounded. The commutator tail has the
better factor $(h/R_j)^{m_0-1}$ and uses $E_g-\zeta$.
Multiply~\eqref{eq:r25-local-energy} by $R_j^{d/2}$ and sum.
Choose $\varepsilon_0$ small using bounded overlap, then $M$ large,
and absorb the local $H^1$ terms. The result is
\begin{equation}\label{eq:r25-GL-one}
 G\le C_Ma_0^*+CL,
\end{equation}
where the coefficient of $L$ is independent of $M$.

For $|\lambda|\le1$ supported in $D_j$, solve $Az=\lambda$.
Continuous estimates give $\|z\|_{H^2}\le CR_j^{d/2}$,
$\|z\|_\infty\le CR_j^2$, $\|Dz\|_\infty\le CR_j$, and
$\|D^2z\|_\infty\le C$ outside an enlarged annulus.
Hence $\|z-\Pi_hz\|_{H^1}\le ChR_j^{d/2}$.
The exterior $W^{1,\infty}$ error is at most $Ch$: split off the
part cut off near $D_j$, whose $L^1$ norm is $CR_j^{d+2}$ and whose
remote projected derivative is $Ch^{m_0-1}R_j^{2-m_0}\le Ch$.
The remaining part is uniformly $W^{2,\infty}$.
This does not assume second derivatives bounded where $\lambda$ is
discontinuous. Now
$\langle E_g,\lambda\rangle=a_0(E_g,z-\Pi_hz)$.
Use Cauchy--Schwarz inside the enlargement and the exterior error
times $\|E_g\|_{W^{1,1}}$ outside. With
$e_j=R_j^{d/2}\|E_g\|_{H^1(\Omega_j)}$, this gives explicitly
\[
 R_j^{-1}\|E_g\|_{1,D_j}
 \le C(h/R_j)(e_j+G),\qquad
 \sum_j e_j\le CG,\qquad \sum_j h/R_j\le C/M.
\]
The core term in $L$ is at most $C_Ma_0^*$ by
\eqref{eq:r25-ritz-global}. Consequently,
\begin{equation}\label{eq:r25-GL-two}
 L\le C_Ma_0^*+(C/M)G.
\end{equation}
Choose $M$ so the product of the two $M$-independent coefficients
in~\eqref{eq:r25-GL-one}--\eqref{eq:r25-GL-two}, divided by $M$,
is below $1/2$. Thus
\begin{equation}\label{eq:r25-local-load}
 G+L\le C_Ma_0^*,\qquad \|g-R_hg\|_{W^{1,1}}\le Ca_0^*.
\end{equation}
The sums were finite before absorption, so no a priori bound is assumed.

\paragraph*{From row loads to finite $p$.}
Derivative evaluation on the retained space is represented by
$D_xK_h(x,\cdot)$. Split this row into shells of radii $\rho_j=2^jh$
up to a small fixed geometric radius $r_*$. Their $L^1$ and $L^2$ norms
are bounded by $Ch^{-1}(h/\rho_j)^{m_0}$ and
$Ch^{m_0-1}\rho_j^{-d/2-m_0}$.
The cancellation uses tail sums of shell means. First remove the total
mean $\mu_x=D\Pi_h1(x)=O(h^2)$ against a fixed smooth unit-mass bump.
Partition the remaining mean-zero row into pieces $f_j$, $0\le j\le N_{\rm sh}$, ending with a
macroscopic piece, and let $b_j$ be unit-mass bumps on their containing
balls. Set
\[
 \mu_j=\int f_j,\qquad t_j=\sum_{k\ge j}\mu_k,\qquad
 F_j=f_j-t_jb_j+t_{j+1}b_{j+1}.
\]
At the two ends $t_0=t_{N_{\rm sh}+1}=0$. Hence $\int F_j=0$ and
$\sum_j F_j=\sum_j f_j$. The row bounds give
$|t_j|\le Ch^{m_0-1}\rho_j^{-m_0}$, including the removed $O(h^2)$
mean, since $m_0-1<2$ and the largest radius is fixed.
Each $F_j$ has support of radius $C\rho_j$ and the same scale of
$L^2$ load bound as its original shell. It is the divergence
of a field with the preceding load bounds and
$a_j^*=C(h/\rho_j)^{m_0-1}$. Explicitly, solve the mean-zero Neumann
Poisson problem on a containing ball and extend its gradient by zero;
zero normal flux prevents a boundary measure. Scaled Poincare and energy
give the required $L^2$ bound, and volume gives the $L^1$ bound.
The sum of $a_j^*$ is bounded.

The macroscopic tail and its correction bumps have $L^2$ norm
$C_{r_*}h^{m_0-1}$. Their continuous Green solutions have Ritz
$W^{1,1}$ error at most this order by energy quasioptimality and finite
volume. The total mean $D\Pi_h1=O(h^2)$ multiplies a fixed smooth
unit-mass bump and is treated identically. Thus the Green solution
$g_x$ of the entire row satisfies $\|g_x-R_hg_x\|_{W^{1,1}}\le C$.
For $E_u=u-R_hu$ and any $v_h\in V_h$,
\[
 D(\Pi_hu-R_hu)(x)=D\Pi_hE_u(x)
       =a_0(u-v_h,g_x-R_hg_x).
\]
Together with~\eqref{eq:r25-projector-sobolev} this gives
$\|DE_u\|_\infty\le C\inf_{v_h}\|u-v_h\|_{W^{1,\infty}}$.
Also $\int E_u=a_0(E_u,1-\Pi_h1)$; the mean bound, energy
quasioptimality and Poincare control $\|E_u\|_\infty$.
Consequently $R_h$ is uniformly bounded on $W^{1,\infty}$ and $H^1$.

For $p>2$, use Sobolev Lipschitz truncation $u=g_\lambda+b_\lambda$,
$\|g_\lambda\|_{W^{1,\infty}}\le C\lambda$, and
\[
 \|b_\lambda\|_{H^1}^2\le
 C\int_{\mathcal M(|u|+|Du|)>\lambda}(|u|+|Du|)^2
       +C\lambda^2|\{\mathcal M(|u|+|Du|)>\lambda\}|.
\]
The endpoint bounds for $R_h$ give a distribution-function estimate.
Integrate against $\lambda^{p-1}$; the integral
$\int_0^{\mathcal M f}\lambda^{p-3}d\lambda$ is finite for $p>2$,
and the maximal-function inequality gives $W^{1,p}$ stability.
For $1<p<2$, symmetry of $a_0$ and continuous elliptic duality give
$\|R_hu\|_{W^{1,p}}\le C\sup_{\|v\|_{W^{1,p'}}\le1}
|a_0(u,R_hv)|\le C_p\|u\|_{W^{1,p}}$.
Thus, for $B_h=R_hA^{-1}$ and $B=A^{-1}$,
\begin{equation}\label{eq:r25-ritz-Lp}
 \|B_h\|_{W^{-1,p}\to W^{1,p}}\le C_p,
 \qquad \|B_h-B\|_{L^p\to W^{1,p}}\le C_ph.
\end{equation}

\paragraph*{The density constraint and the normal mass inverse.}
In reference coordinates let $G=D\Phi^TD\Phi$ and write $q$ for the
reference density. The leading density form is
\[
 a_q(a,z)=2\int q\,\nabla(a/q)^TG^{-1}\nabla(z/q).
\]
Its principal coefficient is $c=2q^{-1}G^{-1}$ and its potential is
\[
 V_{\rm phys}=2\operatorname{div}(q^{-2}G^{-1}\nabla q)
                    +2q^{-3}\nabla q^TG^{-1}\nabla q.
\]
Only $D^2q,Dq,DG$ occur. Put $V=V_{\rm phys}-1$ and define on
$W^{1,p}\times\mathbb R$
\[
 T(u,\lambda)=(u+BVu-\lambda B1,\int u),\qquad
 T_h(u,\lambda)=(u+B_hVu-\lambda B_h1,\int u).
\]
The continuous constrained operator is Fredholm and injective:
the unconstrained kernel of $a_q$ is $\operatorname{span}(q)$,
and $\int q=1$ removes it. Coercivity and regularity give a uniform
inverse. Equation~\eqref{eq:r25-ritz-Lp} implies
$\|T_h-T\|\le C_ph$, so the discrete inverse is uniform for small $h$.
The equation $T_h(u_h,\lambda)=(B_hF,0)$ is precisely the original
mean-zero Galerkin problem. The fixed pivot only identifies that
subspace; it adds no free coordinate or box.
For a functional initially defined only on retained mean-zero tests,
Hahn--Banach supplies an extension to $W^{1,p'}$ of the same norm.
The discrete solution depends only on its restriction to those tests,
so this argument proves the asserted retained dual-norm bound.

For the normal block, partition the reference into $h$-cells and set
$s_0=d+m_0>d$. In the block norm
\[
 \|T\|_{s_0}=\sup_{i,j}(1+\operatorname{dist}(x_i,x_j)/h)^{s_0}
                     \|T_{ij}\|_{2\to2},
\]
polynomial convolution gives an algebra. Splitting a row at radius
$R$ bounds its sum by
$C(R^d\|T\|_{\rm op}+R^{d-s_0}\|T\|_{s_0})$, and optimizing gives
$\|T^2\|_{s_0}\le C\|T\|_{s_0}^{2-\vartheta}
\|T\|_{\rm op}^{\vartheta}$, $\vartheta=1-d/s_0>0$.
For uniformly positive $L$, apply this inequality to dyadic powers
of $I-L/c$, whose operator norm is less than one. Their algebra norms
decay exponentially up to a subexponential factor; the Neumann series
therefore converges in the same algebra, also for Hilbert-space cell blocks.
Indeed, for $U=I-L/c$ and $q_0=\|U\|_{\rm op}<1$, the recurrence gives
\[
 \log\|U^{2^k}\|_{s_0}
 \le 2^k\log q_0+C(2-\vartheta)^k.
\]
Binary decomposition and the algebra inequality give a summable bound
$\|U^n\|_{s_0}\le q_0^n\exp(Cn^\gamma+C\log(n+1))$,
$\gamma=\log_2(2-\vartheta)<1$. The case $U=0$ is immediate.
For the biorthogonal projector put
$C_h=I-(\Pi_h-\Pi_h^*)^2$. Idempotent block algebra gives $C_h\ge I$
and physical orthogonal projection $P_h=\Pi_hC_h^{-1}\Pi_h^*$.
For the positive normal multiplier $m_q=2qN_\theta^*N_\theta$,
$L_m=P_hm_qP_h+I-P_h$ is uniformly positive. Its mass inverse is
$P_hL_m^{-1}P_h$. Keep both exterior smoothing factors:
their cellwise $L^2$ row and column bounds are
$Ch^{-d/2}(1+\operatorname{dist}/h)^{-s_0}$.
Two discrete convolutions bound the physical inverse kernel by
$Ch^{-d}(1+\operatorname{dist}/h)^{-s_0}$. Schur gives all finite-$p$
normal inverse bounds. Block $L^2$ bounds alone would not suffice.

\paragraph*{Gaussian transfer at the same candidate.}
Put $J_\Phi=\sqrt{\det G}$,
$f=(q/J_\Phi)\circ\Phi^{-1}$, and for a reference-normal $b$ set
$T_\theta b=G^{-1}D\Phi^Tb$, $N_\theta b=b-D\Phi T_\theta b$.
Integration by parts identifies the physical first-law coordinates as
\begin{equation}\label{eq:r25-first-law-conversion}
 \left(\frac{a-\operatorname{div}_0(qT_\theta b)}{J_\Phi}
              \circ\Phi^{-1},\quad N_\theta b\circ\Phi^{-1}\right).
\end{equation}
Apply~\eqref{eq:r25-density-comparison} only to
$(a/J_\Phi)\circ\Phi^{-1}$ and
\eqref{eq:r25-normal-comparison} to $N_\theta b\circ\Phi^{-1}$.
For the nonretained tangential density correction use the forward
bound~\eqref{eq:r25-density-forward} instead:
$\|\operatorname{div}_0(qT_\theta b)\|_{W^{1,p}}
\le C\|b\|_{W^{2,p}}$. This uses $D^2q,D^3\Phi$, not $D^3q$.
Retained Bernstein bounds give errors
$C_p\epsilon h^{-1}\|a\|_{W^{1,p}}$,
$C_p\epsilon h^{-1}s^{-1}\|b\|_p$, and $C_ph^{-2}\|b\|_p$.
Their sum is $C_p(\sqrt{c_t}+c_t)\|(a,b)\|_{E_p}$.

The leading map is the pointwise tangent--normal sum
\[
 \mathcal L_\theta(a,b)(x,Z)
 =\nabla_{\mathcal M_\theta}[(a/q)\circ\Phi^{-1}](\Phi(x))
                      +s^{-1}N_\theta b(x).
\]
Its Gram form is exactly $a_q$ plus the $s^{-2}$ normal mass form.
H\"older at $p,p'$ applied separately to the three Gram-error terms
gives
\[
 |I_\theta(w,z)-2\mathbb E\mathcal L_\theta w\cdot\mathcal L_\theta z|
       \le C_{p,p'}\sqrt{c_t}\|w\|_{E_p}\|z\|_{E_{p'}}.
\]
The two diagonal inverses and a Neumann series prove
\eqref{eq:r25-inverse}. Include $p=2$ in the finite list of smallness
requirements to obtain coercivity. Constants may depend on the selected
finite $p$, as allowed by the fixed-$c_t$ theorem.
\end{proof}
\subsubsection{Relative-density adjoints and product tests}
These estimates connect the retained score equation to the dual tests
used to control density error in $H^{-1}$.

\begin{lemma}[Continuous weak Gaussian bounds]
\label{lem:r25-weak-gaussian}
For a current physical law $P=q\nu_{\mathcal M}$ define
\[
 B_s[t]=\frac{K_s(tP)}{p_s},\qquad
 \mathcal K_st(z)=\int\phi_s(y-z)B_s[t](y)\,dy,
\]
\[
 \mathcal F_st(z)=-2\int\phi_s(y-z)
       \{\Delta B_s[t]+\nabla\log p_s\cdot\nabla B_s[t]\}(y)\,dy.
\]
For $u\in H^1$, $t\in H^3$, and measurable $b$ with
$\|b\|_\infty\le\rho\sqrt s$ for fixed $\rho$, one has
\begin{align}
 |\mathcal A_P(uP,tP)-2\int q\nabla u\cdot\nabla t|
       &\le Cs\|u\|_{H^1}\|t\|_{H^3},\label{eq:r25-weak-comparison}\\
 \|\nabla\mathcal F_st\circ(I+b)\|_2
       &\le C\|t\|_{H^3},\label{eq:r25-adjoint-gradient}\\
 \|D^2\mathcal F_st\circ(I+b)\|_2
       &\le Cs^{-1/2}\|t\|_{H^3}.\label{eq:r25-adjoint-hessian}
\end{align}
Furthermore, for the remainder in Lemma~\ref{lem:r25-law-remainder},
\begin{align}
 |\mathcal A_P(-\operatorname{div}(bP),tP)|
       &\le C\|b\|_2\|t\|_{H^3},\label{eq:r25-linear-weak}\\
 |\mathcal A_P(R,tP)|&\le C(\|ub\|_2+s^{-1/2}\||b|^2\|_2)
                                      \|t\|_{H^3}.\label{eq:r25-remainder-weak}
\end{align}
\end{lemma}
\begin{proof}
The quotient heat equation is
$\partial_sB_s=\frac12\Delta B_s+\nabla\log p_s\cdot\nabla B_s$.
Differentiating the outer Gaussian and integrating by parts gives the
exact identities
\begin{equation}\label{eq:r25-adjoint-identity}
 \mathcal F_st=-2\partial_s\mathcal K_st,
 \qquad \mathcal A_P(U,tP)=\langle U,\mathcal F_st\rangle,
\end{equation}
where the scale derivative keeps the ambient query fixed.

\paragraph*{Differentiated normalized query expansion.}
At a fixed anchor $x$, use $X_x(v)=x+v+g_x(v)$, with
$g_x(v)=\frac12\mathrm{II}_x[v,v]+O(|v|^3)$ and
$Dg_x(v)=\mathrm{II}_x[v,\cdot]+O(|v|^2)$.
The graph volume satisfies $j_x(v)=1+O(|v|^2)$, $Dj_x(v)=O(|v|)$.
For $y=x+\epsilon w$, $v=\epsilon A$, the relative unnormalized
posterior weight against $Q_0=\mathcal N(w_T,I_d)$ is
\[
 W_\epsilon(A,w)=\frac{q(X_x(\epsilon A))}{q(x)}j_x(\epsilon A)
 \exp\{w_N\cdot g_x(\epsilon A)/\epsilon
                  -|g_x(\epsilon A)|^2/(2\epsilon^2)\}.
\]
On a moderate rescaled region, with a sufficiently small fixed
$\kappa>0$, it equals $1+\epsilon P_1+E_\epsilon$, where
\[
 P_1=D\log q_x(0)\cdot A+\tfrac12w_N\cdot\mathrm{II}_x[A,A],
 \qquad |E_\epsilon|+|\epsilon\partial_\epsilon E_\epsilon|
       \le C\epsilon^2(1+|A|+|w|)^M.
\]
Scale differentiation holds $A,w,x$ fixed. The density remainder and
its derivative use $D^2q$; the graph remainder and its derivative use
$D^3g_x$; the volume factor uses only $Dg_x,D^2g_x$.
For $Z_\epsilon=Q_0W_\epsilon$, Gaussian moments and small $\kappa$
give $|Z_\epsilon-1|\le1/2$. Differentiating the quotient itself yields
\[
 W_\epsilon/Z_\epsilon=1+\epsilon(P_1-Q_0P_1)+R_\omega,
 \quad Q_0R_\omega=0,
 \quad |R_\omega|+|\epsilon\partial_\epsilon R_\omega|
       \le C\epsilon^2(1+|A|+|w|)^{M'}.
\]
Subtract $t(x)$ before multiplying this expansion. With
$l=Dt_x(0)$, $H=D^2t_x(0)$, the test remainder after its quadratic
Taylor polynomial is controlled by the integral of $D^3t$.
For its scale derivative use the exact expression
\[
 \epsilon\partial_\epsilon R_t
 =\epsilon A\cdot[Dt_x(\epsilon A)-Dt_x(0)-\epsilon D^2t_x(0)A],
\]
which still uses only three derivatives of $t$. The omitted terms
are the cubic test remainder, the first weight correction times the
quadratic remainder, and $R_\omega$ times the first test increment.
Their $L^2(dx)$ norms, before and after scale differentiation, are
$C\epsilon^3\|t\|_{H^3}$ times Gaussian polynomials.
Indeed the local maps $x\mapsto X_x(\lambda\epsilon A)$ have
bounded Jacobians, so translated derivatives are controlled by
Minkowski and change of variables after a bounded local extension.

The query $x+\epsilon\zeta$ enters only through the outer Gaussian
$\phi_1(w-\zeta)$. Up to three $\zeta$ derivatives insert Hermite
polynomials bounded by $C_\rho(1+|w|)^3e^{-|w|^2/4}$ for
$|\zeta|\le\rho$. The same bound holds with the supremum over this
ball inside $L^2(dx)$, permitting a varying measurable $\zeta(x)$
without differentiating it. Outside the moderate region,
\eqref{eq:r25-replica} applies to the exact integrals. A denominator
or scale derivative inserts finitely many extra replicas or residual
polynomials, still with Gaussian decay at every source. Off-chart and
cutoff-derivative contributions have Schur norm $C_N\epsilon^N$ for
every fixed $N$. This also controls free $L^2$ tests in the tails.

Computing the retained Gaussian moments gives
\begin{align}
 \mathcal K_{\epsilon^2}t(x+\epsilon\zeta)
 ={}&t(x)+\epsilon l\cdot\zeta_T
       +\epsilon^2\{\mathcal L_qt(x)
            +\tfrac12H[\zeta_T,\zeta_T]
            +l\cdot\mathrm{II}_{\zeta_N}\zeta_T\}+R,\notag\\
 \mathcal L_q&=\Delta_{\mathcal M}+\nabla_{\mathcal M}\log q\cdot\nabla_{\mathcal M},
       \label{eq:r25-query-expansion}\\
 \left\|\sup_{|\zeta|\le\rho}|\partial_\zeta^aR|\right\|_2
 +\left\|\sup_{|\zeta|\le\rho}
       |\epsilon\partial_\epsilon\partial_\zeta^aR|\right\|_2
 &\le C\epsilon^3\|t\|_{H^3},\quad |a|\le3.\nonumber
\end{align}
At fixed physical query,
$-2\partial_s=-\epsilon^{-1}\partial_\epsilon
+\epsilon^{-2}\zeta\cdot\partial_\zeta$.
It annihilates the displayed homogeneous query terms. One and two
physical query derivatives are $\epsilon^{-1}\partial_\zeta$ and
$\epsilon^{-2}\partial_\zeta^2$, proving
\eqref{eq:r25-adjoint-gradient}--\eqref{eq:r25-adjoint-hessian}.
The Hessian uses the third query derivative explicitly included above,
but no additional source derivative.

\paragraph*{Weak density comparison.}
The symmetric kernel
$B_s(x,z)=q(x)q(z)\int\phi_s(y-x)\phi_s(y-z)/p_s(y)\,dy$
has marginal $q$. The heat equation gives
\begin{equation}\label{eq:r25-pair-identity}
 \mathcal A_P(uP,tP)=\partial_s\int B_s(x,z)
                [u(x)-u(z)][t(x)-t(z)]\,d\nu_xd\nu_z.
\end{equation}
Use midpoint coordinates $x_\pm=\exp_m(\pm v/2)$, $A=v/\epsilon$,
including their volume Jacobian $j(m,v)$. The normalized pair expansion is
\begin{align}
 \epsilon^d B_{\epsilon^2}(x_+,x_-)j(m,\epsilon A)
   &=q(m)(4\pi)^{-d/2}e^{-|A|^2/4}+R_\epsilon(m,A),\notag\\
 |R_\epsilon|+|\epsilon\partial_\epsilon R_\epsilon|
           +|A\cdot D_AR_\epsilon|
   &\le C\epsilon^2\operatorname{poly}(A)e^{-c|A|^2}.
   \label{eq:r25-midpoint}
\end{align}
For completeness, the ambient midpoint is
$m+\epsilon^2\mathrm{II}_m[A,A]/8+O(\epsilon^3|A|^3)$,
also after the scale or radial derivative, using $D^3\Phi$.
Jacobi fields give $j=1+O(\epsilon^2|A|^2)$ and the same order
for its scale and radial derivatives: the integral Jacobi equation
uses bounded curvature, and opposite endpoints cancel the first-order
term, without differentiating curvature. In the pair integral the
first denominator correction is
\[
 D\log q(m)\cdot w_T
 +\tfrac12w_N\cdot\{\mathrm{II}_m[w_T,w_T]+\operatorname{tr}\mathrm{II}_m\}
 -\tfrac18w_N\cdot\mathrm{II}_m[A,A].
\]
It is odd in $w$ and integrates to zero. The density product has no
linear term. The normalized Taylor estimates already proved, with
the same replica tails, give~\eqref{eq:r25-midpoint} using only
$D^2q,D^3\Phi$. At fixed $v$, the differentiated remainder kernel is
\[
 \tfrac12\epsilon^{-d-2}
       [\epsilon\partial_\epsilon R_\epsilon-dR_\epsilon
                                      -A\cdot D_AR_\epsilon].
\]
Its size is $C\epsilon^{-d}\operatorname{poly}(A)e^{-c|A|^2}$.
The two test increments in~\eqref{eq:r25-pair-identity} add
$\epsilon^2$, bounding this part by $Cs\|u\|_{H^1}\|t\|_{H^1}$.

For the leading part use geodesic flow $T_a$ on the tangent bundle,
its generator $X$, and Liouville measure times the radial Gaussian.
Put $G_u(m,v)=du_m(v)$ and
$A_\epsilon=\int_{-1/2}^{1/2}T_{a\epsilon}\,da$.
The leading integral is $s b_\epsilon$,
$b_\epsilon=\langle qA_\epsilon G_u,A_\epsilon G_t\rangle$.
In $c(a,b)=\langle qT_aG_u,T_bG_t\rangle$, move $T_a$ to the other
factor before differentiating. For example,
\[
 \partial_a^2c=\langle G_u,T_{-a}X^2(qT_bG_t)\rangle.
\]
All second derivatives cost at most two derivatives of $q$ and three
of $t$, so $|D^2c|\le C\|u\|_{H^1}\|t\|_{H^3}$.
Symmetric averaging cancels linear terms, yielding
$|b_\epsilon-b_0|+|\epsilon\partial_\epsilon b_\epsilon|
\le C\epsilon^2\|u\|_{H^1}\|t\|_{H^3}$ and
$b_0=2\int q\nabla u\cdot\nabla t$.
Differentiate $s b_\epsilon$ to prove~\eqref{eq:r25-weak-comparison}.
Smooth approximation extends this proof to $u\in H^1,t\in H^3$.

Finally~\eqref{eq:r25-adjoint-identity} pairs the linear displacement
with $\int b\cdot\nabla\mathcal F_st\,dP$.
For the exact remainder, Taylor's formula gives
\[
 \mathcal A_P(R,tP)=\int_0^1(1-l)\int
  [2u b\cdot\nabla\mathcal F_st(x+lb)
    +(1+lu)b^TD^2\mathcal F_st(x+lb)b]\,dP\,dl.
\]
The two query bounds prove~\eqref{eq:r25-linear-weak} and
\eqref{eq:r25-remainder-weak}.
\end{proof}

\begin{lemma}[Product approximation in the actual density space]
\label{lem:r25-product}
Let $q_{\rm ref}=\bar f+U_{\thetah}^{\rm den}v_h$ be a candidate
reference density with $v_h$ in the full retained scalar space.
If $\int q_{\rm ref}t\,d\nu_0=0$, there is a retained, transported
mean-zero density direction $z_h$ with
\begin{equation}\label{eq:r25-product}
 \|q_{\rm ref}t-z_h\|_{H^1}\le Ch^2\|t\|_{H^3}.
\end{equation}
Only the candidate's uniform $W^{2,\infty}$ bound is used; no third
candidate-density derivative is assumed.
\end{lemma}
\begin{proof}
First work without the hidden pullback. For retained $v_h$, reproduction
gives $E(x)=(I-\Pi_h)(v_ht)(x)
=\int K_h(x,y)v_h(y)[t(x)-t(y)]\,dy$.
Differentiate this identity before Taylor expansion:
\[
 D_iE=v_hD_it+\int D_{x_i}K_h(x,y)v_h(y)[t(x)-t(y)]\,dy.
\]
Write $v_h(y)=v_h(x)+D_lv_h(x)(y-x)_l+R_v$ with
$|R_v|\le C\|D^2v_h\|_\infty|y-x|^2$.
The constant term is exactly
$v_h[D_i(t-\Pi_ht)+tD_i\Pi_h1]$, bounded by Jackson.
For the affine term the signed second moment of $D_{x_i}K_h$ is
$O(h^2)$: express it through the differentiated moments of fixed-cutoff
coordinate polynomials $1,y_l,y_j,y_ly_j$; the exact polynomial terms
cancel and their $W^{1,\infty}$ errors are $O(h^2)$.
The remaining integral remainders use
\[
 \int|K_h(x,y)|\operatorname{dist}(x,y)^2dy\le Ch^2,
 \qquad
 \int|D_xK_h(x,y)|\operatorname{dist}(x,y)^3dy\le Ch^2.
\]
The second bound follows from
$Ch^{m_0}\int_h^1r^{1-m_0}dr\le Ch^2$, since $m_0>2$.
This uses the stronger differentiated far bound, not merely $h^{-1}$
times the undifferentiated tail. Integral remainders, bounded extensions
and Minkowski give the $L^2$ estimates without pointwise derivatives of
$t$. The undifferentiated error uses the analogous second and third
moments. Separated chart pieces have operator norm $O(h^{m_0})$.
Thus
\[
 \|(I-\Pi_h)(v_ht)\|_{H^1}
       \le Ch^2\|v_h\|_{W^{2,\infty}}\|t\|_{H^3}.
\]
Approximate the fixed smooth reference contribution $\bar f t$ by
ordinary Jackson. A unit-mass coarse atom $\varphi_0$ corrects the
integral; its coefficient is at most the $L^2$ approximation error.
For the hidden pullback use the exact identity
$(Uv_h)t=U[v_h(t\circ T^{-1})]$.
The prescribed flow has uniform derivatives through the orders required
to control $H^3$ tests and the pulled-back smooth reference contribution.
It preserves mass, so the transported coarse atom makes the same
correction. This proves~\eqref{eq:r25-product} in the actual bank.
\end{proof}
\subsubsection{Uniform empirical bounds and whole-fiber curvature}
We control empirical fluctuations uniformly over localized candidates
and establish curvature on each entire output fiber.

\begin{lemma}[Observation-space and integrated noise bounds]
\label{lem:r25-noise}
Write $p_*=K_sP^*$, $R_n=(p_n-p_*)/p_*$ and
$r(y)=\operatorname{dist}(y,\mathcal M^*)/\epsilon$.
Fix $B_0>0$ and let $n^{-B_0}\le\epsilon\le\epsilon_0$,
where $\epsilon_0>0$ is fixed and sufficiently small.
For each fixed $A>0$, the following holds with probability at least
$1-C_An^{-A}$, with constants allowed to depend on $B_0$:
\begin{equation}\label{eq:r25-noise-event}
 |R_n(y)|+\epsilon|\nabla R_n(y)|\le\delta_n e^{C_0r(y)}
 \quad\hbox{for every }y,\qquad
 \delta_n=C_A\left\{\sqrt{\frac{\log n}{n\epsilon^d}}
                         +\frac{\log n}{n\epsilon^d}\right\}.
\end{equation}
For every $0<\epsilon\le\epsilon_0$ and a sufficiently large fixed $c_1$,
the random variable
\[
 W_n^2=\int p_*(y)e^{c_1r(y)}
              (|\nabla R_n(y)|^2+\epsilon^{-2}|R_n(y)|^2)\,dy
\]
satisfies
\begin{equation}\label{eq:r25-integrated-noise}
 \mathbb EW_n^2\le Cn^{-1}\epsilon^{-d-2}.
\end{equation}
Simultaneously for all candidates in a fixed relative-density and
$O(\epsilon)$ displacement neighborhood of the truth,
\begin{align}
 |\mathcal A_\theta(P_n-P^*,DP_\theta[w])|
       &\le CW_n\|w\|_E,\label{eq:r25-noise-forward}\\
 |\mathcal A_\theta(P_n-P^*,DP_\theta[(a,0)])|
       &\le CW_n\|a\|_{H^1}\quad(a\in H^1),
       \label{eq:r25-noise-nonretained}\\
 \|\mathcal A_\theta(P_n-P^*,DP_\theta[\cdot])\|_{E_{p'}^*}
       &\le C_p\delta_n/\epsilon
       \quad\hbox{on~\eqref{eq:r25-noise-event}}.
       \label{eq:r25-noise-finite-p}
\end{align}
The last functional is restricted to the retained space. The density
forward bound in~\eqref{eq:r25-noise-nonretained} is not so restricted.
\end{lemma}
\begin{proof}
At a nearest source point, an $\epsilon$-ball gives
$p_*(y)\ge c\epsilon^{d-D}e^{-r(y)^2/2-Cr(y)-C}$.
For $f_y(x)=\phi_s(y-x)/p_*(y)$, the fields
$f_y-1$ and $\epsilon\nabla_yf_y$ have envelope
$C\epsilon^{-d}\operatorname{poly}(r)e^{Cr}$ and variance of the
same order. For variance use $P^*f_y=1$ and bound the remaining
factor with residual moments, rather than squaring the worst envelope.
One more $y$ derivative has the same envelope after multiplying by
$\epsilon^2$. Divide by $e^{C_0r}$ with $C_0$ larger than these
exponents. The normalized fields have envelope and variance
$C\epsilon^{-d}$ and Lipschitz constant $C\epsilon^{-d-1}$.
Only Lipschitz continuity of the distance function is needed.

Under the stated lower bound $\epsilon\ge n^{-B_0}$,
the normalized fields outside
$r(y)\le C_A\log n$ are deterministically below $n^{-A-3}$ after
enlarging $C_A$. Cover the remaining bounded ambient region by a grid
of spacing $\epsilon n^{-L}$ for a fixed sufficiently large $L=L(A,B_0)$.
Its cardinality is polynomial in $n$; its approximation error is
$n^{-A-3}$. Scalar Bernstein and a union bound over the $D+1$ fields
prove~\eqref{eq:r25-noise-event}. The intrinsic $\epsilon^{-d}$
variance determines the rate; ambient dimension enters only constants.

Independence and centering give the pointwise integrated variance.
For $Y=X+\epsilon Z$, $r(Y)\le|Z|$, so every remaining
exponential-linear radial moment is finite. This proves
\eqref{eq:r25-integrated-noise} without substituting the supremum event.
For nearby candidates, matching source anchors bound both
$p_*/p_\theta$ and its inverse by $Ce^{Cr(y)}$, and the score
difference by $C\epsilon^{-1}e^{Cr(y)}$. Use the exact identity
\[
 \nabla\frac{p_n-p_*}{p_\theta}
 =\frac{p_*}{p_\theta}
    \{\nabla R_n+R_n(\nabla\log p_*-\nabla\log p_\theta)\}.
\]
The weighted forward bounds in Lemma~\ref{lem:r25-gaussian-forward}
and Cauchy--Schwarz prove~\eqref{eq:r25-noise-forward}.
For a reference density direction its physical relative increment is
$a/q_{\rm ref}$; positivity and $W^{1,\infty}$ bounds control this
in $H^1$ without retention. This proves
\eqref{eq:r25-noise-nonretained}. H\"older and
\eqref{eq:r25-noise-event} give~\eqref{eq:r25-noise-finite-p}.
All comparisons are pointwise uniform in the neighborhood, hence also
in the data-selected hidden configuration.
\end{proof}

\begin{lemma}[Relative-density empirical comparison]
\label{lem:vs-fixed-root-hessian}
For a test $t$ fixed as a function of the common reference variable,
write $tP_\theta$ for its pushforward weighted law and
$N_\theta[t]=\mathcal A_\theta(P_n-P^*,tP_\theta)$.
On~\eqref{eq:r25-noise-event},
\begin{equation}\label{eq:r25-empirical-comparison}
 |N_\theta[t]-N_*[t]|
       \le C\delta_n\|\theta-\theta^*\|_E\|t\|_{H^1}.
\end{equation}
At the fixed truth let $Z_n(t)=(P_n-P^*)\mathcal F_s^*t$.
For $d>2$,
\begin{equation}\label{eq:r25-weak-noise}
 \mathbb E\|Z_n\|_{(H^3)^*}^2\le Cn^{-1}s^{-(d-2)/2},
 \qquad \mathbb E\|Z_n\|_{(H^3)^*}\le Cn^{-1/2}h^{1-d/2}.
\end{equation}
\end{lemma}
\begin{proof}
Let $Q=Q_{\theta,y}$ be the posterior in reference coordinates,
$z=(y-\Phi_\theta(x))/\epsilon$, and $B_\theta[t]=Q t$.
Along a path with $\dot q/q=u$ and $\dot\Phi=b$, quotient differentiation
gives
\[
 DB_\theta[t]=\operatorname{Cov}_Q(t,u)
                         +\epsilon^{-1}\operatorname{Cov}_Q(t,z\cdot b).
\]
Differentiate in $y$ exactly:
\begin{align*}
 \nabla\operatorname{Cov}_Q(t,u)
 &=-\epsilon^{-1}Q[(t-Qt)(u-Qu)(z-Qz)],\\
 \nabla\{\epsilon^{-1}\operatorname{Cov}_Q(t,z\cdot b)\}
 &=\epsilon^{-2}\operatorname{Cov}_Q(t,b)
   -\epsilon^{-2}Q[(t-Qt)(z\cdot b-Q(z\cdot b))(z-Qz)].
\end{align*}
The first term on the second line is essential. Replace centered
factors by differences against replicas. In the joint replica integral,
Cauchy--Schwarz and~\eqref{eq:r25-increment} give one $\epsilon$
for each $H^1$ increment. The density term is bounded in weighted
$L^1(p_\theta)$ by $C\epsilon\|t\|_{H^1}\|u\|_{H^1}$; the
displacement terms by $C\epsilon^{-1}\|t\|_{H^1}\|b\|_2$.
Thus
\[
 \|\nabla DB_\theta[t]\|_{L^1(p_\theta;\,\mathrm{weight})}
 \le C\epsilon\|t\|_{H^1}
                 (\|\dot q\|_{H^1}+s^{-1}\|\dot\Phi\|_2).
\]
Matching-anchor comparisons permit integration along the reference
segment using a fixed larger weight relative to $p_*$.
This is an $L^1$ product estimate; no dimension-dependent Sobolev
embedding into a pointwise product is used.
With $r_n=p_n-p_*$, exact subtraction yields
\begin{align*}
 N_\theta[t]-N_*[t]
 ={}&2\int p_*\nabla R_n\cdot\nabla(B_\theta[t]-B_*[t])\\
 &-2\int r_n(s_\theta-s_*)\cdot\nabla B_\theta[t].
\end{align*}
The preceding $L^1$ bound and~\eqref{eq:r25-noise-event} control the
first term. Weighted Cauchy--Schwarz and first-law forward estimates
control the second, proving~\eqref{eq:r25-empirical-comparison}.
In particular there is no additional factor $h^{-1}$.

For the expectation bound at fixed truth, the exact adjoint kernel
satisfies
\[
 \|\mathcal F_s^*t\|_\infty\le C\|t\|_{W^{2,\infty}},
 \qquad \|\mathcal F_s^*(x,\cdot)\|_{L^2(P^*)}\le Cs^{-1-d/4}.
\]
The second follows by differentiating the exact two-Gaussian kernel
in~\eqref{eq:r25-adjoint-identity} and applying the replica bounds;
each variance derivative costs $s^{-1}$. For the first, the kernel
annihilates constants, its absolute second moment is $O(1)$, and its
signed first graph-coordinate moment is $O(1)$ by the linear-test
version of the normalized scale calculation. Taylor's formula suffices.
Take a fixed smooth orthonormal multiscale decomposition of $H^3$,
independent of the estimator bank. At $2^j\le s^{-1/2}$, Bernstein
bounds the squared norm of its evaluation functional on a level's
$L^2$ unit sphere by $C2^{j(d+4)}$. Multiply by $2^{-6j}$ and sum:
the result is $Cs^{-(d-2)/2}$. No additional level dimension multiplies
this already complete functional norm. Above that scale, Bessel and
the row $L^2$ bound give $Cs^3s^{-2-d/2}$, the same order.
Independence and centering now prove~\eqref{eq:r25-weak-noise}.
The dual norm includes data-dependent tests, and a finite reference
union changes only constants.
\end{proof}

\begin{lemma}[Whole-fiber empirical curvature]
\label{lem:mlsn-uniform-hessian}
Write $P_\theta=\Phi_\#(q\nu_0)$ and
$P^*=(\Phi+b)_\#\{(q+e)\nu_0\}$. Put
$\mathfrak{r}_\theta=\|e\|_\infty+\|b\|_{C^1}$ and suppose
$\|b\|_\infty\le\rho\epsilon$, $\|e/q\|_\infty\le\rho$
for a fixed small $\rho$. For every retained direction $w$, on
\eqref{eq:r25-noise-event},
\begin{equation}\label{eq:r25-hessian}
 |D^2L_{n,J}(\theta)[w,w]-I_\theta(w,w)|
       \le C(\mathfrak{r}_\theta/\epsilon+\delta_n/\epsilon)\|w\|_E^2.
\end{equation}
Consequently, if the right prefactor tends to zero, the exact empirical
loss is uniformly strongly convex on each entire feasible output fiber.
\end{lemma}
\begin{proof}
Put $U=Ds_\theta[w]$, $V=D^2s_\theta[w,w]$. Exact differentiation gives
\[
 D^2L_{n,J}[w,w]=2\int p_n|U|^2
                   +2\int(p_ns_\theta-\nabla p_n)\cdot V.
\]
For a posterior anchor $x_0$, define weighted norms by integrating
$p_\theta(y)Q_y[e^{c|z_0|}|F(y,x_0)|^k]$ with
$z_0=(y-\Phi(x_0))/\epsilon$. Replica bounds allow all fixed weights.
For $w=(a,v)$, put $u=a/q$ and $k=z\cdot v/\epsilon$. The exact
second log derivative is
\[
 D^2\log p_\theta=-(Qu)^2+2\operatorname{Cov}_Q(u,k)
                       +\operatorname{Var}_Q(k)-Q|v|^2/s.
\]
Its spatial gradient splits into $V_{DD},V_{DN},V_{NN}$, with
\begin{align*}
 \|U\|_{2,c}&\le C(\|a\|_{H^1}+\epsilon^{-2}\|v\|_2),\qquad
 \|V_{DD}\|_{1,c}\le C\|a\|_{H^1}^2,\\
 \|V_{DN}\|_{1,c}&\le C\epsilon^{-1}\|a\|_{H^1}\|v\|_2,
 \qquad \|V_{NN}\|_{1,c}\le C\epsilon^{-3}\|v\|_2^2.
\end{align*}
For the mixed term differentiate the covariance as in the previous
lemma; both terms have one $u$ increment and no derivative of $v$.
For the normal-squared term the variance derivative and
$\nabla Q|v|^2/s$ both contain two $L^2$ factors and cost
$\epsilon^{-3}$. For $V_{DD}=-2(Qu)\nabla Qu$ there is the extra
anchor-normal bound
\[
 \|P_{N_{x_0}}V_{DD}\|_{1,c}\le C\epsilon\|a\|_{H^1}^2.
\]
Indeed $\nabla Qu=\epsilon^{-2}\operatorname{Cov}_Q(u,\Phi(x))$;
the normal chord at $x_0$ is quadratic in source distance. Its
$\epsilon^2$, together with the $\epsilon$ from a centered $u$
increment, gives the extra $\epsilon$ after division by $\epsilon^2$.
Jensen followed by Cauchy--Schwarz in the full replica integral proves
the weighted bound. The projection is at the posterior anchor, not
at an unintegrated nearest point.

Matching-anchor ratios give
$|p_*/p_\theta-1|\le C(\|e\|_\infty+\|b\|_\infty/\epsilon)e^{C|z_0|}$.
For $R=s_\theta-s_*$, its tangent and normal components satisfy
\begin{align*}
 |P_{T_{x_0}}R|&\le C\mathfrak{r}_\theta\epsilon^{-1}e^{C|z_0|},\\
 |P_{N_{x_0}}R|&\le C(\|e\|_\infty/\epsilon
                          +\|b\|_\infty/\epsilon^2)e^{C|z_0|}.
\end{align*}
To verify the improved tangent bound choose a fixed-atlas source
field $X$ with $D\Phi(x_0)X(x_0)=\xi$ for a unit tangent $\xi$.
Let $A_0=\xi-D\Phi X$, $A_b=A_0-DbX$, $R_0=y-\Phi$,
$R_b=R_0-b$, and $c_X=\operatorname{div}(qX)/q$.
Source integration by parts gives
$\xi\cdot\nabla\log p_b=Q_b[c_X-A_b\cdot R_b/s]$.
The difference in integrands is
\[
 -A_b\cdot R_b/s+A_0\cdot R_0/s
 =[A_0\cdot b+(DbX)\cdot R_0-(DbX)\cdot b]/s.
\]
Since $A_0(x_0)=0$ and $A_0$ is Lipschitz, the first two terms cost
$C\|b\|_{C^1}/\epsilon$ times a residual weight; the last costs
$C\rho\|Db\|_\infty/\epsilon$. The change of posterior is controlled
by the exact likelihood-ratio path, whose logarithmic derivative is
$(R_0-lb)\cdot b/s$. It costs $C\|b\|_\infty/\epsilon$.
A density change is a posterior covariance costing
$C\|e\|_\infty/\epsilon$, without differentiating $e$.
The ordinary full-score comparison gives the normal bound.

Insert the posterior anchor into the exact Hessian integral.
Changing its first weight from $p_\theta$ to $p_*$ costs
$C\mathfrak{r}_\theta\epsilon^{-1}\|w\|_E^2$.
The tangent residual pairs with the preceding full $V$ bounds.
The larger normal residual, of size $C\mathfrak{r}_\theta/\epsilon^2$, pairs
with the extra normal $V_{DD}$ bound and the mixed and normal-squared
bounds, each at most $C\epsilon\|w\|_E^2$.
Thus the population Hessian difference is
$C\mathfrak{r}_\theta\epsilon^{-1}\|w\|_E^2$, with all density ratios retained
in a larger fixed radial weight.
Finally,
\[
 p_ns_\theta-\nabla p_n
    =p_*[(1+R_n)(s_\theta-s_*)-\nabla R_n].
\]
On~\eqref{eq:r25-noise-event}, the first Hessian term changes by
$C\delta_n\|w\|_E^2$, the last gradient term by
$C\delta_n\epsilon^{-1}\|w\|_E^2$, and the product with the
population residual by $C\delta_n\mathfrak{r}_\theta\epsilon^{-1}\|w\|_E^2$.
This proves~\eqref{eq:r25-hessian}; coercivity follows from
Lemma~\ref{lem:vs-finite-p-inverse}.
\end{proof}
\subsection{Pilot localization and coordinate error bounds}
\label{sec:pilot-verification}

This subsection verifies the localization and error bounds used in
Appendix~\ref{sec:mlsn-proof}.
We construct a neighborhood using independent pilot samples, establish
interior empirical and population roots, and derive uniform coordinate estimates.

Split the data into three comparable parts
$I_{\rm den},I_{\rm nor},I_{\rm fit}$.
$I_{\rm den}$ selects a reference law chart and estimates its intrinsic
density, $I_{\rm nor}$ estimates the normal displacement in that chart, and
$I_{\rm fit}$ is reserved for the joint RePU-coordinate fit.  We refer to the
first two parts as the preliminary, or pilot, samples.  Their output determines
the localized coefficient domain; conditioning on this output leaves the
fitting observations independent for the empirical-process analysis in
Appendix~\ref{sec:weighted-local-estimation}. Put
$\widehat P_{\rm den}=|I_{\rm den}|^{-1}\sum_{i\in I_{\rm den}}
\delta_{Y_0^i}$, and select, with fixed tie breaking,
\begin{equation}\label{eq:vs-chart-selector}
 \widehat m\in\arg\min_{1\le m\le M_0}
 W_2^2(\widehat P_{\rm den},\bar P_m).
\end{equation}
Let $\mathcal E_n^{\rm chart}:=
\{W_2(\widehat P_{\rm den},P^*)\le a_0\}$.  The selector and the triangle
inequality give $W_2(P^*,\bar P_{\widehat m})\le3a_0$ there; compact support,
a finite-partition approximation, and McDiarmid's inequality yield
$\mathbb P\{(\mathcal E_n^{\rm chart})^c\}\le Cn^{-2}$ uniformly.

\begin{lemma}[Observable pilot localization]
\label{lem:vs-observable-pilot}
Under Assumptions~\ref{G1}--\ref{G2} and the three-way split, fix
$\tau_J=c_t2^{-2J}$ and $\beta=(\eta+2)/2$. For sufficiently large
$J,n$ in the theorem regime, use the accuracy condition
\begin{equation}\label{eq:vs-pilot-sample-condition}
 Q_J\log n\le c_0 n
\end{equation}
for fixed $c_0$.  Let $\varepsilon$ be the fixed common graph-tube radius
in Proposition~\ref{prop:canonical-law-sieve}, chosen within the unique
projection tubes of the finite reference supports, and set
\[
 \mathcal T_m:=\{y:\operatorname{dist}(y,\bar{\mathcal M}_m)
                         <\varepsilon\}.
\]
Fix a point $x_m^0\in\bar{\mathcal M}_m$ for each branch.  For every
pilot observation define the Borel maps
\[
 (X_{i,m},V_{i,m}):=
 \begin{cases}
  (\Pi_{\bar{\mathcal M}_m}(Y_0^i),
       Y_0^i-\Pi_{\bar{\mathcal M}_m}(Y_0^i)),&Y_0^i\in\mathcal T_m,\\
  (x_m^0,0),&Y_0^i\notin\mathcal T_m.
 \end{cases}
\]
Use the smallest index in the chart-selector argmin.  The observable
safety event for the selected branch is
\begin{equation}\label{eq:pm-observable-safety}
 \mathcal S_n:=
 \bigcap_{i\in I_{\rm den}\cup I_{\rm nor}}
       \{Y_0^i\in\mathcal T_{\widehat m}\}.
\end{equation}
Both preliminary samples and the fit sample are taken nonempty.  An
empty subsample, if allowed at small sample sizes, invokes $h_{\rm ref}$
before forming any empirical average.  All remaining definitions are
made for every realization, including $\mathcal S_n^c$.
With
$K_{J,m}$ the wavelet projection kernel and
$\{\Psi_{\lambda,m}^{\rm an},\Psi_{\lambda,m}^{\rm sy}\}$ the stable dual
normal frame, define
\begin{align}
 \widehat f_{J,m}^{\rm raw}(x)
 &=|I_{\rm den}|^{-1}\sum_{i\in I_{\rm den}}K_{J,m}(x,X_{i,m}),
 \label{eq:vs-pilot-density}\\
 \widehat\theta_{\lambda,m}^{\rm pil}
 &=|I_{\rm nor}|^{-1}\sum_{i\in I_{\rm nor}}
 \frac{\langle V_{i,m},\Psi_{\lambda,m}^{\rm an}(X_{i,m})\rangle}
 {\widetilde f_{J,m}(X_{i,m})},\qquad
 \widetilde v_{J,m}=\sum_{|\lambda|\le J}
 \widehat\theta_{\lambda,m}^{\rm pil}\Psi_{\lambda,m}^{\rm sy}.
 \label{eq:vs-pilot-normal}
\end{align}
Here $\widetilde f_{J,m}$ is the clipped-to-$[m_f/2,2M_f]$ and normalized
version of $\widehat f_{J,m}^{\rm raw}$.  Put
$\widetilde\zeta_{J,m}=(\widetilde f_{J,m}-\bar f_m,\widetilde v_{J,m})$,
$\|(r,v)\|_{\mathcal C,m}=\|r\|_\infty+\|v\|_{C^1}$, and
\begin{align}
 r_{n,J}^{\rm pil}
 &:=A_{\rm pil}\left\{2^{-J(\eta-1)}
 +2^J\sqrt{Q_J(1+J)\log n/n}\right\},
 \label{eq:vs-pilot-radius-supp}\\
 \widetilde\Theta_{n,J,m}^{\rm loc}
 &:=\{\zeta\in\Theta_{J,m}:
 \|\zeta-\widetilde\zeta_{J,m}\|_{\mathcal C,m}\le r_{n,J}^{\rm pil}\}.
 \label{eq:vs-pilot-local-set}
\end{align}
The set in~\eqref{eq:vs-pilot-local-set} is the unwarped localized set
and is allowed to be empty; it is not replaced by a singleton inside
this definition.  The jointly trained estimator instead uses the joint
domain and the observable fallback specified in Appendix~\ref{sec:repu-class}.
For a constant
$A_{\rm pil}$ depending only on the regularity bounds, there is, for
each fixed truth, an event $\mathcal E_{n,J}^{\rm pil}$ in the
sigma-field of the first two samples with
$\mathbb P\{(\mathcal E_{n,J}^{\rm pil})^c\}\le Cn^{-2}$.
Take this event to be the intersection of the branchwise pilot-accuracy
events over all admissible branches; the fixed finite number of branches
preserves the probability bound.
This is an analysis event, not a truth-independent membership test.
Write $\mathcal G_{n,J}^{\rm pil}:=
\mathcal E_n^{\rm chart}\cap\mathcal E_{n,J}^{\rm pil}$.
On this event the selected branch is admissible, $\mathcal S_n$ holds,
and
\begin{align}
 \|\widetilde\zeta_{J,m}-\zeta_m^*\|_{\mathcal C,m}
       &\le\tfrac14 r_{n,J}^{\rm pil},\label{eq:vs-pilot-bound}\\
 \frac{r_{n,J}^{\rm pil}}{\sqrt{\tau_J}}
       &\le C h^{\eta-2}(1+\log n)\longrightarrow0,
       \quad h=2^{-J}\asymp n^{-1/(2\eta+d)}.
       \label{eq:vs-pilot-shrinking}
\end{align}
Every feasible output is within $5r_{n,J}^{\rm pil}/4$ of the truth
in reference $C^0\times C^1$ norm on this event. The smaller retained
root neighborhoods and their three separate margins are verified in
Lemma~\ref{lem:pm-newton-margins}. None of these analysis events is
used as a membership test by the estimator.
\end{lemma}
\begin{proof}
Fix an admissible branch, without conditioning on the selected index.
Its projected clean observations have density $f_m^*$, uniformly
positive and $C^\eta$, and $V_{i,m}=v_m^*(X_{i,m})$ with
$v_m^*\in C^{\eta+1}$. From Lemma~\ref{lem:r25-cutoff}, for
evaluation derivatives $a=0,1,2$ the scalar and bundle kernels satisfy
\[
 \sup_x\int|D_x^aK_h(x,z)|^2dz\le Ch^{-d-2a},\qquad
 \sup_{x,z}|D_x^aK_h(x,z)|\le Ch^{-d-a}.
\]
Jackson gives density sup-norm and normal $C^1$ bias $Ch^\eta$;
also $\|\Pi_h^{\rm nor}u\|_{C^1}\le Ch^{-1}\|u\|_\infty$.
For each fixed $A>0$, scalar Bernstein and a polynomial-size net
of mesh $hn^{-L}$ in the fixed reference atlas give
\[
 \|\widehat f^{\rm raw}-\Pi_hf^*\|_\infty\le C_A t_n,
 \qquad t_n=\sqrt{\frac{\log n}{nh^d}}+\frac{\log n}{nh^d},
\]
except with probability $C_An^{-A}$. The first evaluation derivative
controls the net error. Clipping cannot increase the pointwise error
to a truth in the clipping interval; normalization changes it by at
most a constant times that error, since the true integral is one and
the clipped integral is bounded away from zero. Therefore
\begin{equation}\label{eq:r25-pilot-density-error}
 \|\widetilde f-f^*\|_\infty\le C_A(h^\eta+t_n).
\end{equation}

Condition on the whole density sample $I_{\rm den}$. The independent
normal sample gives exactly
\[
 \widetilde v(x)=|I_{\rm nor}|^{-1}\sum_{i\in I_{\rm nor}}
       K_h^{\rm nor}(x,X_i)V_i/\widetilde f(X_i),
 \qquad \mathbb E(\widetilde v\mid I_{\rm den})
       =\Pi_h^{\rm nor}(v^*f^*/\widetilde f).
\]
Apply Bernstein to its components and first evaluation derivatives;
the second derivative controls the net error. Bounded $V_i$ and
$1/\widetilde f$ preserve the variance and envelope orders, giving
$\|\widetilde v-\mathbb E(\widetilde v\mid I_{\rm den})\|_{C^1}
\le C_Ah^{-1}t_n$. Its conditional bias is at most
\[
 \|\Pi_h^{\rm nor}v^*-v^*\|_{C^1}
 +\|\Pi_h^{\rm nor}[v^*(f^*/\widetilde f-1)]\|_{C^1}
 \le Ch^\eta+Ch^{-1}\|\widetilde f-f^*\|_\infty.
\]
No derivative of the clipped density is required. Consequently,
\begin{equation}\label{eq:r25-pilot-error}
 \|\widetilde f-f^*\|_\infty+\|\widetilde v-v^*\|_{C^1}
       \le C_A(h^{\eta-1}+h^{-1}t_n).
\end{equation}
The condition $nh^d/\log n\to\infty$ absorbs the linear Bernstein
term. The prescribed $r_{n,J}^{\rm pil}$ dominates four times this
bound for a sufficiently large fixed $A_{\rm pil}$. At the theorem
resolution, division by $\sqrt{\tau_J}$ gives
$Ch^{\eta-2}(1+\log n)\to0$ for every fixed $\eta>2$.
There is no requirement that $r_{n,J}^{\rm pil}/\tau_J\to0$.

Intersect the accuracy events over the fixed finite admissible branch
family before selecting $\widehat m$. The selector and density estimate
use the same $I_{\rm den}$; no independence between them is assumed.
The normal sample is independent of that whole sample, and the fitting
sample is independent of both. Intersecting also with the chart event
preserves failure probability $Cn^{-2}$. On an admissible branch
$\|v_m^*\|_{C^1}\le\varepsilon/8$, so all clean pilot observations
lie in its projection tube and $\mathcal S_n$ holds. The totalized
projections, clipping, normalization and coefficient operations are
Borel. The joint-domain and fallback rules in Appendix~\ref{sec:repu-class} complete
the measurable estimator on every data realization.
\end{proof}
\subsubsection{Interior empirical and population roots}
\label{sec:repu-roots}
We verify the three feasibility margins and construct the empirical
and population roots needed to apply Theorem~\ref{thm:local-erm-transfer}.

\begin{lemma}[Separate margins and interior profile roots]
\label{lem:pm-newton-margins}
Fix $\eta>2$, $d>2$, $h\asymp n^{-1/(2\eta+d)}$ and
$s=c_t h^2$. Choose one finite $p>d$ such that
\begin{equation}\label{eq:r25-p-choice}
 \sigma=d/p\in(0,\min\{1,\eta-\beta\}),\qquad
 \beta=(\eta+2)/2.
\end{equation}
For sufficiently small fixed $c_t$ and sufficiently large $n$, on the
pilot and empirical good events every prescribed output fiber contains
an interior empirical score root. It is that fiber's unique global
output minimizer, and every joint global ERM uses this root in its
selected hidden fiber.
\end{lemma}
\begin{proof}
For example, $\sigma=\min\{1/2,(\eta-2)/4\}$ satisfies
\eqref{eq:r25-p-choice}. This is an analysis exponent, not a model
parameter. At this resolution, $\epsilon=\sqrt{c_t}h$ satisfies
$n^{-B_0}\le\epsilon\le\epsilon_0$ for any fixed $B_0>1/(2\eta+d)$
and all sufficiently large $n$, so Lemma~\ref{lem:r25-noise} applies.
Uniformly on the good event,
\[
 r_{n,J}^{\rm pil}=O(h^{\eta-1}\log n),\qquad
 \delta_n=O(h^\eta\sqrt{\log n}+h^{2\eta}\log n).
\]
The three smallness conditions used below are exactly
\begin{equation}\label{eq:r25-three-smallness}
 \frac{r_{n,J}^{\rm pil}}{\epsilon}\to0,\qquad
 \frac{\delta_n}{h^2}\to0,\qquad
 h^{\eta-\beta-\sigma}\sqrt{\log n}\to0.
\end{equation}
Their positive powers are $\eta-2$, $\eta-2$ and
$(\eta-2)/2-\sigma$, respectively. There is no change at $\eta=3$.

Fix one hidden configuration and its comparator $\theta^0$.
Let $r^0=\theta^0-\theta^*$ denote its two full-field errors, and
consider retained corrections $c$ in the closed $E_p$ ball of radius
\[
 R=C_0\{h^{\eta-1}+\delta_n/\epsilon\}.
\]
Morrey and the same-bank coefficient characterization show that the
level-$j$ density coefficient, divided by its box width, is at most
$C2^{j(\beta-1+\sigma)}R$. For a normal coefficient the $L^{p'}$
norm of its dual atom gives $Cs2^{j(\beta+1+\sigma)}R$.
The maxima of both ratios are
$O(h^{\eta-\beta-\sigma}\sqrt{\log n})=o(1)$.
Since the comparator uses less than $A/8$, the entire correction ball
is strictly inside the original independent coefficient boxes.
There is no free mass-pivot coefficient.

For the pilot constraint, Sobolev and retained Bernstein give
\[
 \|c^{\rm den}\|_\infty\le CR,\qquad
 \|c^{\rm nor}\|_{C^1}\le Cs h^{-1-\sigma}R.
\]
Both are $o(r_{n,J}^{\rm pil})$ at this resolution. The comparator's
$C^0\times C^1$ error is $O(h^\eta)$ and the pilot error is at most
$r_{n,J}^{\rm pil}/4$. The correction ball is thus strictly inside
the prescribed pilot ball. These are not yet the coarse margins.

For those margins, let $\varphi_0$ be the original unit-mass pivot,
$\mu_i=\int\psi_i$ and $\varphi_i=\psi_i-\mu_i\varphi_0$.
For mean-zero input $g$, the corrected input-coarse truncation is
exactly
\begin{equation}\label{eq:r25-input-coarse}
 \sum_{i\ne0,\,j_i\le L}\langle g,\widetilde\psi_i\rangle\varphi_i
       =Q_{\rm mass}\Pi_{\le L}^{\rm full}g,
 \qquad Q_{\rm mass}g=g-\varphi_0\int g.
\end{equation}
Apply $Q_{\rm mass}$ to each full-bank atom to see this identity;
it annihilates the pivot. In particular a discarded fine atom's pivot
contribution is discarded too, as required by input-index truncation.
The actual density coarse correction is
$U^{\rm den}Q_{\rm mass}\Pi_{\le L}^{\rm full}
(U^{\rm den})^{-1}c^{\rm den}$.
Since $L$ is fixed, this finite-rank operator has a uniform
$L^p\to C^0$ bound. The corresponding normal input-coarse operator
has a uniform $L^p\to C^1$ bound. Hence
\[
 \|c^{\rm den}_{\le L}\|_{C^0}\le C_LR,\qquad
 \|c^{\rm nor}_{\le L}\|_{C^1}\le C_LsR.
\]
Both tend to zero, preserving the original $\lambda/4$ positivity
and $\varepsilon/4$ tube margins for all sufficiently large $n$.
This verifies separately all three original constraints.

For $\theta=\theta^0+c$ the exact law expansion at that candidate is
\[
 P^*-P_\theta=DP_\theta[-r^0-c]+R_\theta(-r^0-c).
\]
Equation~\eqref{eq:r25-exact-gradient} is therefore equivalent to
the following continuous fixed-point equation in the retained space:
\begin{align}
 c={}&-I_\theta^{-1}I_\theta(r^0,\cdot)
    +I_\theta^{-1}\mathcal A_\theta
             (R_\theta(-r^0-c),DP_\theta[\cdot])\notag\\
   &+I_\theta^{-1}\mathcal A_\theta(P_n-P^*,DP_\theta[\cdot]).
   \label{eq:r25-brouwer}
\end{align}
The first term is $O(h^{\eta-1})$ by the comparator and first-law
forward bounds followed by the retained inverse. For
$\delta=-r^0-c$, Lemma~\ref{lem:r25-law-remainder} bounds the
second by
\[
 C(\|\delta^{\rm den}\|_\infty
                +\|\delta^{\rm nor}\|_\infty/\epsilon)\|\delta\|_{E_p}.
\]
Its prefactor tends to zero on the entire ball by the preceding
estimates. The last term is $C_p\delta_n/\epsilon$ by
\eqref{eq:r25-noise-finite-p}. With fixed slack in $C_0$, the right
side maps the ball into itself for large $n$. Its continuity follows
from the finite law integrals and the uniformly invertible information
matrix. Brouwer gives a root strictly inside all original constraints.

Every feasible candidate is within $Cr_{n,J}^{\rm pil}$ of the truth
in $C^0\times C^1$. Lemmas~\ref{lem:vs-finite-p-inverse}
and~\ref{lem:mlsn-uniform-hessian}, with
\eqref{eq:r25-three-smallness}, give Hessian at least $c\|w\|_E^2$
on the entire convex output fiber. The constructed root is therefore
its unique global output minimizer. Every joint global minimizer
minimizes over its own output fiber; no stationarity, convexity or
identifiability in hidden parameters is required. The construction
also proves nonemptiness on the good event without changing the
observable fallback.
\end{proof}

\begin{lemma}[Interior population output roots]
\label{lem:r26-population-root}
Under the resolution and smallness conditions of
Lemma~\ref{lem:pm-newton-margins}, for every good pilot realization
and every hidden configuration the original output fiber has a unique
population minimizer $\theta_{\rm o}^\dagger(\thetah)$.
It is an interior score root, is measurable in the pilot and hidden
parameters, and its synthesized fields satisfy
\begin{equation}\label{eq:r26-population-energy}
 \|(r^\dagger_{\thetah}-r^*,v^\dagger_{\thetah}-v^*)\|_E
 \le Ch^{\eta-1}
\end{equation}
uniformly over these choices.
\end{lemma}
\begin{proof}
Use $p,\sigma$ from~\eqref{eq:r25-p-choice} and the comparator
$\theta^0$ of Lemma~\ref{lem:mlsn-architecture}, with full-field error
$r^0$ satisfying $\|r^0\|_{E_p}\le C_ph^{\eta-1}$.
In the fixed-point equation~\eqref{eq:r25-brouwer}, set $P_n=P^*$
and use the correction radius $R_0=C_0h^{\eta-1}$.
The three feasibility checks in Lemma~\ref{lem:pm-newton-margins}
remain valid at this smaller radius: the coefficient-box ratio is
$O(h^{\eta-\beta-\sigma})=o(1)$, and the pilot-ball and input-coarse
margins remain strict.

The empirical term vanishes. The first-law forward bound and
Lemma~\ref{lem:vs-finite-p-inverse} bound the linear term by
$Ch^{\eta-1}$. For $\theta=\theta^0+c$ and $\delta=-r^0-c$,
the comparator and retained Bernstein bounds give
\[
 \|\delta^{\rm den}\|_\infty=o(1),\qquad
 \|\delta^{\rm nor}\|_\infty
 \le Ch^{\eta+1}+Cs h^{-\sigma}R_0=o(\epsilon).
\]
Lemma~\ref{lem:r25-law-remainder} and the retained inverse therefore
bound the nonlinear term by $o(R_0)$.
For sufficiently large fixed $C_0$, the continuous map preserves the
closed correction ball, so Brouwer gives an interior root. Its full-field $E_p$ error is $O(h^{\eta-1})$,
which implies~\eqref{eq:r26-population-energy} since $p>2$.

The population Hessian calculation in Lemma~\ref{lem:mlsn-uniform-hessian}
has error at most $Cr_{n,J}^{\rm pil}\epsilon^{-1}\|w\|_E^2$.
Since $r_{n,J}^{\rm pil}/\epsilon\to0$,
Lemma~\ref{lem:vs-finite-p-inverse} gives strict convexity on the
entire output fiber, hence uniqueness of the global population minimizer.
The measurable-argmin argument in Appendix~\ref{sec:repu-class}
gives measurability in the pilot and hidden parameters.
Only the good pilot realization is used; no empirical event is required.
\end{proof}

\subsubsection{Uniform weak-coordinate estimates}
We combine energy and weak-density estimates to control both fits
relative to the truth and the quadratic term in Appendix~\ref{sec:mlsn-proof}.

\begin{lemma}[Uniform population and empirical coordinate bounds]
\label{lem:mlsn-consequences}
Under the resolution and smallness conditions of
Lemma~\ref{lem:pm-newton-margins}, fix a good pilot realization $z$.
Uniformly over its hidden configurations,
\begin{equation}\label{eq:r26-population-weak}
 \mathcal W(r^\dagger_{\thetah}-r^*,v^\dagger_{\thetah}-v^*)
 \le Ch^{\eta+1}.
\end{equation}
On the empirical good event $\mathcal G:=\mathcal E_n$, the empirical
output minimizers satisfy
\begin{equation}\label{eq:r26-empirical-weak}
 \mathcal W(\widehat r_{\thetah}-r^*,\widehat v_{\thetah}-v^*)
 \le C\{h^{\eta+1}+h^2W_n+\|Z_n\|_{(H^3)^*}\},
\end{equation}
where $W_n$ and the fixed-truth reference process $Z_n$ are from
Lemmas~\ref{lem:r25-noise} and~\ref{lem:vs-fixed-root-hessian}.
For any fitting-data-measurable hidden selection, with its corresponding
empirical output minimizer, this implies
\begin{equation}\label{eq:r26-coordinate-second-moment}
 \mathbb E_{\rm fit}\!\left[
 \mathbf1_{\mathcal G}
 \|\widehat\theta_{\rm o}
     -\theta_{\rm o}^\dagger(\widehat\theta_{\rm h})\|_{\bm W_J}^2
 \mid z\right]
 \le C\{h^{2\eta+2}+n^{-1}h^{2-d}\}.
\end{equation}
The constants are uniform over good pilots and admissible branches.
\end{lemma}
\begin{proof}
Fix the whole good pilot realization, hence its branch and output domain.
The fitting observations remain independent with law $P^*$.
All pointwise estimates below are uniform over the hidden box, so they
apply also at the hidden configuration selected using these observations.

\paragraph*{Energy error.}
Strong convexity on the output fiber, the feasible comparator, and
\eqref{eq:r25-noise-forward} give
\begin{equation}\label{eq:r25-energy-error}
 \|\widehat\theta-\theta^*\|_E
       \le C(h^{\eta-1}+W_n)\quad\hbox{on }\mathcal G.
\end{equation}
Indeed the deterministic comparator gradient has dual energy norm
$O(h^{\eta-1})$ by its exact law expansion and
\eqref{eq:r25-law-remainder}; its centered empirical gradient is at
most $CW_n$. Apply the convex variational inequality at the minimizer
and curvature along the segment to the comparator. This argument also
works at a boundary minimizer and does not presume interiority to
obtain the energy bound.

\paragraph*{The weak density equation.}
Write $e=q_*^{\rm ref}-q_{\widehat\theta}^{\rm ref}$ and
$b=\Phi_* -\Phi_{\widehat\theta}$ on the selected fixed reference.
Then $\int e=0$, and $\|b\|_2\le sE_*$, where
$E_*:=\|\widehat\theta-\theta^*\|_E$.
For a reference $H^1$ test $g$, transfer it to the fitted physical
manifold and solve
\[
 -2\operatorname{div}_{\mathcal M}(q\nabla_{\mathcal M}t)
       =q\left(g-\int qg\,d\nu_{\mathcal M}\right),
 \qquad \int qt\,d\nu_{\mathcal M}=0.
\]
Here $q$ is the physical fitted density. In reference coordinates the
coefficient matrix is $q_{\rm ref}G^{-1}$, uniformly elliptic and
$W^{2,\infty}$. Elliptic regularity gives
$\|t_{\rm ref}\|_{H^3}\le C\|g_{\rm ref}\|_{H^1}$.
The absolute reference density direction is $q_{\rm ref}t_{\rm ref}$,
not the relative test alone. Lemma~\ref{lem:r25-product} supplies a
retained mean-zero $w_h$ such that
\[
 a:=q_{\rm ref}t_{\rm ref}-w_h,
 \qquad \|a\|_{H^1}\le Ch^2\|t_{\rm ref}\|_{H^3},\qquad \int a=0.
\]
Interiority from Lemma~\ref{lem:pm-newton-margins} permits testing the
empirical score equation with $w_h$. Replacing it by the full
absolute density direction costs at most
$Ch^2(E_*+W_n)\|t_{\rm ref}\|_{H^3}$.
For the empirical part this follows from the explicitly nonretained
forward bound~\eqref{eq:r25-noise-nonretained}; for the deterministic
part use the first-law forward bound and the exact law remainder.
No inverse or coercivity is applied to $a$.

Expand $P^*-P_{\widehat\theta}$ at the fitted candidate.
The density part of~\eqref{eq:r25-weak-comparison}, followed by the
dual equation, is $\int e g_{\rm ref}\,d\nu_0$ with error
$CsE_*\|t_{\rm ref}\|_{H^3}$.
The linear support contribution costs
$C\|b\|_2\|t_{\rm ref}\|_{H^3}$ by~\eqref{eq:r25-linear-weak}.
The exact geometric weak remainder costs
\[
 C(\|e/q_{\rm ref}\|_\infty+\|b\|_\infty/\epsilon)
                      \|b\|_2\|t_{\rm ref}\|_{H^3},
\]
which is bounded by the same support term on $\mathcal G$.
The empirical comparison~\eqref{eq:r25-empirical-comparison} separates
the fitted candidate from the fixed truth at cost
$C\delta_nE_*\|t_{\rm ref}\|_{H^1}$.
The remaining process is $Z_n(t_{\rm ref})$ from
\eqref{eq:r25-weak-noise}. Its dual norm permits this data-dependent
test. Taking the supremum over $\|g_{\rm ref}\|_{H^1}\le1$ gives
\begin{equation}\label{eq:r25-weak-error}
 \|e\|_{H^{-1}}+\|b\|_2
 \le C\{h^2(E_*+W_n)+\delta_nE_*+\|Z_n\|_{(H^3)^*}\}.
\end{equation}
Since $\delta_n/h^2\to0$, the middle term is absorbed into the first.
Together with~\eqref{eq:r25-energy-error}, this proves
\eqref{eq:r26-empirical-weak}.

\paragraph*{Population counterpart.}
At the population minimizer supplied by
Lemma~\ref{lem:r26-population-root}, use the population score equation
in the same retained test direction. The preceding weak-density
calculation has $W_n=0$, $\delta_n=0$, and $Z_n=0$, and gives
\[
 \mathcal W(r^\dagger_{\thetah}-r^*,v^\dagger_{\thetah}-v^*)
 \le Ch^2\|(r^\dagger_{\thetah}-r^*,v^\dagger_{\thetah}-v^*)\|_E.
\]
Equation~\eqref{eq:r26-population-energy} proves
\eqref{eq:r26-population-weak}. This step uses the population root,
not the projection comparator as a substitute for that minimizer.

\paragraph*{Second moment in the output coordinates.}
Synthesis is linear in the output coefficients. The two-sided norm
equivalence in Lemma~\ref{lem:normal-density-chart} and the uniform
pullback bounds in Lemma~\ref{lem:mlsn-architecture} give, in each fiber,
\[
 \|a\|_{\bm W_J}\asymp\mathcal W(S_{J,m,\thetah}a).
\]
The triangle inequality in the reference-field norm, followed by
\eqref{eq:r26-population-weak} and~\eqref{eq:r26-empirical-weak}, yields
\[
 \|\widehat\theta_{\rm o}(\thetah)
       -\theta_{\rm o}^\dagger(\thetah)\|_{\bm W_J}
 \le C\{h^{\eta+1}+h^2W_n+\|Z_n\|_{(H^3)^*}\}
 \quad\text{on }\mathcal G.
\]
This is a coefficient-to-field norm equivalence, not an inverse
Wasserstein inequality. Equations~\eqref{eq:r25-integrated-noise}
and~\eqref{eq:r25-weak-noise} give
\[
 \mathbb E_{\rm fit}[W_n^2\mid z]\le Cn^{-1}h^{-d-2},
 \qquad
 \mathbb E_{\rm fit}[\|Z_n\|_{(H^3)^*}^2\mid z]
 \le Cn^{-1}h^{2-d}.
\]
The pilot is independent of the fitting observations, and a finite
reference union changes only constants. Squaring the pointwise bound
at the selected hidden configuration proves
\eqref{eq:r26-coordinate-second-moment}. We insert the good-event
indicator into this expectation rather than condition on that event.
Neither a fixed-hidden trace identity at a random hidden selection nor
an expectation bound obtained from the logarithmic supremum envelope
is used.
\end{proof}

\subsection{Minimax lower bound}
\label{sec:minimax-lower-bound}
We establish the matching lower bound on a fixed-support subclass.
An Assouad construction gives the density lower bound, and prescribed-Jacobian
reconstruction embeds this family in the transport class.

\begin{proposition}[Fixed-support minimax transfer]
\label{prop:vs-minimax-transfer}
Under Assumptions~\ref{G1}--\ref{G2}, let $d>2$.  Then
$\inf_{\widetilde P_n}\sup_{h^*\in\mathcal H_*}\mathbb E
W_2(P_{h^*},\widetilde P_n)\gtrsim n^{-(\eta+1)/(2\eta+d)}$, where the
infimum is over all measurable estimators based on $n$ observations.
\end{proposition}

\begin{proof}
Choose the strictly interior reference $h_0\in\mathcal H_*$ from
Assumption~\ref{G2}, and write $P_{h_0}=f_0\nu_0$ on
$\mathcal M_0=h_0(\mathcal U)$.  In one coordinate cube choose
$m_j\asymp2^{jd}$ disjoint $L^2$-normalized smooth wavelets $\psi_{jk}$,
supported in a smaller cube, and put
\[
 a_j=c2^{-j(\eta+d/2)},\qquad
 f_\omega=f_0+a_j\sum_{k=1}^{m_j}\omega_k\psi_{jk},\qquad
 P_\omega=f_\omega\nu_0,\qquad
 \omega\in\{0,1\}^{m_j}.
\]
The wavelets have zero $\nu_0$ integral.  For a fixed sufficiently small $c$, their
disjoint supports give uniformly $f_\omega\ge f_0/2$ and
$\|f_\omega\|_{C^\eta}\le C$.  If $\omega$ and $\omega^{(k)}$ differ in one
bit, the common positive background and the bi-Lipschitz chart give
\begin{equation}\label{eq:vs-lower-neighbor-information}
 H^2(P_\omega,P_{\omega^{(k)}})
 +{\rm KL}(P_\omega,P_{\omega^{(k)}})\le Ca_j^2.
\end{equation}
Localized Kantorovich--Rubinstein test functions with slope one on each cell,
extended by McShane outside the cube, also give
\begin{equation}\label{eq:vs-lower-W1-hamming}
 W_2(P_\omega,P_{\omega'})\ge W_1(P_\omega,P_{\omega'})
 \ge c a_j2^{-j(d/2+1)}d_H(\omega,\omega').
\end{equation}
Choose $j$ so that $2^{j(2\eta+d)}\asymp n$.  Then
$na_j^2\asymp1$, and Assouad's lemma applied to
\eqref{eq:vs-lower-neighbor-information}--
\eqref{eq:vs-lower-W1-hamming} yields
$m_ja_j2^{-j(d/2+1)}\asymp
2^{-j(\eta+1)}\asymp n^{-(\eta+1)/(2\eta+d)}$.

{Put $V_0=\nu_0(\mathcal M_0)$ and $\mu=\nu_0/V_0$, a smooth probability volume.
Apply the local part of Lemma~\ref{lem:quantitative-prescribed-jacobian}
to $V_0f$ about the positive $C^{\eta+1}$ density $V_0f_0$.
It supplies $S_f\in C^{\eta+1}$ with $(S_f)_\#P_{h_0}=f\nu_0$ and $S_{f_0}=I$.}
Then $h_f=S_f\circ h_0$ obeys $P_{h_f}=f\nu_0$ and, by
\eqref{eq:quantitative-prescribed-jacobian-local},
$\|h_f-h_0\|_{C^{\eta+1}}\le C\|f-f_0\|_{C^\eta}$.  The preceding
$C^\eta$ perturbation bound makes this at most $Cc$ uniformly.  Choose $c$
below the fixed interiority margin of $h_0$; then every $h_f$ belongs to
$\mathcal H_*$.  Thus the explicit Assouad family is a submodel of the
theorem's parameter class and transfers the lower bound.
\end{proof}

\subsection{Finite-anchor specialization}
\label{sec:finite-anchor-specialization}
This section studies a fixed-proposal, anchor-averaged criterion separately
from the exact-loss result of Theorem~\ref{thm:acc}.
We bound the loss perturbation, use output-fiber curvature to control
Wasserstein error, and identify a sufficient anchor budget.

\begin{lemma}[Fresh-anchor perturbation on a dominated local class]
\label{lem:vs-anchor-rate}
Condition on an admissible independent-pilot realization. Fix one hidden
configuration and a nonempty compact convex subdomain
$\mathcal B_{n,J}^{\rm anc}$ of its prescribed output fiber. Suppose
it contains the ideal empirical root except on an event of probability
$Cn^{-2}$, and both criteria have fixed Borel minimizer selections.
Let $\widehat\theta_J^{\rm anc}$ minimize the exact criterion on this
subdomain, and let $\widetilde\theta_{J,K}^{\rm anc}$ minimize
\[
 \widetilde L_{n,J,K}(\theta)=\frac1n\sum_{i=1}^n(1-t_J)
       \mathbb E_{Z,\rm anc}
       \ell_K(\alpha_{t_J}X_i+\sigma_{t_J}Z,X_i;h_\theta).
\]
Set $\Delta_{J,K}^{\rm anc}=\sup_{\theta\in\mathcal B_{n,J}^{\rm anc}}
|\widetilde L_{n,J,K}(\theta)-L_{n,J}(\theta)|$ and
$\delta_{J,K}^{\rm anc}=(\mathbb E\Delta_{J,K}^{\rm anc})^{1/2}$.
To express the quantity from Appendix~\ref{sec:mcapprox} in additive-noise
coordinates, define
$D_{2,h}^{\rm add}(y):=D_{2,h}(\alpha_{t_J}y)$, with any query-dependent
proposal also evaluated at the VP query $\alpha_{t_J}y$.
If, for $Y\sim p_{\tau_J}^*$,
\begin{equation}\label{eq:vs-anchor-domination}
 \mathbb E\sup_{\theta\in\mathcal B_{n,J}^{\rm anc}}
                         D_{2,h_\theta}^{\rm add}(Y)\le C_D\tau_J^{-d/2},
\end{equation}
then, in the resolution regime of Lemma~\ref{lem:mlsn-consequences},
\begin{align}
 \mathbb E\Delta_{J,K}^{\rm anc}&\le\frac C{K\tau_J^{d/2+2}},
 \qquad \delta_{J,K}^{\rm anc}\le\frac C{\sqrt K\tau_J^{d/4+1}},
       \label{eq:vs-anchor-rate}\\
 \mathbb EW_2(P^*,P_{h_{\widetilde\theta_{J,K}^{\rm anc}}})
 &\le C\left\{2^{-J(\eta+1)}+(\mathfrak B_J/n)^{1/2}
          +\frac1{\sqrt K\tau_J^{d/4+1}}+n^{-2}\right\}.
       \label{eq:vs-anchor-w2}
\end{align}
At $2^{J_n}\asymp n^{1/(2\eta+d)}$, the sufficient budget
$K_n\gtrsim n^{1+6/(2\eta+d)}$ makes the added term no larger than
the displayed minimax rate.
\end{lemma}
\begin{proof}
Lemma~\ref{lem:anchor-averaged-loss}, averaged over corruption, gives
\[
 \mathbb E\Delta_{J,K}^{\rm anc}\le\frac C{K\tau_J^2}
 \mathbb E_{Y\sim p_{\tau_J}^*}
       \sup_{\theta\in\mathcal B_{n,J}^{\rm anc}}D_{2,h_\theta}^{\rm add}(Y).
\]
Here $\sigma_{t_J}^2=\alpha_{t_J}^2\tau_J$ and $\alpha_{t_J}$ is
bounded away from zero. This proves~\eqref{eq:vs-anchor-rate}.
Optimality gives $L_{n,J}(\widetilde\theta)
\le L_{n,J}(\widehat\theta)+2\Delta_{J,K}^{\rm anc}$.
On the uniform good event, whole-fiber curvature from
Lemma~\ref{lem:mlsn-uniform-hessian} and the convex variational
inequality give
$\|\widetilde\theta-\widehat\theta\|_E
\le C\sqrt{\Delta_{J,K}^{\rm anc}}$.
The local transport bound, or
$\|a\|_{H^{-1}}+\|b\|_2\le C\|(a,b)\|_E$, then controls
the Wasserstein difference by the same quantity.
The fixed-hidden version of~\eqref{eq:r25-final-bound} applies when
the subdomain contains its ideal root. Since
$\mathfrak B_J\asymp h^{2-d}$ for $d>2$, its ideal error and
Cauchy--Schwarz prove~\eqref{eq:vs-anchor-w2}. Bounded support handles
all failed events. Finally $\tau_{J_n}\asymp n^{-2/(2\eta+d)}$
gives the stated budget. The unchanged $O(n^{-2})$ pilot-failure term
may be added when removing the conditioning.
\end{proof}

\endgroup

\newcommand{\ip}[2]{\left\langle #1,#2 \right\rangle}

\FloatBarrier
\bigskip
\setlength{\bibsep}{0pt plus 0.3ex}
\bibliography{references}
\end{document}